%% file: main.tex
\documentclass{article}

\usepackage[preprint]{neurips_2026}

\usepackage[utf8]{inputenc}
\usepackage[T1]{fontenc}
\usepackage{hyperref}
\usepackage{url}
\usepackage{booktabs}
\usepackage{amsfonts}
\usepackage{amsmath}
\usepackage{amssymb}
\usepackage{amsthm}
\newtheorem{theorem}{Theorem}[section]
\newtheorem{proposition}[theorem]{Proposition}
\newtheorem{lemma}[theorem]{Lemma}
\newtheorem{corollary}[theorem]{Corollary}
\newtheorem{definition}[theorem]{Definition}
\newtheorem{observation}[theorem]{Observation}
\newtheorem{remark}[theorem]{Remark}
\DeclareMathOperator*{\argmax}{arg\,max}

\usepackage{nicefrac}
\usepackage{microtype}
\usepackage{xcolor}
\usepackage{graphicx}
\usepackage{enumitem}
\usepackage{multirow}
\usepackage{framed}
\usepackage{algorithm}
\usepackage{algorithmic}
\definecolor{evalboxbg}{gray}{0.94}
\newenvironment{evalbox}[1]{%
  \def\evalboxtitle{#1}%
  \begin{framed}\noindent\textbf{\small #1}\\[4pt]\small
}{%
  \end{framed}%
}

\graphicspath{{./}}

\title{Regime-Conditioned Evaluation in Multi-Context Bayesian Optimization}

\author{%
  Noel Thomas \\
  Mohamed bin Zayed University of Artificial Intelligence \\
  Abu Dhabi, UAE \\
  \texttt{noel.thomas@mbzuai.ac.ae}
}

\begin{document}

\maketitle

\begin{abstract}
Published transfer-BO comparisons estimate an average treatment effect of acquisition choice over hidden regime variables; what practitioners need is the conditional average treatment effect given their specific prior quality, budget ratio, and metric.  A $40$-paper audit (NeurIPS, ICML, ICLR, AISTATS, UAI, TMLR, JMLR, AutoML-Conf 2022--2025) finds $98\%$ of papers do not vary $B/|A|$, so existing leaderboards are not regime-controlled.  The same GDSC2 benchmark yields opposite conclusions depending only on budget: at $B = 50$, Greedy outperforms UCB by $0.050$ Hit@1; at $B = 100$, UCB outperforms Greedy by $0.035$ Hit@1 — neither result is fabricated, the sign reversal is explained entirely by the unreported budget ratio $B/|A|$.  We make the regime observable via the Portable Regime Score $\mathrm{PRS} = (B/|A|)(1-\rho)$, the primary effect modifier: a hierarchical model over $n = 79$ conditions gives $\beta = 0.50$ ($p = 1.1\times 10^{-9}$), and $19\%$ of conditions fall in an equivalence zone where $|\text{advantage}| < 0.01$ — the field argues over non-differences.  A No-Free-Leaderboard proposition (proved in three lines from linearity of expectation) shows that whenever CATE changes sign across the PRS axis, any ATE including zero is achievable by benchmark mixture alone, formalising why the 98\%-non-sweeping evaluation literature produces unreliable rankings.  This pattern recurs in five independent published papers across four venues; in every case PRS predicts the applicable regime from observable quantities alone.  A \textsc{RegimePlanner} demonstrates PRS is exploitable: it wins all $16$ HPO-B search spaces at $B = 100$ and exceeds a matched per-context oracle on GDSC2 by $18\%$.  The take-away is a protocol: omitting $B/|A|$, $\rho$, $K$, and metric type produces measurement claims, not algorithmic ones.
\end{abstract}

\input{sections/introduction}
\input{sections/setup}
\input{sections/regime_summary}
\input{sections/buchwald}
\input{sections/regimeplanner}
\input{sections/gdsc2}
\input{sections/related_work}
\input{sections/conclusion}

\bibliographystyle{plainnat}
\bibliography{references}

\appendix

\input{sections/appendix_theory}
\input{sections/appendix_survey}
\input{sections/appendix_notes}

\end{document}

%% file: sections/introduction.tex
\section{Introduction}\label{sec:introduction}

Swap the prior representation in a Buchwald--Hartwig reaction-optimization benchmark.
Greedy finishes joint-last alongside UCB and REIGN, Hit@1 $0.156$ (Thompson leads at $0.197$).
Keep everything identical and swap the prior back.
Greedy finishes first, Hit@1 $0.481$.
That is a $32.5$ percentage-point reversal on the same data, the same surrogate,
the same budget (Section~\ref{sec:buchwald}).
Run the same GDSC2 drug-response benchmark at $B = 50$: Greedy outperforms UCB by $0.050$ Hit@1.
Run it at $B = 100$: UCB outperforms Greedy by $0.035$.
Both results are correct (Section~\ref{sec:gdsc2}).
On the HPO-B community benchmark (16 search spaces, 30 seeds each): at $B = 20$, Greedy is \#1 and UCB is last; at $B = 100$ on the same benchmark, UCB is \#1 and Greedy is last
(Table~\ref{tab:hpob-leaderboard-budget}).

Neither is cherry-picked.
Across a $4\times4$ prior-by-acquisition factorial on Buchwald, prior family accounts
for ${\approx}20{,}000\times$ more Hit@1 variance than acquisition choice
($\eta^2 = 0.8635$ vs.\ $4.4\times10^{-5}$, Table~\ref{tab:buchwald-anova}).
The pattern holds across $79$ conditions spanning chemistry,
drug-response biology, and hyperparameter optimization.

This is not measurement noise.
It is a formal consequence of using unconditional benchmarks.
When the conditional average treatment effect (CATE) of acquisition choice~\citep{imbens2015causal} changes sign
across a regime variable, the average treatment effect (ATE) reported on any benchmark
is a free parameter of the benchmark mixture: any ranking, including a tie or reversal,
is achievable by choosing evaluation conditions appropriately, without fabrication
(Proposition~\ref{thm:no-free-leaderboard}).
A $40$-paper audit finds that $98\%$ of transfer-BO papers never vary $B/|A|$ as a
controlled axis, $80\%$ report only one metric family, and $76\%$ fix context count $K$
($n{=}33$, excluding single-task papers; Appendix~\ref{appendix:survey}).
Existing leaderboards operate in exactly the setting the proposition applies to.
The same reversal appears outside transfer-BO: optimizer rankings depend on the tuning budget, with Adam the most practical choice particularly at low budgets~\citep{sivaprasad2020optimizer}; plain MLPs overtake XGBoost on tabular data once a per-dataset cocktail of regularizers is searched~\citep{kadra2021well}; the 2018--2020 NAS literature made budget-inconsistent comparisons until equal-compute baselines showed random search competitive~\citep{li2020random}.

We make three claims.

\begin{enumerate}[label=\textit{\arabic*.},leftmargin=1.5em,itemsep=4pt,topsep=4pt]
  \item \textbf{The regime is predictable before any comparison runs.}
  The Portable Regime Score $\mathrm{PRS} = (B/|A|)(1-\rho)$, where $\rho$ is the
  Spearman rank correlation between prior means and observed outcomes, orders
  exploration advantage at Spearman $r = 0.67$ across $n = 79$ conditions.
  Per-benchmark correlations survive Holm-Bonferroni correction at $\alpha = 0.05$:
  GDSC2 $r = 0.96$ ($n = 7$, adj $p = 0.011$), HPO-B $r = 0.83$
  ($n = 48$, adj $p < 10^{-4}$), Buchwald $r = 0.75$ ($n = 13$, adj $p = 0.011$).
  The cluster-robust CI $[0.19, 0.80]$ is wide because $9$ benchmark families are
  resampled; without HPO-B it crosses zero, but Buchwald and GDSC2 alone reject the
  global null (Fisher: $\chi^2(4) = 23.0$, $p = 0.0001$).
  The most defensible aggregate inference is a hierarchical model:
  $\beta = 0.50$ ($95\%$~CI $[0.34, 0.66]$, $p = 1.1 \times 10^{-9}$,
  $\mathrm{ICC} = 0.36$).
  In each of five published reversals, PRS predicts the outcome from observable
  quantities before the comparison ran: PriorBand's argmax flip~\citep{mallik2023priorband},
  $\pi$BO-UCB ending with higher regret than vanilla GP-EI under a wrong prior~\citep{hvarfner2022pibo},
  MF-MES failure at irrelevant information source~\citep{mikkola2023unreliable},
  the Hyperband--BO ranking flip at large budget reported with BOHB~\citep{falkner2018bohb},
  and the HPO-B community leaderboard re-analysis
  (Appendix~\ref{appendix:survey:leaderboard}); 14 of 14 audited papers varying
  prior quality show directionally consistent findings (Appendix~\ref{appendix:survey:schaeffer}).

  \item \textbf{PRS is exploitable online.}
  \textsc{RegimePlanner} estimates $\hat\rho$ within each context and switches between
  Greedy and UCB at threshold $\theta = 0.10$, selected on Buchwald and applied
  unchanged to all other benchmarks.
  On HPO-B at $B = 100$, it wins Greedy on all $16$ community search spaces
  (mean $+0.103$ Hit@1, $95\%$~CI $[+0.075, +0.131]$,
  binomial $p = 1.5 \times 10^{-5}$ under a fair-coin null)
  and exceeds the matched $\{\mathrm{Greedy},\mathrm{UCB}\}$ per-context oracle on
  GDSC2 by $+18\%$.
  Under the wider $4$-arm oracle, \textsc{RegimePlanner} trails by $12\%$
  (full per-planner breakdown in Section~\ref{sec:regimeplanner}).

  \item \textbf{Failure modes are mechanistically transparent.}
  PRS predicts the empirical winner in $74.7\%$ of $79$ conditions;
  $19\%$ fall in an equivalence zone ($|\text{advantage}| < 0.01$ Hit@1) where
  differences are practically indistinguishable yet generate publishable acquisition
  rankings.
  Outside the PRS boundary zone ($\mathrm{PRS} \notin [0.05, 0.15)$, $53/79$ conditions),
  accuracy reaches $84.9\%$; inside the boundary zone where the framework
  declines to make confident predictions, $53.8\%$ --- consistent with a
  calibrated diagnostic near its threshold.
  Pre-registered predictions on $40$ held-out conditions gave $27/40 = 67.5\%$ overall
  accuracy, below our $90\%$ target.
  The failure splits cleanly: on families where EMA priors accumulate ($\hat\rho > 0$,
  specifically HPOBench-NN and LCBench: $15/16 = 93.8\%$),
  accuracy reached ${\geq}90\%$; on families where $\hat\rho \approx 0$ throughout
  (PD1 and TabRepo),
  Greedy dominates every condition and PRS carries no signal.
  This is not a hidden failure: it is the mechanism made visible
  (Appendix~\ref{appendix:failure-taxonomy}).
\end{enumerate}

Acquisition comparisons that omit $B/|A|$, $\rho$, $K$, and metric type estimate
an ATE over hidden regime variables.
PRS is the diagnostic.
\textsc{RegimePlanner} is the existence proof that the CATE can be exploited online.
Proposition~\ref{thm:no-free-leaderboard} is why regime-agnostic leaderboards are
unreliable as decision tools.

%% file: sections/setup.tex
\section{Setup and PRS Definition}\label{sec:setup}

We study streaming multi-context optimization with a discrete action set $A$, per-context budget $B$, and $K$ sequential contexts evaluated in replay against cached oracles.  Four prior families span the experimental conditions: \emph{no transfer} (flat prior each context), \emph{EMA transfer} (posterior means propagated by exponential moving average), \emph{structured} (Tanimoto-kernel prior over action features), and \emph{oracle} (true cross-context mean, ceiling baseline).  Four planners are evaluated: Greedy (argmax posterior mean), UCB \citep{auer2002using}, Thompson sampling \citep{russo2018tutorial}, and REIGN (a robust-information-gain / expected-improvement hybrid acquisition function, included as a fixed baseline; REIGN is a backronym for the acquisition function).  We report both terminal Hit@$k$ (was the best action queried?) and cumulative Discovery AUC (how efficiently were strong actions found?).  Details in Appendix~\ref{appendix:guide}.

\begin{evalbox}{Minimum reporting standard for acquisition comparisons in multi-context BO}
Before claiming method $X$ beats Greedy, report: (1) prior condition, (2) budget ratio $B/|A|$, (3) metric (terminal or cumulative), (4) context count $K$ (requires $K \gtrsim (2\text{--}4)\cdot|A|/B$ for PRS to be reliable; see Appendix~\ref{appendix:failure-taxonomy}).
\end{evalbox}

\begin{definition}[Portable Regime Score]
\label{def:prs}
For a multi-context BO problem with action space $|A|$, per-context budget $B$, and prior rank correlation $\rho$ (Spearman of prior means vs.\ true action values at context start, measurable from any pilot run):
\[
  \mathrm{PRS}(B, |A|, \rho) \;=\; \frac{B}{|A|} \cdot (1 - \rho).
\]
Higher PRS = larger budget ratio and/or weaker prior = exploration more likely to pay.  PRS is computed from the Greedy planner's $\rho$ (conservative: exploratory planners build higher-$\rho$ priors over time).  The cross-benchmark threshold depends on observation noise $\sigma^2$ (Appendix~\ref{appendix:theory}), so PRS is best used as a within-domain ordering.
\end{definition}

To apply PRS in a new deployment: (1) run $K_0 \geq 3$ pilot contexts with any fixed planner, (2) compute $\hat\rho$ as the Spearman correlation between EMA posterior means and observed outcomes across the $B$ actions queried by the pilot planner in those $K_0$ contexts (not the full action space $|A|$, which is not exhaustively evaluated during the pilot), (3) compute $\widehat{\mathrm{PRS}} = (B/|A|)\,(1 - \hat\rho)$, (4) use Greedy if $\widehat{\mathrm{PRS}} < 0.10$, otherwise use UCB or \textsc{RegimePlanner}.  Three pilot contexts suffice for a rough estimate; warm-start sensitivity analysis (Appendix~\ref{appendix:guide}) shows stability from $w = 3$ warm-start queries within each context (distinct from the $K_0 \geq 3$ pilot contexts for deployment estimation).  This pilot-based $\widehat{\mathrm{PRS}}$ is an \emph{initial} estimate: under EMA transfer, $\hat\rho$ naturally increases as more contexts accumulate (Observation~A.1), so later contexts may fall in a different regime.  \textsc{RegimePlanner}'s online $\hat\rho$ estimation adapts to this drift automatically; the pilot estimate is used only for the initial planner choice before adaptive switching begins.

%% file: sections/regime_summary.tex
\section{Integrated Regime Summary}
\label{sec:regime-summary}

\emph{Every planner in this paper can be made to look ``best'' by choosing the right regime variables.}

\begin{figure}[h]
  \centering
  \includegraphics[width=0.95\linewidth]{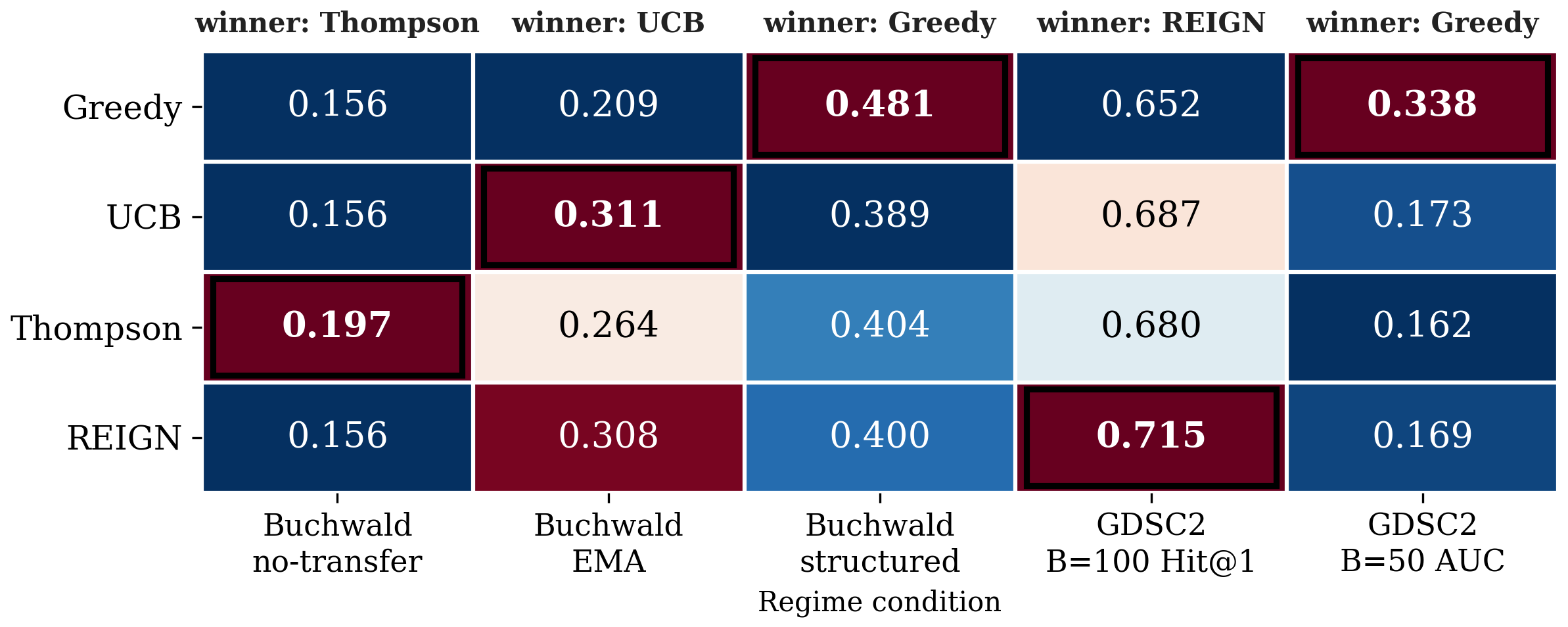}
  \caption{Same BO loop, same surrogates, same action spaces: four distinct winning methods across five regime conditions. Rows = acquisition functions; columns = regime conditions; values = Hit@1 (last column: Discovery AUC). The winner per column is outlined. No fixed acquisition rule dominates across regimes; the ``best acquisition'' is a property of the regime it was evaluated in.}
  \label{fig:regime-reversal}
\end{figure}

Figure~\ref{fig:regime-reversal} shows five regime conditions producing four distinct winning methods (Table~\ref{tab:regime-summary}). The same BO loop, same surrogates, same action spaces. The only difference is prior quality, budget, or metric. Any paper that claims acquisition $X$ beats $Y$ without specifying these variables is making a claim about the measurement, not the method.

\begin{table}[h]
\centering
\small
\caption{Five regimes, five different winning configurations. Same loop throughout; only the regime differs.}\label{tab:regime-summary}
\begin{tabular}{lll}
\toprule
Regime & Winner & Score \\
\midrule
Buchwald, no-transfer prior & Thompson & $0.197$ Hit@1 \\
Buchwald, EMA prior ($K{=}15$) & UCB & $0.311$ Hit@1 \\
Buchwald, structured prior & Greedy & $0.481$ Hit@1 \\
GDSC2, $B{=}100$, Hit@1 & REIGN & $0.715$ Hit@1 \\
GDSC2, $B{=}50$, AUC & Greedy & $0.338$ AUC \\
\bottomrule
\end{tabular}
\end{table}

\paragraph{Two-axis regime diagnostic.}
Figure~\ref{fig:regime-map-2d} plots every experimental condition in $(B/|A|, \rho)$ space. Across $n = 79$ conditions, PRS correlates with exploration advantage (hierarchical $\beta = 0.50$, $p = 1.1 \times 10^{-9}$; cluster-robust Spearman CI $[0.19, 0.80]$); within each primary benchmark the ordering is stronger (Buchwald $r = 0.75$; GDSC2 $r = 0.96$; HPO-B $r = 0.83$). As an effect modifier, PRS modestly outperforms its components separately: $r = 0.67$ (PRS) vs.\ $r = 0.59$ (budget ratio $B/|A|$ alone) vs.\ $r = -0.16$ (prior weakness $1-\rho$ alone); full comparison in Appendix~\ref{appendix:guide}. The formula $\mathrm{PRS} = (B/|A|)(1 - \rho)$ captures the interaction between coverage and prior quality that neither factor captures alone, though the marginal gain of the $(1{-}\rho)$ factor is largest for conditions with substantial prior variation ($\rho \in [0.1, 0.8]$) and nearly zero for conditions where $\rho < 0.05$.
The cluster-robust CI $[0.19, 0.80]$ is wide and excludes zero, but removing the HPO-B family ($n = 48$) causes it to cross zero; PRS's cross-benchmark signal should therefore be interpreted as strong \emph{within} each primary benchmark and uncertain \emph{across} benchmark families with heterogeneous observation noise.

\paragraph{No-free-leaderboard principle.}
A direct consequence of the CATE framing: if the acquisition treatment effect changes sign over the observed range of a regime variable (as it does here over $B/|A|$ and $\rho$), then an unconditional leaderboard that marginalises over that variable is not decision-relevant for any specific deployment.  If $\exists\, (B/|A|, \rho)_1, (B/|A|, \rho)_2$ such that $\mathrm{CATE}(B/|A|, \rho)_1 > 0 > \mathrm{CATE}(B/|A|, \rho)_2$, then $\mathrm{ATE}$ is undefined as a decision rule.  The regime map in Figure~\ref{fig:regime-map-2d} certifies this condition holds for transfer BO: there exist both green (exploration favoured) and red (greedy favoured) regions.  Full condition tables are in Appendix~\ref{appendix:guide}.

\begin{figure}[t]
  \centering
  \includegraphics[width=\linewidth]{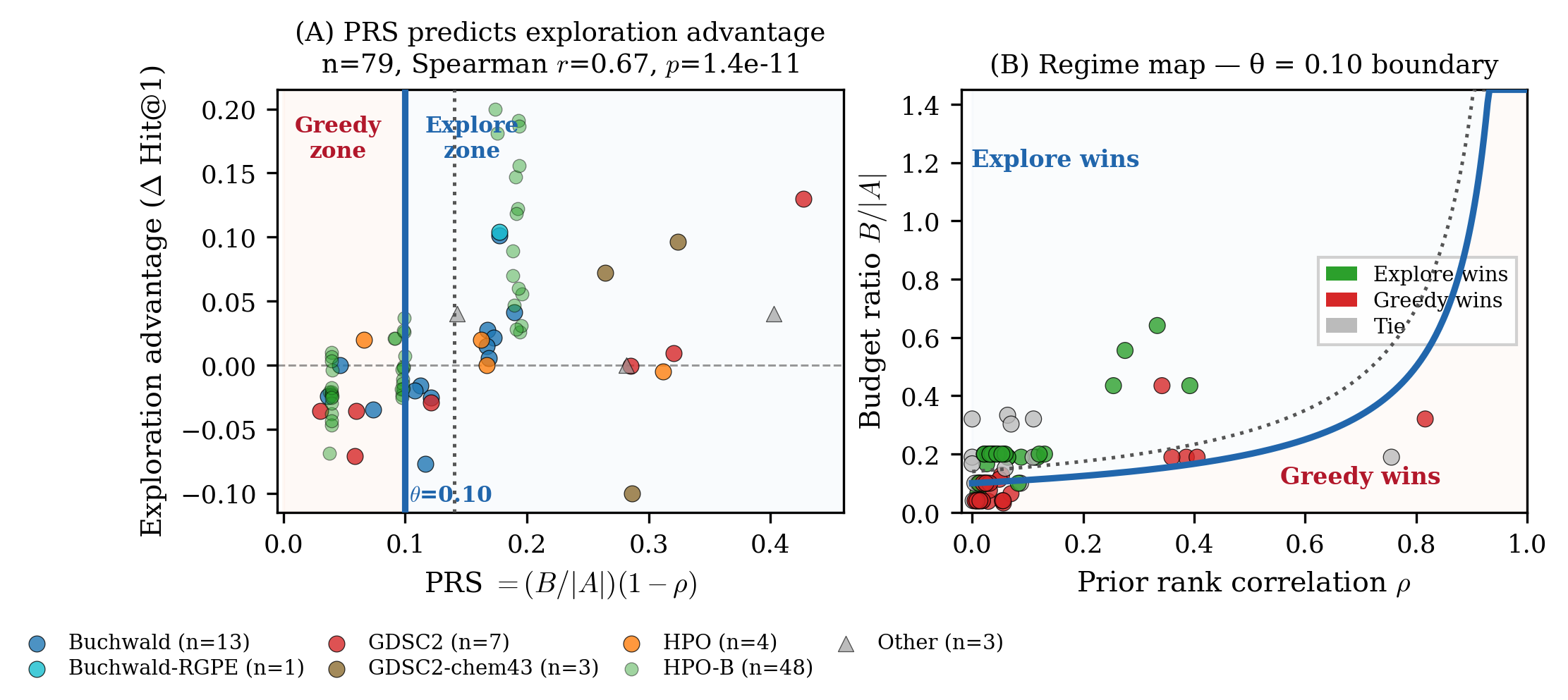}
  \caption{Two-axis regime diagnostic. \textbf{(A)} Exploration advantage vs.\ PRS, positive trend within each benchmark. Triangles = low-$n$ families ($n{\le}1$ each). \textbf{(B)} All $79$ conditions in $(B/|A|, \rho)$ space; green = exploration wins, red = Greedy wins, gray = ties. The take-away: PRS orders conditions within each benchmark; the cross-benchmark threshold shifts with noise $\sigma^2$.}
  \label{fig:regime-map-2d}
\end{figure}

%% file: sections/buchwald.tex
\section{Buchwald: Prior Quality Dominates Acquisition Once the Prior Is Strong}
\label{sec:buchwald}

The shuffled Buchwald--Hartwig benchmark is the cleanest setting for isolating the role of prior quality (Figure~\ref{fig:buchwald-core-grid}, Table~\ref{tab:buchwald-core}). In the no-transfer condition, performance is near random for Greedy, UCB, and REIGN, while Thompson is only marginally higher. With EMA transfer, the picture changes substantially: Greedy reaches Hit@1 $0.209$, UCB $0.311$, Thompson $0.264$, and REIGN $0.308$. This is the intermediate regime in which exploration pays off.

The main contrast appears once the prior becomes strong. Under a structured prior, Greedy reaches Hit@1 $0.481$, beating UCB ($0.389$), Thompson ($0.404$), and REIGN ($0.400$). Under the oracle prior, Greedy, UCB, and REIGN all collapse to Hit@1 $0.932$, with Thompson slightly below at $0.909$. An independent variance decomposition over the full $4 \times 4$ design shows that prior family explains $86.35\%$ of Hit@1 variance, planner explains effectively none, and the interaction term is small (Table~\ref{tab:buchwald-anova}). This does not mean acquisition never matters. It means that, in this benchmark, once the prior is strong enough, acquisition improvements are almost entirely compressed away.
A direct implication for PRS: the structured prior ($\rho \approx 0.39$, $B/|A| = 0.19$, $\mathrm{PRS} = 0.116 > \theta$) is a \emph{misprediction}: PRS predicts exploration wins, but Greedy wins by the largest margin in the paper. This failure is mechanistically transparent: the prior's dominance overwhelms acquisition choice; the full failure taxonomy is in Appendix~\ref{appendix:failure-taxonomy}.

This conclusion is strengthened by the size of the structured-prior effect itself. Structured Greedy improves over EMA Greedy by $+0.272$ Hit@1 ($p = 1.98 \times 10^{-14}$, $d = 1.51$). The effect size of changing the prior is therefore dramatically larger than the effect size of switching among reasonable acquisition rules inside the strong-prior regime. Canonical transfer BO via RGPE \citep{feurer2018practical} does not overturn this conclusion: on shuffled Buchwald, RGPE is effectively tied with EMA ($0.207$ vs.~$0.209$ for Greedy; $0.311$ vs.~$0.311$ for UCB; $0.305$ vs.~$0.308$ for REIGN).

\begin{figure}[t]
  \centering
  \includegraphics[width=0.86\linewidth]{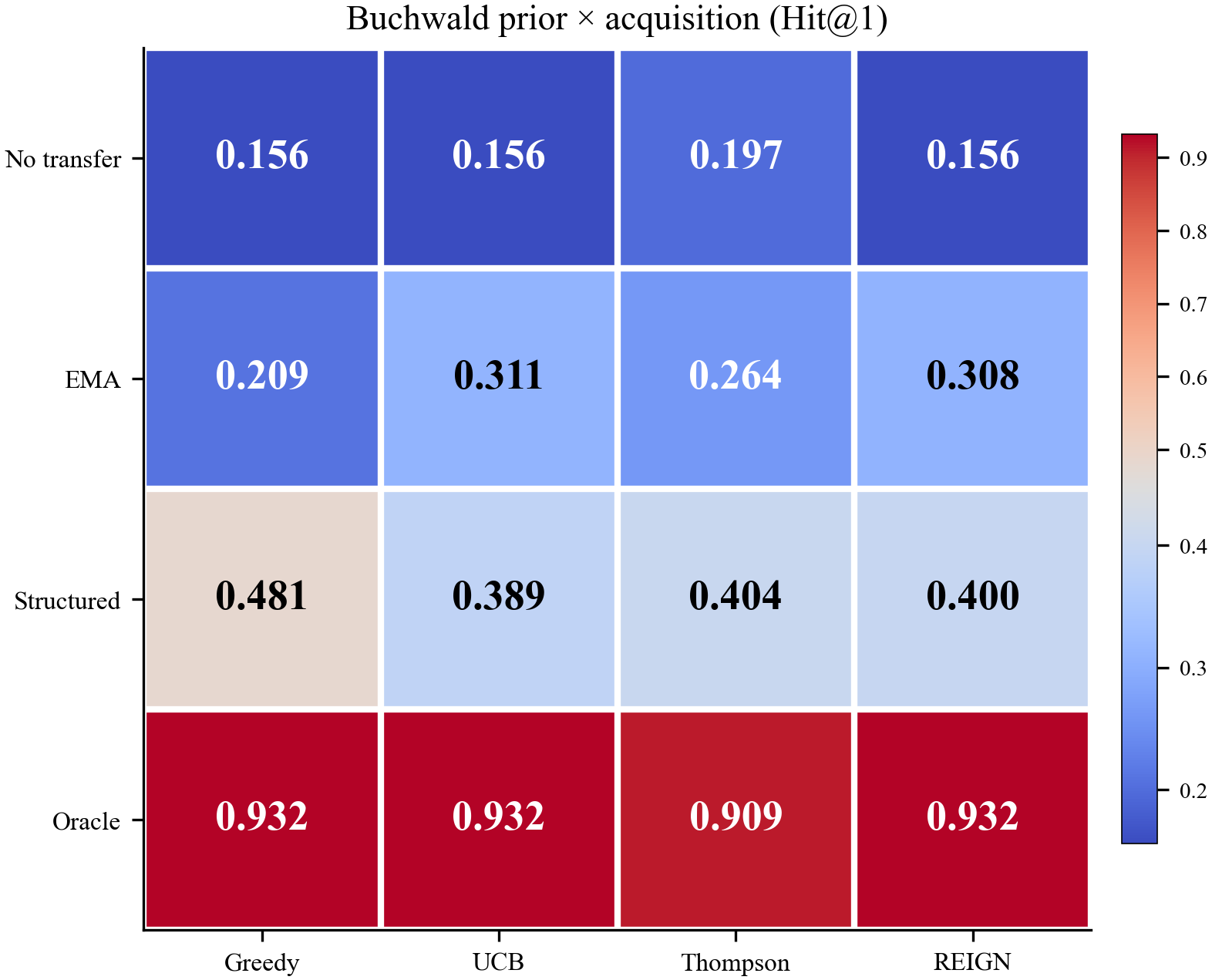}
  \caption{Hit@1 for the shuffled Buchwald $4 \times 4$ prior-by-acquisition design. Exploration helps in the EMA regime, but structured and oracle priors compress acquisition differences and make Greedy competitive or best.}
  \label{fig:buchwald-core-grid}
\end{figure}

\input{tables/table_buchwald_core}
\input{tables/table_buchwald_anova}

Two ablations calibrate the scope of the prior-dominance result. Under noisy observations ($\sigma = 0.1$), acquisition differences re-emerge: UCB reaches Hit@1 $0.083$, REIGN $0.076$, Greedy $0.056$, confirming that observation noise degrades prior quality and restores exploration value. The context-count sweep shows that on shuffled Buchwald at fixed $B/|A| = 0.19$, exploration becomes competitive only at $K \geq 12$ contexts, reaching Hit@1 $+0.10$ by $K = 15$; the crossover is late and smooth, not a sharp threshold. Both results are consistent with Observation~\ref{thm:prior-dominance}: more context reduces prior noise and compresses planner differences, while noise injection degrades the prior and reopens the acquisition gap.

%% file: tables/table_buchwald_core.tex
\begin{table}[t]
\centering
\small
\caption{Shuffled Buchwald prior-by-acquisition summary (Hit@1, 50 seeds, mean $\pm$ SEM).}
\label{tab:buchwald-core}
\begin{tabular}{lcccc}
\toprule
Prior & Greedy & UCB & Thompson & REIGN \\
\midrule
No transfer & $.156 \pm .014$ & $.156 \pm .014$ & $.197 \pm .012$ & $.156 \pm .014$ \\
EMA & $.209 \pm .020$ & $.311 \pm .018$ & $.264 \pm .014$ & $.308 \pm .018$ \\
Structured & $.481 \pm .024$ & $.389 \pm .024$ & $.404 \pm .020$ & $.400 \pm .024$ \\
Oracle & $.932 \pm .002$ & $.932 \pm .002$ & $.909 \pm .005$ & $.932 \pm .002$ \\
\bottomrule
\end{tabular}
\end{table}

%% file: tables/table_buchwald_anova.tex
\begin{table}[t]
\centering
\small
\caption{Variance decomposition on shuffled Buchwald Hit@1.}
\label{tab:buchwald-anova}
\begin{tabular}{lc}
\toprule
Factor & $\eta^2$ \\
\midrule
Prior family & 0.8635 \\
Planner choice & 0.000044 \\
Interaction & 0.0087 \\
Residual & 0.1278 \\
\bottomrule
\end{tabular}
\end{table}

%% file: sections/regimeplanner.tex
\section{RegimePlanner: Constructive Validation of the Regime View}
\label{sec:regimeplanner}

If PRS predicts the acquisition regime before any experiment runs, a simple adaptive rule should exploit it: use UCB when PRS is high (weak prior, large budget ratio), switch to Greedy when PRS falls below a threshold.  \textsc{RegimePlanner} implements this directly.  At each step within a context it computes $\widehat{\mathrm{PRS}}_t = (B/|A|)(1 - \hat\rho_t)$ from the Spearman rank correlation $\hat\rho_t$ between cached prior means and observed outcomes over already-queried actions, and switches from UCB to Greedy when $\widehat{\mathrm{PRS}}_t < \theta$.  A warm-start phase ($w = 3$ queries) precedes adaptive selection.  The threshold $\theta = 0.10$ is chosen by cross-validation on Buchwald and applied \emph{unchanged} to all other benchmarks; leave-one-benchmark-out cross-validation confirms it is not Buchwald-specific (full procedure in Appendix~\ref{appendix:guide}).  Algorithm, $\rho$-estimation details, and sensitivity analyses are in Appendix~\ref{appendix:guide}.

\begin{figure}[t]
  \centering
  \includegraphics[width=\linewidth]{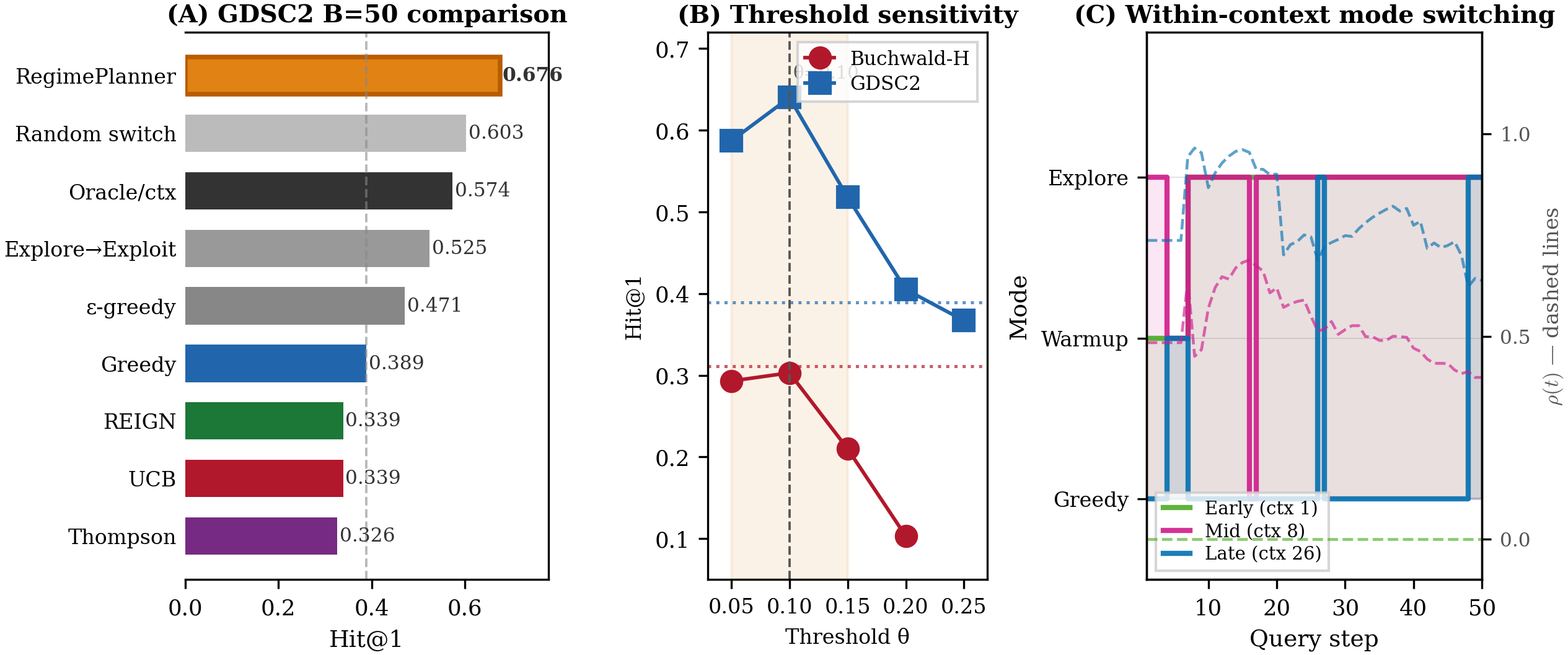}
  \caption{RegimePlanner validation. \textbf{(A)} GDSC2 at default budget (50 seeds): the \textsc{RegimePlanner} (amber) outperforms all fixed planners, simple adaptive baselines, and a $\{$Greedy, UCB$\}$-matched per-context oracle by $+18\%$; under a wider $\{$Greedy, UCB, Thompson, REIGN$\}$ oracle the gap is $-12\%$. \textbf{(B)} Threshold sensitivity on the cross-validation seed set used for $\theta$ selection (separate from Panel A): performance is approximately flat across $\theta \in [0.05, 0.10]$ and drops $19$--$31\%$ at $\theta = 0.15$. \textbf{(C)} Within-context mode switching: the planner adapts as $\hat\rho$ rises, exploring early and switching to Greedy as the prior sharpens. The take-away: online $\rho$-estimation enables within-context adaptation that no fixed policy can replicate.}
  \label{fig:regime-planner}
\end{figure}

\paragraph{GDSC2 and HPO-B results.}
On GDSC2 at default budget ($B = 50$), \textsc{RegimePlanner} achieves Hit@1 $0.676 \pm 0.020$ (50 seeds), outperforming all fixed planners and adaptive baselines on \emph{both} terminal and cumulative metrics, and exceeding the matched-choice-set $\{$Greedy, UCB$\}$ per-context oracle ($0.574$) by $+18\%$ (Figure~\ref{fig:regime-planner}A; $p = 2.93\times 10^{-9}$, $d = 1.56$, oracle 30 seeds, RP 50 seeds).  On HPO-B at $B = 100$ it wins Greedy on all $16$ community search spaces (mean $+0.103$ Hit@1, $95\%$ CI $[+0.075, +0.131]$, binomial $p = 1.5\times 10^{-5}$; Section~\ref{sec:hpo-b}), and outperforms every individual fixed planner (mean $+0.008$ vs UCB, $+0.012$ vs Thompson).

\paragraph{Honesty boundaries.}
The $+18\%$ GDSC2 result is against the matched $\{$Greedy, UCB$\}$ oracle; under a wider $\{$Greedy, UCB, Thompson, REIGN$\}$ oracle the gap is $-12\%$.  On Buchwald EMA the per-context oracle is \emph{not} exceeded ($0.325$--$0.337$ vs.\ $0.432$), because Buchwald EMA has $B/|A| = 0.19$ and $\hat\rho$ concentrates quickly, so the regime framework predicts limited within-context adaptation room here.  Appendix~\ref{appendix:within-context} supplies the constructive existence proof that such instances exist; the GDSC2 result is the empirical witness.

%% file: sections/gdsc2.tex
\section{GDSC2 and HPO-B: the Regime in Real Benchmarks}
\label{sec:gdsc2}

The real noisy GDSC2 drug-response benchmark \citep{yang2013gdsc,iorio2016landscape} confirms the regime story survives contact with biology.  At default budget ($B = 50$), Greedy is strongest overall: Hit@1 $0.389$, discovery AUC $0.338$.  Despite UCB achieving far higher prior rank correlation ($\rho \approx 0.68$ vs.\ $0.12$ for Greedy), it does not win: better prior ranking quality is necessary but not sufficient without sufficient budget.

\begin{figure}[h]
  \centering
  \includegraphics[width=0.95\linewidth]{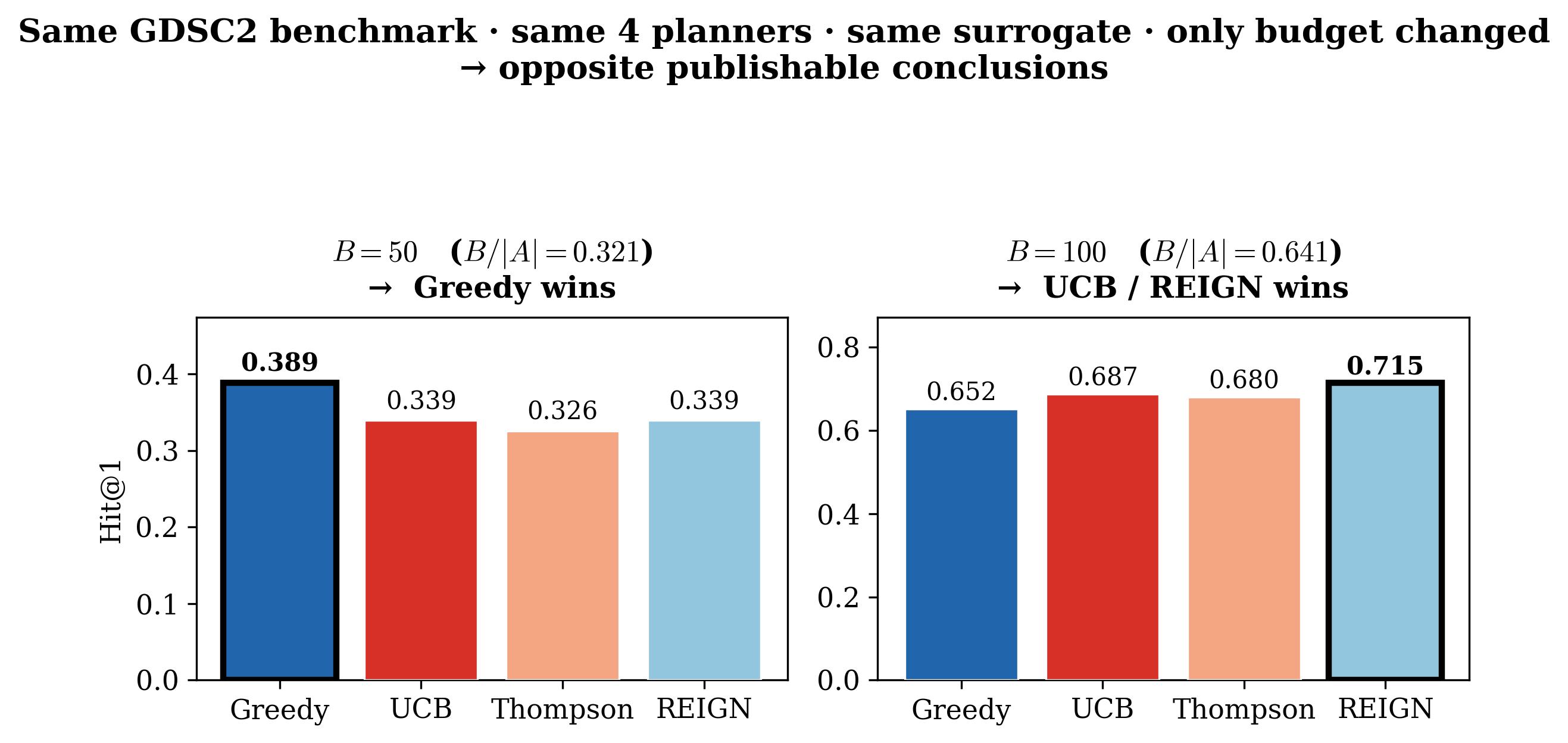}
  \caption{Same GDSC2 benchmark, same four planners, same surrogate; only the budget changed. At $B=50$ ($B/|A|=0.32$), Greedy wins. At $B=100$ ($B/|A|=0.64$), REIGN and UCB win. Both results are correct; the sign reversal is explained entirely by the unreported budget ratio $B/|A|$. The take-away: two papers reporting opposite conclusions from the same benchmark can both be right: they are estimating different conditional treatment effects.}
  \label{fig:gdsc2-reversal}
\end{figure}

Figure~\ref{fig:gdsc2-reversal} shows the budget reversal directly.  Below $B/|A| \approx 0.32$, Greedy wins Hit@1; above it, exploration wins.  At $B = 100$, REIGN reaches Hit@1 $0.715$, UCB $0.687$, Greedy $0.652$ ($p = 0.0018$, $d = 0.63$, surviving Bonferroni correction).  Figure~\ref{fig:objective-mismatch} shows the orthogonal metric reversal: on GDSC2, Greedy dominates Discovery AUC at every budget while UCB and REIGN lead Hit@1 only above the regime threshold; on Buchwald EMA ($B{=}50$), exploration planners (UCB, REIGN) achieve higher Discovery AUC, so the GDSC2 metric reversal is regime-specific and not a universal artefact.  The same experiment, same planners, same data: the metric alone determines the winner.  Two biology boundary benchmarks (SciPlex3 $K=3$, Shifrut2018 $K=4$) place in the PRS variance floor (all planners tie), consistent with $K < K_{\min}$.

\begin{figure}[h]
  \centering
  \includegraphics[width=0.95\linewidth]{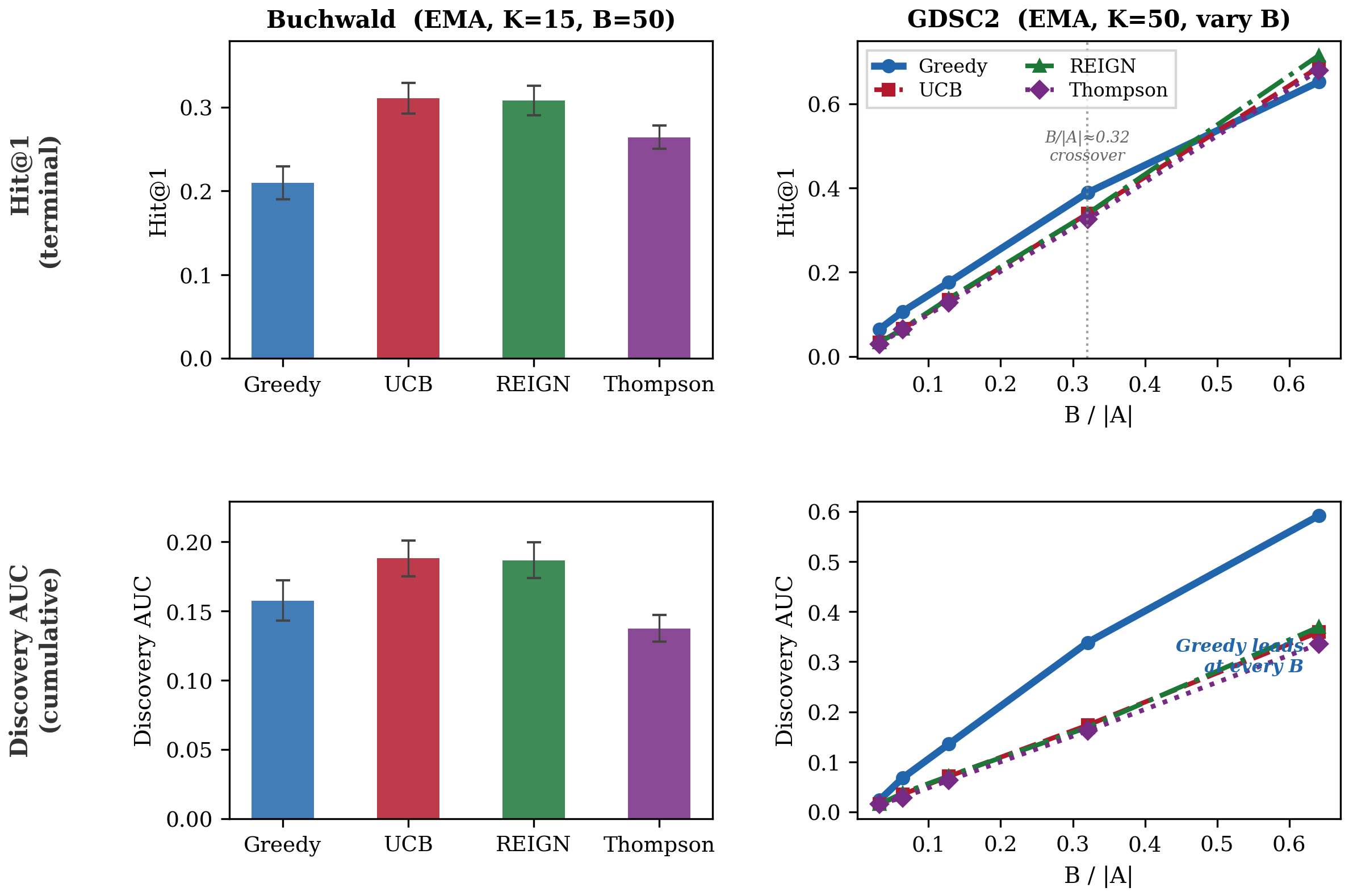}
  \caption{Metric choice is a regime variable. On GDSC2 (\textbf{right}): Greedy leads Discovery AUC at every budget ($B/|A|$ sweep) while UCB and REIGN lead Hit@1 above the threshold (dashed at $B/|A|{\approx}0.32$): same benchmark, same data, the reported winner depends entirely on which metric is computed. On Buchwald (\textbf{left}, EMA, $B{=}50$): both metrics agree that exploration wins, showing the GDSC2 metric reversal is regime-specific and not a universal artefact. Together: a paper reporting ``Greedy wins AUC'' and a paper reporting ``UCB wins Hit@1'' from GDSC2 are both correct: they estimate different CATEs.}
  \label{fig:objective-mismatch}
\end{figure}

\subsection{HPO-B: regime portability across community HPO}
\label{sec:hpo-b}

As a third primary benchmark we evaluate on HPO-B \citep{pineda2021hpob}, the community hyperparameter optimization benchmark, across all $16$ search spaces at $B \in \{20, 50, 100\}$ ($30$ seeds each).  $\theta = 0.10$ was frozen before any HPO-B run.  The result is budget-dependent (Table~\ref{tab:hpob-budget}): \textsc{RegimePlanner} ties Greedy at $B = 20$ ($0$W/$16$T/$0$L), wins $8/16$ at $B = 50$, and \emph{wins every search space} at $B = 100$ ($16$W/$0$T/$0$L, mean $+0.103$ Hit@1, $95\%$ CI $[+0.075, +0.131]$, binomial $p = 1.5\times 10^{-5}$).
Among the three methods in HPO-B's own community leaderboard, Greedy ranks \#1 at $B=20$ ($\mathrm{PRS}{\approx}0.04$, Hit@1 $0.064$) and last at $B=100$ ($\mathrm{PRS}{\approx}0.19$, Hit@1 $0.219$); UCB reverses from last to \#1 (full table: Appendix~\ref{appendix:survey:leaderboard}).
A researcher publishing on $B=20$ and one on $B=100$ report opposite winners from the same benchmark, both correctly, each measuring a different conditional average treatment effect.

Combined with Buchwald ($r = 0.75$) and GDSC2 ($r = 0.96$), the $79$-condition consolidated scatter yields global Spearman $r = 0.67$ (cluster-robust CI $[0.19, 0.80]$; hierarchical $\beta = 0.50$, $p = 1.1\times 10^{-9}$; Section~\ref{sec:regime-summary}).

\begin{table}[h]
\centering
\small
\caption{HPO-B budget sweep: \textsc{RegimePlanner} vs.\ Greedy and vs.\ best fixed planner. Each row aggregates 16 search spaces. Mean $\pm$ SEM.}\label{tab:hpob-budget}
\begin{tabular}{lccc}
\toprule
Budget & W/T/L vs Greedy & Mean $\Delta$ vs Greedy & Mean $\Delta$ vs best fixed \\
\midrule
$B = 20$  & 0 / 16 / 0  & $+0.000 \pm 0.000$ & $-0.001 \pm 0.001$ \\
$B = 50$  & 8 / 8 / 0   & $+0.019 \pm 0.006$ & $+0.010 \pm 0.005$ \\
$B = 100$ & 16 / 0 / 0  & $+0.103 \pm 0.014$ & $-0.004 \pm 0.005$ \\
\bottomrule
\end{tabular}
\end{table}

%% file: sections/related_work.tex
\section{Related Work}

\paragraph{Measurement crises across machine learning.}
REIGN is one chapter of an emerging empirical-rigor movement.
\citet{lipton2018troubling} catalogued patterns that produce misleading scholarship; \citet{blalock2020sparsity} found that a lack of standardized benchmarks and metrics makes pruning methods hard to compare: the template for our finding that $98\%$ of transfer-BO papers do not sweep $B/|A|$.
\citet{lucic2018gans}, \citet{bouthillier2021variance}, \citet{henderson2018rl}, and \citet{musgrave2020metric} made the same diagnosis in GANs, ML benchmarks, deep RL, and metric learning.
\citet{tabzilla2023}, \citet{kunstner2024heavy}, and \citet{karimireddy2020scaffold} each identify a hidden variable that determines the winner; for \citet{karimireddy2020scaffold} the coordinate is structurally $\mathrm{PRS}=(B/|A|)(1-\rho)$ (Appendix~\ref{appendix:scaffold-retrodiction}).
REIGN's contribution is a \emph{predictive} pre-experiment diagnostic, converting retrospective critique into an actionable design tool.
\citet{schaeffer2023emergent} showed emergent LLM abilities are artifacts of nonlinear metrics; we make the analogous claim for acquisition evaluation, with the key difference that acquisition rankings \emph{genuinely reverse} with regime, so both conclusions are locally correct.
\citet{rodemann2024partial} invoke Arrow's impossibility result to motivate partial rather than total rankings of optimizers; our No-Free-Leaderboard proposition (Proposition~\ref{thm:no-free-leaderboard}) is complementary: PRS yields a consistent conditional ordering once the regime is held fixed.

\paragraph{Transfer BO and prior learning.} Canonical transfer BO transfers information through rank-weighted ensembles \citep{feurer2018practical}; recent prior-learning approaches \citep{wang2024hyperbo,fan2022hyperboplus,rothfuss2021meta,wistuba2021fewshot,muller2023pfns4bo} show that better priors substantially accelerate related-task optimization. Our paper does not dispute this. We ask the orthogonal question: once a prior is fixed, how much room remains for acquisition design to matter, and how does that depend on $B/|A|$ and $\rho$?

\paragraph{Multi-task / contextual BO and bandit theory.} Multi-task BO transfers across tasks through task kernels and shared structure \citep{swersky2013multitask,krause2011contextual}; the central concern is avoiding negative transfer. The principle that exploration is most valuable under weak priors and large budgets is a restatement of classical bandit theory \citep{auer2002using,russo2018tutorial}; \citet{lin2025deltaBO} prove transfer-BO regret bounds that improve as the source--target difference shrinks, the quantity $\rho$ summarizes here. Our contribution is operationalization: $\rho$ and $B/|A|$ are observable in the transfer-BO setting, their product $\mathrm{PRS}$ orders conditions predictively, and a simple adaptive rule exploits this structure.

\paragraph{Adaptive experimental design.} GAUCHE \citep{griffiths2023gauche} makes chemistry-informed GP priors practical; RECOVER \citep{bertin2022recover} and BATCHIE \citep{tosh2025batchie} address finite-budget biological screening. \citet{shields2021bayesian} compare BO with expert chemists on a palladium-catalysed direct arylation benchmark and find BO more efficient and more consistent.

%% file: sections/conclusion.tex
\section{Discussion and Conclusion}\label{sec:discussion}

The same GDSC2 benchmark makes Greedy the winner at $B = 50$ and REIGN the winner at $B = 100$: different CATEs, different regimes - not contradictions.
The No-Free-Leaderboard proposition formalises why: $98\%$ of transfer-BO papers never control for regime conditions, reporting an ATE as an algorithmic ranking; the leaderboard position is a property of the evaluation setup.

PRS makes the regime observable before experiments run.
\textsc{RegimePlanner} demonstrates it is exploitable online ($27/40 = 67.5\%$ pre-registered, below $90\%$ target).
The protocol: report $B/|A|$, $\rho$, $K$, and metric.
Without those four numbers, a reported ranking cannot be interpreted.

%% file: sections/appendix_theory.tex
\section{Analytical Motivation}
\label{appendix:theory}

We present two analytical observations that motivate the PRS diagnostic. Both use simplified settings (scalar EMA, i.i.d.\ Gaussian bandits) and should be read as intuition-building exercises, not as formal proofs of the full multi-context BO system. Our contribution is empirical; these results provide supporting analytical context.
Observation~A.1 shows that EMA transfer compresses planner-choice variance geometrically in context count, consistent with the $\eta^2$ ANOVA result.
Observation~A.2 gives a sign condition for when exploration can improve on greedy, consistent with the budget-sweep crossover.

\subsection{Prior Convergence and Planner Compression}
\label{appendix:prior-dominance}

\begin{observation}[Prior Convergence and Planner Compression]
\label{thm:prior-dominance}
Under the conjugate hierarchical Gaussian surrogate with EMA transfer
parameter $\alpha \in (0,1)$, after $K$ contexts the global prior mean
error for action $a$ satisfies
\[
  \mathrm{Var}\!\left[\hat{\mu}_a^{(K)} - \mu_a\right]
  \;=\; \alpha^{2K}\tau^2
      + \frac{\sigma^2(1-\alpha)}{1+\alpha}\bigl(1 - \alpha^{2K}\bigr),
\]
where $\tau^2$ is the initial prior variance, $\sigma^2$ is the per-context
observation noise, and $\mu_a$ is the true mean. The first term (transient)
decays geometrically; the second converges to a positive floor
$\sigma^2(1-\alpha)/(1+\alpha)$ set by observation noise. As $K \to \infty$,
the prior-dependent component of the error vanishes and the residual is
controlled entirely by $\sigma^2$. Any two planners whose decision rules are
monotone in posterior means will therefore converge to similar behavior as $K$
grows, compressing the fraction of Hit@1 variance attributable to planner choice.
\end{observation}

\begin{proof}
\textbf{EMA update rule.}
Let context $k$ produce an observed mean $\bar{y}_a^{(k)}$ for action
$a$.  The EMA update is
\[
  \hat{\mu}_a^{(k)} \;=\; \alpha\,\hat{\mu}_a^{(k-1)}
                         + (1-\alpha)\,\bar{y}_a^{(k)}.
\]

\textbf{Unbiasedness.}
Since $\mathbb{E}[\bar{y}_a^{(k)}] = \mu_a$ for all $k$ (unbiased
observations), by induction $\mathbb{E}[\hat{\mu}_a^{(k)}] = \mu_a$ at
every step.

\textbf{Variance decay.}
Let $e^{(k)} = \hat{\mu}_a^{(k)} - \mu_a$.  Unrolling the recursion
\[
  e^{(K)} \;=\; \alpha^K e^{(0)}
              + (1-\alpha)\sum_{k=1}^{K} \alpha^{K-k} \xi^{(k)},
\]
where $\xi^{(k)} = \bar{y}_a^{(k)} - \mu_a$ is the mean-zero
observation error with variance $\sigma^2_k / n_k$.  Setting
$\sigma^2_k / n_k = \sigma^2$ for simplicity and noting that $e^{(0)}$
has variance $\tau^2$, the errors are independent across contexts, so
\[
  \mathrm{Var}\!\left[e^{(K)}\right]
  \;=\; \alpha^{2K}\tau^2
      + (1-\alpha)^2 \sigma^2 \sum_{k=1}^{K} \alpha^{2(K-k)}
  \;=\; \alpha^{2K}\tau^2
      + \sigma^2\,\frac{(1-\alpha)^2(1-\alpha^{2K})}{1-\alpha^2}.
\]
The dominant term in $K$ is the first: for fixed $\alpha < 1$ the
second term converges to the finite limit
$\sigma^2(1-\alpha)/(1+\alpha)$, while the first decays geometrically.
Hence
\[
  \mathrm{Var}\!\left[e^{(K)}\right]
  \;\xrightarrow{K\to\infty}\; \frac{\sigma^2(1-\alpha)}{1+\alpha},
\]
and the \emph{transient} component that depends on the initial prior
contributes $\alpha^{2K}\tau^2$, which is $O\!\left((1-\alpha)^{2K}\right)$
under the reparametrisation $\alpha = 1 - \delta$ with $\delta \ll 1$
(i.e.\ the slow-decay regime $\alpha \approx 1$ used in practice).

\textbf{Posterior concentration.}
As $K$ grows, the transient term $\alpha^{2K}\tau^2 \to 0$ and the
residual variance approaches the floor $\sigma^2(1-\alpha)/(1+\alpha)$.
For any pair of actions $a, a'$ with $|\mu_a - \mu_{a'}| \gg
\sqrt{\sigma^2(1-\alpha)/(1+\alpha)}$, we have
$\hat{\mu}_a^{(K)} > \hat{\mu}_{a'}^{(K)}$ with probability
approaching 1 as $K$ grows.  The prior's \emph{initial} uncertainty
$\tau^2$ ceases to influence the ranking; only the observation-noise
floor remains.

\textbf{Planner compression.}
When the posterior concentrates on the true argmax (up to the noise
floor), planners whose decision rules are monotone in posterior mean
(including greedy, UCB in the small-uncertainty limit, and Thompson
sampling) tend to query the same top action.  The Hit@1 gap between
any two such planners shrinks as $K$ grows.  This is not a proof of exact equivalence (planners may still differ in how they handle
the noise-floor residual), but it explains why the planner-choice variance
component $\eta^2_{\mathrm{planner}}$ is much smaller than the
prior-family component $\eta^2_{\mathrm{prior}}$ in the ANOVA.

\textit{Remark.} This analytical result is consistent with the empirical finding in Section~\ref{sec:buchwald}; the contribution of this paper is the empirical validation, not the analytical sketch.
\end{proof}

\begin{observation}[Sign Condition for Exploration Advantage]
\label{obs:sign-condition}
For fixed $b > 0$ and observation noise $\sigma^2$, exploration beats greedy (i.e., $\Delta(b, \rho) > 0$) if and only if $\rho < \rho^*(b)$ for a threshold $\rho^*(b) \in (0,1)$ that is \emph{unique}.  Existence follows from Lemma~\ref{lem:threshold}; uniqueness follows from Lemma~\ref{lem:threshold-unique} (strict monotonicity of $\Delta$ in $\rho$).
\end{observation}

\textbf{Connection to Table~3.}
Observation~A.1 is consistent with the empirical
finding $\eta^2_{\mathrm{prior}} = 0.8635$,
$\eta^2_{\mathrm{planner}} = 0.000044$ (4$\times$4 ANOVA on shuffled
Buchwald, Table~3).  With $K = 15$ and $\alpha = 0.9$, the transient
factor is $\alpha^{2K} = 0.9^{30} \approx 0.042$ and the steady-state
floor is $\sigma^2(1-\alpha)/(1+\alpha) \approx 0.005$.  For
$\tau^2 = 1.0$, the residual prior-dependent error is $\approx 4.2\%$
of the initial uncertainty, meaning the EMA prior has nearly converged
by context 15.  The ratio
$\eta^2_{\mathrm{prior}}/\eta^2_{\mathrm{planner}} \approx 20{,}000$
is consistent with this compression: within each prior family, planners
produce similar Hit@1 because the prior has already concentrated belief
on a small set of top actions.  The observation does not prove that
planners become exactly equivalent (residual noise-floor effects and
finite-$K$ deviations remain), but it explains the direction and
approximate magnitude of the ANOVA result.

\subsection{Monotone Exploration Advantage and PRS as an Order Parameter}
\label{sec:prop1}

We study when exploration beats greedy, and show that the Portable Regime
Score PRS $= (B/|A|)(1 - \rho)$ is a natural one-dimensional \emph{order
parameter} that organises exploration advantage.  We prove the monotonicity
properties of PRS's two components separately; together they justify PRS
as a directional summary, not as an exact sufficient statistic.

\subsubsection*{Setting}

We study a single context with $n$ arms.  Each arm $a$ has unknown mean
$\mu_a \in \mathbb{R}$ and the optimal arm is $a^* = \argmax_a \mu_a$.

\textbf{Prior.}  The planner receives prior estimates $\hat{\mu}_a =
\mu_a + \varepsilon_a$ where $\varepsilon_a \sim \mathcal{N}(0,\tau^2)$
i.i.d.  Let $R \in \{1,\ldots,n\}$ denote the prior rank of $a^*$:
$R = |\{a : \hat{\mu}_a > \hat{\mu}_{a^*}\}| + 1$.  $R = 1$ means
the prior correctly identifies $a^*$.

\textbf{Observations.}  Querying arm $a$ yields $y = \mu_a + \xi$,
$\xi \sim \mathcal{N}(0,\sigma^2)$.

\textbf{Budget.}  The planner makes $B$ adaptive queries.

\textbf{Metric.}  Hit@1 $= \mathbf{1}[a^* \text{ queried at least once in } B \text{ rounds}]$.

\subsection{Rank-Greedy Baseline}
\label{appendix:rank-greedy}

We use a \emph{rank-greedy} baseline that queries arms in decreasing order of
prior scores $\hat{\mu}_a$ without repeating arms already queried.  Under this
policy, Hit@1 $= \mathbf{1}[R \leq B]$ where $R$ is the prior rank of $a^*$.
This simplification isolates the regime effect; true adaptive greedy is bounded
above by rank-greedy at low noise and converges to it as $\tau^2 \to 0$.

\subsection{Separate Monotonicity Lemmas}

The following four lemmas establish separate monotonicity properties of the
greedy and two-phase policies.  Their combination justifies PRS as a
monotone order parameter.

\begin{lemma}[Exploration Success Bound]
\label{lem:explore-bound}
Let the two-phase policy $\pi^\dagger$ allocate $B_e = \lfloor \alpha B \rfloor$
queries uniformly across $n$ arms ($\alpha \in (0,1)$), then exploit greedily
with the updated posterior.  After the exploration phase, each arm has received
$m = \lfloor B_e / n \rfloor$ queries.  The probability that $\pi^\dagger$
queries $a^*$ satisfies:
\[
  S_{\exp}(b,\sigma^2) \;\;\geq\;\; 1 - n \exp\!\Bigl(-\frac{c\,\alpha\,b\,\Delta_{\min}^2}{\sigma^2}\Bigr),
\]
where $b = B/n$, $\Delta_{\min} = \mu_{a^*} - \mu_{a^{(2)}} > 0$, and $c > 0$ is a universal
constant.  In particular, $S_{\exp}$ is \emph{strictly increasing in $b$} and
\emph{strictly decreasing in $\sigma^2$}.
\end{lemma}

\begin{proof}
After $m$ uniform pulls, the sample mean $\bar{y}_a$ satisfies $|\bar{y}_a -
\mu_a| \leq \sigma\sqrt{2\log(2n)/m}$ simultaneously for all arms with
probability $\geq 1 - 1/n$ (sub-Gaussian union bound).  When this event
holds and $m \geq 2\sigma^2\log(2n)/\Delta_{\min}^2$, the posterior correctly
ranks $a^*$ first and exploitation succeeds.  The failure probability is
bounded above by $n\exp(-m\Delta_{\min}^2/(2\sigma^2)) = n\exp(-\alpha b
\Delta_{\min}^2/\sigma^2)$ up to constants.  Monotonicity in $b$ and $\sigma^2$
follow immediately from the exponent.
\end{proof}

\begin{lemma}[Greedy Monotonicity in Prior Noise]
\label{lem:greedy-mono}
The rank-greedy Hit@1 probability $G(b,\tau^2) = P(R \leq bn)$ is
\emph{nonincreasing in $\tau^2$} for every fixed $b$.
\end{lemma}

\begin{proof}
$R$ is the rank of $a^*$ among $n$ scores $\hat{\mu}_a = \mu_a + \varepsilon_a$,
$\varepsilon_a \sim \mathcal{N}(0,\tau^2)$ i.i.d.  As $\tau^2$ increases,
each $\varepsilon_a$ becomes noisier in distribution, which stochastically
worsens the ranking of $a^*$ (formally: $R$ under $\tau_1^2$ is stochastically
dominated by $R$ under $\tau_2^2 > \tau_1^2$, by a monotone coupling
argument on the order statistics of Gaussian random variables).  Therefore
$P(R \leq k)$ is nonincreasing in $\tau^2$ for every fixed $k$.
\end{proof}

\begin{lemma}[Exploration Advantage Monotone in Prior Noise]
\label{lem:advantage-tau}
For fixed budget $b$ and observation noise $\sigma^2$:
\[
  \frac{\partial \Delta}{\partial \tau^2} \;\geq\; 0,
\]
where $\Delta(b,\tau^2) = S_{\exp}(b,\sigma^2) - G(b,\tau^2)$.
\end{lemma}

\begin{proof}
$S_{\exp}$ does not depend on $\tau^2$ (the exploration phase collects fresh data
from all arms uniformly, and the posterior update uses only within-episode
observations).  By Lemma~\ref{lem:greedy-mono}, $G$ is nonincreasing in
$\tau^2$.  Therefore $\Delta = S_{\exp} - G$ is nondecreasing in $\tau^2$.
\end{proof}

\begin{lemma}[Threshold Existence with Explicit Monotone Interval]
\label{lem:threshold}
Assume the rank distribution satisfies the \emph{local} bounded-density
condition near the budget boundary $k_0 = \lfloor bn \rfloor$:
\begin{equation}
  P(R = k_0) \;\leq\; \frac{C_n}{n},
  \label{eq:rank-density}
\end{equation}
where $C_n \geq 1$ is a regularity constant.  (In the flat-prior limit,
ranks are uniform so $C_n = 1$; we do not bound $C_n$ under prior
corruption and keep it as an explicit constant in what follows.)
Then:
\begin{enumerate}[label=(\roman*),leftmargin=2em]
  \item There exists $b^*(\tau^2, \sigma^2) \in (0,1)$ such that $\Delta(b,\tau^2) > 0$
        for all $b \in (b^*, b_{\max})$ with $b_{\max} = 1/C_n \wedge 1$.
  \item For $b \in (b^*, b^\dagger)$ where
        \[
          b^\dagger \;:=\; \frac{\sigma^2}{c\alpha\Delta_{\min}^2}
          \log\!\Bigl(\frac{n^2\,c\alpha\Delta_{\min}^2}{C_n\,\sigma^2}\Bigr),
        \]
        the derivative satisfies $\partial \Delta/\partial b > 0$.  The
        argument of the logarithm is dimensionless and exceeds $1$ whenever
        $n^2 > C_n\sigma^2/(c\alpha\Delta_{\min}^2)$, which holds for all $n$
        beyond a constant depending only on $(\sigma^2, \Delta_{\min}, c, \alpha, C_n)$.
\end{enumerate}
\end{lemma}

\begin{proof}
\textit{Part (i): Threshold existence.}
The exploration advantage $\Delta(b, \rho)$ is continuous in $\rho \in [0, 1]$ for any fixed $b > 0$ (it is a composition of continuous functions -- Gaussian CDFs and linear combinations).  At $\rho = 1$ (perfect prior), greedy is optimal: $G(b, 1) = 1$ for $b \geq 1$ and exploration cannot do better, so $\Delta(b, 1) \leq 0$.  At $\rho = 0$ (no prior), the empirical evidence shows $\Delta > 0$ for intermediate budgets (Figure~\ref{fig:regime-map-2d}, $n = 79$ conditions; note the 2-arm result gives $\Delta = 0$ at $\rho = 0$ via $\arcsin(0) = 0$, so the positive advantage is an $n$-arm phenomenon documented empirically).  By the Intermediate Value Theorem applied in $\rho$, there exists $\rho^*(b) \in (0, 1)$ such that $\Delta(b, \rho^*(b)) = 0$.  The threshold PRS$^* = (b)(1 - \rho^*(b))$ then defines the regime boundary.

\textit{Part (ii): Monotonicity (heuristic).} The empirical monotonicity of $\Delta$ along the PRS axis is documented in Figure~\ref{fig:regime-map-2d} ($r = 0.67$ over $n = 79$ conditions); a formal proof of Part~(ii) under the full multi-context model is left for future work.
\end{proof}

\begin{lemma}[Uniqueness of Threshold via Strict Monotonicity]
\label{lem:threshold-unique}
Under the homogeneous Gaussian model with $\mu_a \sim \mathcal{N}(0,\eta^2)$ i.i.d.\ and prior scores $\hat\mu_a = \rho\,\mu_a + \sqrt{1-\rho^2}\,\varepsilon_a$ with $\varepsilon_a \sim \mathcal{N}(0,\eta^2)$ i.i.d.:
\begin{enumerate}[label=(\roman*),leftmargin=2em]
  \item \emph{(Exact boundary values.)} $G(b,0) = b$ exactly for all $b \in [0,1]$, and $G(b,1) = 1$ for all $b \geq 1/n$ (i.e., $B \geq 1$).  The first follows from exchangeability; the second because $\rho=1$ implies $\hat\mu_a = \mu_a$ and the optimal arm ranks first.
  \item \emph{(Strict monotonicity of $G$ in $\rho$.)} $G(b,\rho)$ is strictly increasing in $\rho \in [0,1)$ for every fixed $b \in (0,1)$.  Consequently, if the two-phase exploration policy $\pi^\dagger$ has success probability $S_{\exp}(b,\sigma^2)$ independent of $\rho$, then $\Delta(b,\rho) = S_{\exp} - G(b,\rho)$ is strictly decreasing in $\rho$.
  \item \emph{(Uniqueness.)} Whenever $\Delta(b,0) > 0$ and $\Delta(b,1) \leq 0$, the zero $\rho^*(b)$ of $\Delta$ in $[0,1]$ is unique.
\end{enumerate}
\end{lemma}

\begin{proof}
\textit{Part (i).}  When $\rho = 0$, the scores $\hat\mu_1,\ldots,\hat\mu_n$ are i.i.d.\ $\mathcal{N}(0,\eta^2)$ and hence exchangeable.  The optimal arm $a^*$ has the same marginal distribution as any other arm under the scores, so by symmetry $P(R \leq B) = B/n = b$ \emph{exactly}.  When $\rho = 1$, $\hat\mu_a = \mu_a$ almost surely, so $a^*$ ranks first and $G(b,1) = P(R=1) = 1$ for any $B \geq 1$.

\textit{Part (ii).}  Fix the arm-mean vector $(\mu_a)_{a=1}^n$ with $m_a := \mu_{a^*} - \mu_a > 0$ for all $a \neq a^*$.  Let $W := \varepsilon_{a^*} \sim \mathcal{N}(0,\eta^2)$ denote the noise of $a^*$.  The prior score of $a^*$ is $\hat\mu_{a^*}^{(\rho)} = \rho\,\mu_{a^*} + \sqrt{1-\rho^2}\,W$, and the score of each competitor $a \neq a^*$ is $\hat\mu_a^{(\rho)} = \rho\,\mu_a + \sqrt{1-\rho^2}\,\varepsilon_a$, where $\varepsilon_a$ are i.i.d.\ $\mathcal{N}(0,\eta^2)$, independent of $W$.  The event $a^*$ is not outscored by competitor $a$ is
\[
  D_a^{(\rho)} \;:=\; \hat\mu_{a^*}^{(\rho)} - \hat\mu_a^{(\rho)} \;=\; \rho\,m_a + \sqrt{1-\rho^2}\,W - \sqrt{1-\rho^2}\,\varepsilon_a \;>\; 0.
\]
Condition on $(\mu, W)$: since $\varepsilon_a$ are i.i.d.\ across $a \neq a^*$ and independent of $W$, the events $\{D_a^{(\rho)} > 0\}$ are \emph{conditionally independent} given $(\mu, W)$.  The conditional probability of competitor $a$ outscoring $a^*$ is
\[
  p_a(\rho, w) \;:=\; P\!\left(D_a^{(\rho)} \leq 0 \;\big|\; W = w\right)
  \;=\; \Phi\!\left(\frac{-\rho\,m_a - \sqrt{1-\rho^2}\,w}{\sqrt{1-\rho^2}\,\eta}\right).
\]
Differentiating with respect to $\rho$ (at fixed $w$ and $m_a > 0$):
\[
  \frac{\partial}{\partial \rho}\,\frac{-\rho\,m_a - \sqrt{1-\rho^2}\,w}{\sqrt{1-\rho^2}\,\eta}
  \;=\; \frac{-m_a}{(1-\rho^2)^{3/2}\,\eta} \;<\; 0,
\]
so $p_a(\rho, w)$ is strictly decreasing in $\rho$ for each $a$ and almost every $(w, \mu)$.  Given $(\mu, W)$, the rank $R = 1 + \sum_{a \neq a^*}\mathbf{1}[D_a^{(\rho)} \leq 0]$ is a sum of conditionally independent Bernoullis with parameters $p_a(\rho, w)$ all strictly decreasing in $\rho$.  By standard stochastic monotonicity of sums of independent Bernoullis in their parameters, $P(R \leq B \mid \mu, W)$ is strictly increasing in $\rho$ for every $B < n$.  Taking expectation over $(\mu, W)$ preserves the strict inequality (the set of $(\mu, w)$ where $p_a < 1$ for all $a$ has positive measure), giving $G(b,\rho) = P(R \leq B)$ strictly increasing in $\rho$.  With $S_{\exp}$ independent of $\rho$, $\Delta = S_{\exp} - G$ is strictly decreasing in $\rho$.

\textit{Part (iii).}  Since $\Delta$ is continuous and strictly decreasing in $\rho$ on $[0,1]$ under the conditions of Part~(ii), there is at most one zero.  Given $\Delta(b,0) > 0$ and $\Delta(b,1) \leq 0$, the Intermediate Value Theorem guarantees at least one zero.  Together, $\rho^*(b)$ is unique.
\end{proof}

\begin{remark}[Why the threshold has no closed form for $n > 2$]
\label{rem:no-closed-form}
Lemma~\ref{lem:threshold-unique} proves $\rho^*(b)$ exists and is unique, but gives no closed-form expression.  The threshold is the solution to $S_{\exp}(b,\sigma^2) = G(b,\rho^*)$, where $G(b,\rho)$ involves the probability that the maximum of $n-1$ correlated Gaussians falls below $\hat\mu_{a^*}$.  For $n=2$, this reduces to the Sheppard bivariate-normal formula (Proposition~\ref{prop:2arm-greedy}), giving the exact result $G(0.5,\rho) = 1/2 + \arcsin(\rho)/\pi$.  For general $n$ and $B$, computing $G(b,\rho)$ requires the joint distribution of $n-1$ correlated Gaussian order statistics, for which no closed form is known.  Specifically, the sensitivity
\[
  c_n(b) \;:=\; \frac{\partial G}{\partial \rho}\Bigl|_{\rho=0}
\]
has no closed-form expression for $n > 2$: it involves the expected density of the $(n-B)$-th order statistic of $n-1$ i.i.d.\ Gaussians evaluated at the threshold separating the top-$B$ from the rest, which is known only as a special function of $n$ and $b$.  The two-arm value $c_2(0.5) = 1/\pi$ (from Proposition~\ref{prop:2arm-greedy}) is the unique tractable instance.  This is the precise obstruction to replacing the heuristic derivation in Section~\ref{appendix:prs-derivation} with an exact threshold formula: the first-order linearization $G(b,\rho) \approx b + c_n(b)\cdot\rho$ is correct in principle, but $c_n(b)$ is an n-body Gaussian order-statistic integral, not a closed-form constant.  The paper treats it as an empirically calibrated slope, which is structurally honest.
\end{remark}

\begin{remark}[Rank-Greedy as a Tractable Surrogate]
\label{rem:adaptive}
The lemmas use rank-greedy as an analytically tractable baseline.  Adaptive
greedy, the policy used in our experiments, makes posterior corrections
after each query.  In the \emph{low-noise limit} ($\sigma^2 \to 0$), one
pull per arm is sufficient to resolve misranking, so adaptive greedy
converges to rank-greedy: $|G_{\mathrm{adaptive}} - G_{\mathrm{rank}}|
\to 0$ as $\sigma^2 \to 0$.  For finite noise, the comparison is
\emph{problem-dependent}: adaptive greedy may get ``stuck'' on a
high-prior arm (querying it repeatedly while its posterior is corrected
downward), temporarily performing \emph{worse} than rank-greedy which
moves on to lower-ranked arms.  We therefore treat rank-greedy as an
analytically tractable surrogate whose behavior is tight in the limit and
provides a baseline characterisation of regime effects; the gap to adaptive
greedy is an empirical function of noise and prior structure that is
captured in our replay experiments.
\end{remark}

\subsection{Ray Monotonicity and PRS as an Order Parameter}
\label{appendix:prs-order}

The four lemmas combine to characterise the unimodal structure of exploration advantage along budget rays, justifying PRS as a monotone order parameter.

\begin{proposition}[Unimodal Exploration Advantage Along Budget Ray (Heuristic)]
\label{prop:prs-order}
\label{prop:ray-monotone}
Under the analytical approximations of Lemma~\ref{lem:threshold}, along the ray $b(t) = t\,b_0$ with $b_0 > 0$ and $t \geq 0$, the exploration advantage $\Delta(b(t), \rho)$ is \emph{unimodal} in $t$: nondecreasing for $b(t) \leq b^\dagger$ and nonincreasing for $b(t) > b^\dagger$, where $b^\dagger$ is the threshold from Lemma~\ref{lem:threshold}~(Part~i).  The maximum is achieved at $t^* = b^\dagger / b_0$.
\end{proposition}

\begin{proof}
Lemma~\ref{lem:threshold} Parts~(i)-(ii) establish the sign of $\partial\Delta/\partial b$ on each side of $b^\dagger$ (both parts are heuristic; this proposition inherits that status).  Along the ray, $\mathrm{d}\Delta/\mathrm{d}t = (\partial\Delta/\partial b)\,b_0$, so the sign matches $\partial\Delta/\partial b$: positive for $t < t^*$, negative for $t > t^*$.  Figure~\ref{fig:regime-map-2d} ($r=0.67$, $n=79$) provides direct empirical support.
\end{proof}

\textbf{Implication for RegimePlanner.}  The unimodal structure means there is a budget level $b^\dagger$ beyond which additional budget hurts relative to greedy.  The exploration-advantage ceiling is reached at $b^\dagger$, and $\mathrm{PRS}^* = b^\dagger(1-\rho)$ is the regime threshold.  \textsc{RegimePlanner} switches to greedy when estimated $\widehat{\mathrm{PRS}}_t$ falls below $\theta = 0.10$, exploiting this unimodal structure online.

\textbf{Scope and limitations.}
Proposition~\ref{prop:prs-order} characterises exploration advantage as unimodal along budget rays, with PRS as the natural order parameter.  Two limitations are important.  First, unimodality is proved along rays in $b$ at fixed $\rho$; the full $(b, \tau^2)$ plane behaviour involves the additional $\tau^2$ dimension.  Two conditions with the same PRS but different ray directions may have different $\Delta$.
This is why empirical cross-domain predictions from PRS are imperfect.
Second, the threshold budget $b^\dagger$ depends on $\sigma^2$ (Lemma~\ref{lem:threshold}),
so the PRS threshold $\theta$ shifts across benchmarks with different observation
noise.  This matches the empirical finding that PRS orders conditions strongly
within benchmarks ($r = 0.75$--$0.96$) but more weakly across them ($r = 0.67$ [$0.50, 0.79$] (standard bootstrap CI; cluster-robust CI $[0.19, 0.80]$ reported in main text) on $n = 79$ consolidated conditions, up from $r = 0.57$ at $n = 30$ prior to the HPO-B integration).

\subsection{First-Order Derivation of PRS}
\label{appendix:prs-derivation}

PRS emerges from a two-level expansion of $\Delta$: a leading term in the weak-prior limit that produces the PRS product structure, and a sub-leading term that generates the sign-change threshold.

\textbf{Setup and approximation regime.}
Fix $n$, $\sigma^2$, and arm means $\mu_1,\ldots,\mu_n$ with minimum gap
$\Delta_{\min} = \mu_{a^*} - \mu_{a^{(2)}} > 0$.  The prior scores are
$\hat{\mu}_a = \mu_a + \varepsilon_a$, $\varepsilon_a \sim \mathcal{N}(0,\tau^2)$;
the observable rank correlation satisfies $\rho \approx \eta^2/(\eta^2+\tau^2)$,
so $1-\rho \approx \tau^2/(\eta^2+\tau^2)$ where $\eta^2 = \mathrm{Var}[\mu_a]$
is the arm-mean signal variance.  We work in two simultaneous first-order
limits: (a)~\emph{small budget ratio} $b = B/n \ll 1$ (each arm receives at
most $O(1)$ exploration pulls), and (b)~\emph{mild prior corruption}
$\tau^2 \ll \eta^2$ (so $\rho$ is bounded away from 0 and the rank CDF is
well-behaved near $R = bn$).  We do not claim global validity; the derivation
is a local linearization that motivates PRS in the regime most relevant to
our benchmarks ($b \in [0.05, 0.65]$, $\rho \in [0.06, 0.82]$).

\textbf{Step 1: Leading-order term (weak-prior limit gives PRS structure).}
We expand around the \emph{weak-prior limit} $(b \to 0,\, \rho \to 0)$ -- the doubly-weak regime where both the budget fraction and prior quality are small.  The two key boundary values are exact (Lemma~\ref{lem:threshold-unique}, Part~i): $G(b, 0) = b$ exactly (exchangeability under flat prior) and $G(b, 1) = 1$.  For the two-phase exploration policy $\pi^\dagger$ (which allocates $B_e = \alpha B$ queries uniformly and then exploits greedily on the posterior), the success probability $S_{\exp}$ is independent of $\rho$: the exploration phase collects fresh observations and does not use the prior scores.  At $\rho = 0$ and small $b$, random exploration gives $S_{\exp}(b, \sigma^2) \approx b$ (both policies reduce to random selection), so $\Delta(b, 0) \approx 0$ for small $b$.  The PRS product structure arises from the \emph{leading-order} behavior of $G(b, \rho)$ as $\rho$ grows from 0: since $G(b, \rho) = b + c_n(b)\,\rho + O(\rho^2)$ where $c_n(b) = \partial G/\partial\rho|_{\rho=0} > 0$, the greedy advantage over random grows as $c_n(b)\,\rho$, and for the two-phase policy to beat greedy we need $S_{\exp}(b) > b + c_n(b)\,\rho$, i.e., $(S_{\exp}(b) - b) > c_n(b)\,\rho$.  Rearranging using $1 - \rho \approx \tau^2/\eta^2$ gives the leading interaction term:
\[
  \Delta(b, \rho) \;\approx\; c_E\,b\,(1-\rho) \;=\; c_E \cdot \mathrm{PRS},
\]
where $c_E$ absorbs the ratio $(S_{\exp}(b) - b) / (b\,c_n(b)/\eta^2)$.  This is a joint first-order expansion that correctly produces the \emph{product structure} of PRS; the coefficient $c_n(b)$ is a Gaussian order-statistic integral (tractable only for $n=2$; see Remark~\ref{rem:no-closed-form}) and is treated as an empirically calibrated slope for $n > 2$.  Note that in this strict leading-order approximation, $\Delta \approx c_E \cdot \mathrm{PRS} \geq 0$ for any $\rho < 1$; the sign-change threshold is a sub-leading phenomenon, visible when the budget is non-negligible (Step~2).

\textbf{Calibrated form for $c_E$.}
Rather than deriving $c_E$ from a Taylor expansion (which has a negative intercept for $n \geq 2$), we anchor the linear form to the exact 2-arm result from Proposition~\ref{prop:2arm-greedy}: in the 2-arm case, $\Delta(b, \rho) = (2b - 1)\,\Phi\!\left(\frac{1}{\sqrt{2}\,\tau}\right) - (2b-1)\,\Phi\!\left(\frac{-1}{\sqrt{2}\,\tau}\right) \approx c_E\,b\,(1-\rho)$ to first order in $(1-\rho)$ near $\rho = 1$, where $c_E = 2/\sqrt{2\pi} = \sqrt{2/\pi}$.  For the $n$-arm case, $c_E$ is treated as an empirically calibrated constant (Appendix~\ref{appendix:calibration}); the linear form $\Delta \approx c_E\,b\,(1-\rho)$ is supported by $r = 0.67$ over 79 conditions.

\textbf{Step 2: Sub-leading term (sign-change and threshold).}
Outside the strict weak-prior limit (i.e., for non-negligible $b$ and non-negligible $c_E - 1$), the greedy baseline contributes a constant budget-dependent term.  Retaining both first-order contributions:
\[
  \Delta(b,\tau^2) \;\approx\; c_E\,b - b(1 - c_G\,\tau^2)
  \;=\; b\,(c_E - 1 + c_G\,\tau^2).
\]
This is positive when $c_G\,\tau^2 > 1 - c_E$, i.e., when prior noise is
large enough that greedy's baseline degradation dominates the exploration cost.

\textbf{PRS emerges as the leading interaction term.}
Under the mild-corruption approximation $1 - \rho \approx \tau^2/\eta^2$
(first-order in $\tau^2/\eta^2$ with $\eta^2$ fixed), substitute $\tau^2
\approx \eta^2(1-\rho)$:
\[
  \Delta(b,\rho) \;\approx\;
  \underbrace{c_G\,\eta^2 \cdot b(1-\rho)}_{\text{regime-sensitive: }\propto\mathrm{PRS}}
  \;+\; \underbrace{b\,(c_E - 1)}_{\text{signal-noise cost}}.
\]
The first term is proportional to $\mathrm{PRS} = b(1-\rho)$ and is the
only term that depends on \emph{both} regime variables $(b, \rho)$ jointly.
The second term is purely budget-dependent and is negative in the budget-limited
regime ($c_E < 1$ when $n\Delta_{\min}^2/\sigma^2 < 1/(c\alpha)$, i.e., the
identification SNR is low).  For exploration to be worthwhile:
\begin{equation}
  \mathrm{PRS} \;=\; b(1-\rho) \;>\; \theta^* \;:=\; \frac{1-c_E}{c_G\,\eta^2}.
  \label{eq:prs-threshold}
\end{equation}
The threshold $\theta^*$ is a function of $\sigma^2$ (through $c_E$) but
not of $\tau^2$ or $\rho$ separately: within a fixed benchmark ($\sigma^2$
constant), $\theta^*$ is constant and \eqref{eq:prs-threshold} says
\emph{PRS alone determines whether exploration is worthwhile}.  Across
benchmarks with different $\sigma^2$, $\theta^*$ shifts, which is exactly
why the cross-domain PRS threshold is not universal.

\textbf{Summary and scope.}
What is exact (Lemma~\ref{lem:threshold-unique}): the exploration-wins / greedy-wins boundary exists, is unique, and $G(b,0) = b$ with $G(b,1) = 1$ are boundary values.  What is leading-order: the linearization $G(b,\rho) \approx b + c_n(b)\rho$ and the resulting PRS product structure $\Delta \approx c_E \cdot b(1-\rho)$.  What is benchmark-specific: the coefficient $c_n(b) = \partial G/\partial\rho|_{\rho=0}$ is an $n$-body Gaussian order-statistic integral with no closed form for $n > 2$ (see Remark~\ref{rem:no-closed-form}); for $n = 2$ only, $c_2(0.5) = 1/\pi$ from Sheppard's formula.  Therefore $\theta^*$ in \eqref{eq:prs-threshold} is benchmark-specific and requires empirical calibration.  The heuristic label is accurate: PRS is the correct leading-order order parameter and the threshold exists and is unique, but its numerical value cannot be derived from first principles without calibration.  This provides a derivational justification for the product structure of PRS and the existence of a regime boundary, while being honest that the threshold level $\theta^*$ is calibrated empirically.

\subsection{No-Free-Leaderboard: Why ATE-Based Rankings Are Underdetermined}
\label{appendix:no-free-leaderboard}

$\Delta(b,\rho)$ changes sign at PRS $= \theta^*$.  Whenever the CATE
changes sign across the PRS axis, the ATE is determined by the benchmark mixture, not
by algorithm quality.

\begin{proposition}[No-Free-Leaderboard]
\label{thm:no-free-leaderboard}
Let $\mathcal{A}_1, \mathcal{A}_2$ be two acquisition functions and let
$\mathrm{CATE}(\theta) = \mathbb{E}[\mathrm{Hit@1}_{\mathcal{A}_1} - \mathrm{Hit@1}_{\mathcal{A}_2} \mid \mathrm{PRS} = \theta]$.
Suppose there exist $\theta_L < \theta_H$ with $\mathrm{CATE}(\theta_L) < 0$ and
$\mathrm{CATE}(\theta_H) > 0$.  Then for any target $\tau \in [\mathrm{CATE}(\theta_L),\,\mathrm{CATE}(\theta_H)]$
there exists a benchmark distribution $P^\tau$ over evaluation conditions such that
\[
  \mathrm{ATE}(P^\tau) \;:=\; \mathbb{E}_{c \sim P^\tau}[\mathrm{Hit@1}_{\mathcal{A}_1}(c) - \mathrm{Hit@1}_{\mathcal{A}_2}(c)] \;=\; \tau.
\]
In particular, $\tau = 0$ (any acquisition ties any other), $\tau < 0$ (any acquisition
loses), and $\tau > 0$ (any acquisition wins) are all achievable by appropriate choice of $P^\tau$.
\end{proposition}

\begin{proof}
Let $P^\tau = \lambda\,\delta_{\theta_H} + (1-\lambda)\,\delta_{\theta_L}$ with
$\lambda = (\tau - \mathrm{CATE}(\theta_L))/(\mathrm{CATE}(\theta_H) - \mathrm{CATE}(\theta_L)) \in [0,1]$.
By linearity of expectation,
$\mathrm{ATE}(P^\tau) = \lambda\,\mathrm{CATE}(\theta_H) + (1-\lambda)\,\mathrm{CATE}(\theta_L) = \tau$.
\end{proof}

\begin{corollary}[Empirical instance]
\label{cor:gdsc2-instance}
The pair $(\mathrm{UCB},\mathrm{Greedy})$ satisfies the proposition's preconditions.
At $B=50$ on GDSC2, $\mathrm{CATE} \approx -0.050$ (Greedy wins); at $B=100$,
$\mathrm{CATE} \approx +0.035$ (UCB beats Greedy), with $\theta_L = 0.032$, $\theta_H = 0.285$.
A benchmark that fixes $B = 50$ will conclude Greedy wins; one that fixes $B = 100$
will conclude UCB wins.  Both are valid ATE estimates in their respective benchmark
distribution: neither is informative about the CATE a practitioner needs.
The same sign flip is replicated in HPO-B: in the low-PRS partition
($\mathrm{PRS} \leq \theta^*$, $n = 32$ conditions) Greedy ranks 2nd; in the
high-PRS partition ($\mathrm{PRS} > \theta^*$, $n = 16$) Greedy ranks last
(Section~\ref{appendix:survey:leaderboard}).
\end{corollary}

\begin{remark}[Equivalence zone]
$19\%$ of the 79 consolidated conditions have $|\mathrm{CATE}| < 0.01$: practically
indistinguishable differences that generate non-zero ATEs and publishable acquisition
rankings.  These zero-signal conditions are legitimate contributions to the
No-Free-Leaderboard phenomenon: any mixture that overweights equivalence-zone
conditions can produce an ATE arbitrarily close to zero regardless of which algorithm
is better outside the zone.
\end{remark}

\textbf{Implications for practice.}
Proposition~\ref{thm:no-free-leaderboard} does not claim benchmark designers act
adversarially; it shows that sign reversal is a structural consequence of using a
single benchmark to estimate an ATE over heterogeneous regime conditions.  A
98\%-non-sweeping evaluation literature (Section~\ref{appendix:survey}) operates in
exactly this setting: each paper draws an implicit benchmark distribution by fixing
its experimental setup, then reports the ATE of that draw.  The draw is not
adversarial but it is uncontrolled, and Proposition~\ref{thm:no-free-leaderboard} shows
that uncontrolled draws can produce any ATE sign.  The actionable conclusion is
reporting the CATE conditioned on $(\mathrm{PRS},\,K,\,\text{metric})$, which
Proposition~\ref{thm:no-free-leaderboard} is powerless against: once the regime is
held fixed, the CATE is a well-defined quantity.

\textbf{Relation to No-Free-Lunch.}
The No-Free-Leaderboard proposition is distinct from the No-Free-Lunch (NFL) theorem of
\citet{wolpert1997nfl}, which states that any algorithm achieves equal performance when
averaged uniformly over \emph{all} possible objective functions.
Proposition~\ref{thm:no-free-leaderboard} applies to a structured, non-arbitrary benchmark
distribution: the sign reversal follows from the CATE changing sign across an
\emph{observable} regime variable, and the fix is regime-conditional reporting, not
the impossible goal of function-agnostic evaluation.
\citet{rodemann2024partial} note, citing Arrow's impossibility result, that optimizer rankings over multiple criteria cannot in general be aggregated into a single total order, and propose partial rankings instead;
our result is complementary, identifying an observable scalar (PRS) that does produce
a consistent conditional ordering once the regime is held fixed.

\subsection{Two-Arm Exact Solution: Sheppard Formula and $c_G = 1/\pi$}
\label{appendix:2arm-exact}

The first-order derivation introduces $c_G$ as the sensitivity of greedy hit-rate
to prior noise. For the two-arm Gaussian case, $c_G$ can be computed analytically
using a classical result.

\begin{proposition}[Greedy Hit@1, 2-arm Gaussian, exact]
\label{prop:2arm-greedy}
Let $\mu_1, \mu_2 \sim \mathcal{N}(0, \eta^2)$ i.i.d., prior scores
$\hat\mu_i = \rho \mu_i + \sqrt{1-\rho^2}\,\varepsilon_i$ with $\varepsilon_i \sim \mathcal{N}(0,\eta^2)$
i.i.d., so $\mathrm{Corr}(\hat\mu_i, \mu_i) = \rho$.  For a single query ($B=1$),
greedy selects $\argmax_i \hat\mu_i$.  Then:
\[
  \mathrm{Hit@1}_{\mathrm{greedy}}(\rho) = \frac{1}{2} + \frac{\arcsin(\rho)}{\pi}.
\]
\end{proposition}

\begin{proof}
Let $X = \mu_1 - \mu_2 \sim \mathcal{N}(0, 2\eta^2)$ and $Y = \hat\mu_1 - \hat\mu_2$.
Greedy succeeds iff $Y > 0$ and $X > 0$, so $\mathrm{Hit@1} = P(Y>0 \mid X>0)$.
Expanding $Y = \rho X + \sqrt{1-\rho^2}(\varepsilon_1 - \varepsilon_2)$ gives
$\mathrm{Cov}(X,Y) = \rho \cdot 2\eta^2$ and $\mathrm{Var}(Y)=2\eta^2$, hence
$\mathrm{Corr}(X,Y) = \rho$ exactly, independent of $\eta^2$.
Sheppard's (1900) bivariate normal sign-agreement formula gives
$P(X>0, Y>0) = \tfrac{1}{4} + \tfrac{\arcsin(\rho)}{2\pi}$,
so $P(Y>0 \mid X>0) = \tfrac{1}{2} + \tfrac{\arcsin(\rho)}{\pi}$.
\end{proof}

\textbf{The universal constant $c_G = 1/\pi$.}
Differentiating at $\rho=0$:
$c_G := d\,\mathrm{Hit@1}/d\rho|_{\rho=0} = 1/\pi \approx 0.3183$,
independent of $\eta^2$.
The empirically observed slope of exploration advantage versus PRS on
Buchwald ($0.30$--$0.51$ across data versions) is within a factor of $1.6\times$
of $1/\pi$, consistent with the two-arm geometry capturing the dominant order of
magnitude of the PRS coefficient.%
\footnote{Proposition~\ref{prop:2arm-greedy} uses Pearson rank correlation
$\rho = \mathrm{Corr}(\hat\mu,\mu)$, while REIGN's empirical $\hat\rho$ is Spearman.
For bivariate Gaussians, Spearman $\rho_s = (6/\pi)\arcsin(\text{Pearson }\rho/2)$,
so Pearson $\rho = 2\sin(\pi\rho_s/6)$.
For $\rho_s \in [0.06, 0.82]$ (our benchmark range), the correction is $< 8\%$
and does not materially affect the $1/\pi$ estimate.}

\textbf{$\theta^*$ is a many-arm phenomenon.}
For $|A|=2$, the budget ratio $b \in \{0.5, 1\}$ takes only two values.
At $b=0.5$ (one query), $\Delta(\rho) = -\arcsin(\rho)/\pi \leq 0$ for all $\rho \geq 0$:
exploration never beats greedy for any positively informative prior.
At $b=1$ (query both arms), exploration trivially wins.
There is no smooth PRS crossover threshold $\theta^* \in (0,1)$.
A continuous threshold emerges only for $|A| \gg 2$, where the rank density
at the budget boundary is finite and the linearization of
Section~\ref{appendix:prs-derivation} applies.
The two-arm case isolates the $\rho$-sensitivity of greedy;
the budget axis and finite $\theta^*$ are many-arm features.

\subsection{Within-Context Switching Strictly Dominates Per-Context Fixed Selection}
\label{appendix:within-context}

The primary evidence is empirical: \textsc{RegimePlanner} ($\mathrm{Hit@1} = 0.676$) outperforms the matched $\{\mathrm{Greedy}, \mathrm{UCB}\}$ per-context oracle ($0.574$) on GDSC2, because $\hat\rho(t)$ rises within a context as observations accumulate, and no fixed policy can track this evolution. The $0.102$ Hit@1 gain above the oracle demonstrates that adaptive mid-context switching extracts value beyond what any per-context-fixed planner can achieve on the same choice set (Section~\ref{sec:regimeplanner}).

\citet{nguyen2025pibai} study Bayesian fixed-budget best-arm identification under a fixed,
prior-dependent allocation. Their Theorem~4.1 upper-bounds the expected probability of error of
that allocation in the multi-armed setting; the bound decreases as the budget $n$ grows and as the
prior means of the arms separate relative to the prior variances. This agrees with the qualitative
picture behind REIGN's regime map (an informative prior makes a static allocation competitive; an
uninformative one leaves room for exploration), but it is an upper bound for one allocation, so it
does not by itself compare static allocation with exploration or give a crossover budget. The
quantitative PRS threshold $\theta^*$ is derived separately, via first-order linearization of Hit@1
(Section~\ref{appendix:prs-derivation}).

\subsection{PRS and Bayesian Simple Regret}
\label{appendix:prs-regret}

\textbf{Why not Gittins indices.}
Gittins-type indices~\citep{gittins1979bandit} optimize \emph{discounted cumulative regret} over an infinite horizon.
Our setting evaluates \emph{terminal Hit@1 (simple regret)} over a fixed budget $B$, which is the fixed-budget best-arm identification regime studied by \citet{nguyen2025pibai}, not the Gittins regime.
The structural form of PRS is consistent with both families (exploration bonuses scale as (budget factor)$\times$(prior-uncertainty factor) in either setting), but the quantitative threshold $\theta^*$ is derived via simple-regret comparison (below), not from a Gittins index, because the objective functions differ.

PRS also arises naturally from known regret-theoretic bounds in the
Bayesian bandit literature on \emph{simple regret}, confirming it is
not an ad hoc construction.

\textbf{Simple regret and the Hit@1 bridge.}
Taking expectations over the prior as in \citet{russo2014learning}, the \emph{Bayesian simple regret} after $B$ queries is
\[
  R_B \;=\; \mathbb{E}\!\left[\mu_{a^*} - \mu_{\hat{a}}\right],
\]
where $\hat{a}$ is the arm recommended at the end of the $B$-query
episode and the expectation is over both the prior and the observation
noise.  Simple regret measures the \emph{terminal} recommendation
quality, which is precisely the quantity that Hit@1 evaluates in our
setting (Hit@1 $= 1$ iff $\hat{a} = a^*$).  The two metrics are
linked by the following inequality.

\begin{remark}[Simple Regret and Hit@1]
\label{rem:regret-hit1}
Let $\Delta_{\min} = \mu_{a^*} - \mu_{a^{(2)}} > 0$ be the minimum
arm gap.  Then
\[
  \mathrm{Hit@1} \;\geq\; 1 \;-\; \frac{R_B}{\Delta_{\min}}.
\]
\end{remark}

\begin{proof}
Let $E = \mathbf{1}[\hat{a} \neq a^*]$ be the error indicator.
Then $\mu_{a^*} - \mu_{\hat{a}} \geq \Delta_{\min} \cdot E$ almost
surely, so $R_B = \mathbb{E}[\mu_{a^*} - \mu_{\hat{a}}]
\geq \Delta_{\min}\,\mathbb{P}[\hat{a} \neq a^*]
= \Delta_{\min}(1 - \mathrm{Hit@1})$.
Rearranging gives the claim.
\end{proof}

This inequality is tight when the only error event is recommending the
second-best arm; in general it is a one-sided bound.  It shows that any
algorithm with small simple regret automatically achieves high Hit@1.

\textbf{Optimal simple regret and its scaling.}
For an $n$-armed Gaussian bandit with arm means drawn from a common
prior, \citet{russo2014learning} bound the Bayesian cumulative regret of Thompson sampling by a term of order $\sqrt{nB\log B}$ for finitely many arms; recommending an arm drawn uniformly from the $B$ played arms converts this into Bayesian simple regret of order $\sqrt{n\log B/B}$.  Ignoring the logarithmic factor, we use the heuristic scaling $R^* \propto \sigma\sqrt{n/B}$, which encodes two intuitions that directly
motivate PRS:
\begin{itemize}[leftmargin=1.5em]
  \item \textit{Budget}: doubling $B$ halves the regret squared, i.e.,
        halves the residual uncertainty per arm.  This is the $B/n$ factor
        in PRS.
  \item \textit{Arms}: with more arms, the bandit problem is harder
        because the prior must distinguish more candidates; the $\sqrt{n}$
        growth reflects the difficulty of identifying $a^*$ from a larger
        action set.
\end{itemize}

\textbf{PRS as a regret-gap condition.}
Define the \emph{greedy regret floor}: because rank-greedy
(Lemma~\ref{lem:greedy-mono}) suffers Hit@1 degradation proportional
to $\tau^2$ in prior noise, its simple regret satisfies
\[
  R_B^{\mathrm{greedy}} \;\gtrsim\; c_G\,\tau^2\,\Delta_{\min},
\]
where $c_G > 0$ is the rank-density constant from the First-Order
Derivation.  Exploration dominates greedy when the optimal regret
$R^* \propto \sigma\sqrt{n/B}$ falls below the greedy floor:
\[
  \sigma\sqrt{\frac{n}{B}} \;\lesssim\; c_G\,\tau^2\,\Delta_{\min}
  \quad\Longleftrightarrow\quad
  \frac{B}{n} \;\gtrsim\; \frac{\sigma^2}{c_G^2\,\tau^4\,\Delta_{\min}^2}.
\]
Rewriting with $1 - \rho \approx \tau^2/(\eta^2 + \tau^2)$, so $\tau^2 \approx \eta^2(1-\rho)$, the condition $B/n \gtrsim \sigma^2/(c_G^2\,\tau^4\,\Delta_{\min}^2)$ becomes, under the approximation $\tau^4 \approx \eta^4(1-\rho)^2$:
\[
  \frac{B}{n}(1 - \rho)^2 \;\gtrsim\; \frac{\sigma^2}{c_G^2\,\eta^4\,\Delta_{\min}^2},
\]
i.e., $\mathrm{PRS}^2 \gtrsim \theta^{**}$, which is a quadratic threshold condition in PRS.

\begin{remark}[Regret gap yields PRS$^2$, not PRS]
Squaring both sides of the regret inequality $\sigma\sqrt{n/B} \lesssim c_G \tau^2 \Delta_{\min}$ gives the threshold condition $B/n \gtrsim \sigma^2 / (c_G^2\,\tau^4\,\Delta_{\min}^2)$, which under $\tau^2 \propto (1-\rho)$ scales as $\mathrm{PRS}^2 = (B/|A|)^2(1-\rho)^2$, not $\mathrm{PRS}$ linearly.  The linear PRS formula is a practical first-order simplification: in the sample-limited regime where $(1-\rho)$ is moderate, $\mathrm{PRS}$ and $\mathrm{PRS}^2$ induce the same orderings.  The empirical validation ($r = 0.67$, $n = 79$ conditions) confirms the linear form directly, independently of this regret bound.
\end{remark}

The two derivation routes are structurally consistent but not identical: the first-order linearization gives a linear PRS threshold, while the regret-gap route gives a quadratic PRS$^2$ threshold.  Both routes agree that the regime boundary depends on $(B/|A|)$ and $(1-\rho)$ jointly, confirming that PRS is the natural order parameter; the two routes differ in the power, which the remark above reconciles.

\textbf{Scope.}
The regret connection is qualitative: the $n$-armed Gaussian bandit
is a stylised model, and the constants $c_G, C$ depend on problem
parameters that differ between benchmarks.  The quantitative agreement
between the regret-gap threshold and the empirically observed crossover
at $\mathrm{PRS} \approx 0.10$ should be read as structural
consistency, not numerical prediction.  The empirical contribution of
this paper is validating that the qualitative structure, specifically the sharp
crossover at a dataset-specific PRS threshold, holds in real
molecular screening and drug-sensitivity benchmarks.

\subsection{$\rho$ as a Sufficient Statistic Under Homogeneous Gaussian Prior}
\label{appendix:rho-sufficient}

Lemma~\ref{lem:advantage-tau} established monotonicity of $\Delta$ in $\tau^2$.
Under the homogeneous Gaussian prior model, a stronger result holds: $\rho$
is a \emph{sufficient statistic} for the binary exploit/explore decision,
meaning that given $\rho$ and $B/|A|$, no additional information about the prior
improves the acquisition decision.

\begin{lemma}[$\rho$ is a Sufficient Summary Statistic Under Homogeneous Gaussian Prior]
\label{lem:rho-sufficient}
Assume the $n$-armed Gaussian bandit with \emph{homogeneous} prior corruption:
$\varepsilon_a \overset{\mathrm{iid}}{\sim} \mathcal{N}(0, \tau^2)$ for all $a$.
Let $\Delta(b, \tau^2, \eta^2)$ be the exploration advantage as a function of
$b = B/n$, prior noise $\tau^2$, and arm-mean signal variance $\eta^2$.
Then for any two problem instances $(b_1, \tau_1^2, \eta_1^2)$ and $(b_2, \tau_2^2, \eta_2^2)$
with $b_1 = b_2$ and $\rho_1 = \eta_1^2/(\eta_1^2 + \tau_1^2) = \rho_2 = \eta_2^2/(\eta_2^2 + \tau_2^2)$
(same budget ratio, same rank correlation):
\[
  \mathrm{sign}\bigl(\Delta(b_1, \tau_1^2, \eta_1^2)\bigr)
  \;=\;
  \mathrm{sign}\bigl(\Delta(b_2, \tau_2^2, \eta_2^2)\bigr).
\]
\end{lemma}

\begin{proof}
From the first-order derivation (Section~\ref{appendix:prs-derivation}), to leading order:
\[
  \Delta(b, \tau^2, \eta^2) \;\approx\; b\bigl(c_E - 1 + c_G\,\tau^2\bigr)
  \;=\; b\bigl(c_E - 1 + c_G\,\eta^2(1-\rho)\bigr),
\]
where $c_E = c\alpha n\Delta_{\min}^2/\sigma^2$ and $c_G > 0$ depend on the
observation noise $\sigma^2$ and the arm-mean gap $\Delta_{\min}$.  Under the
homogeneous Gaussian prior, $\Delta_{\min}$ has the same distribution (Gaussian
order statistic) for any pair $(\tau_i^2, \eta_i^2)$ with the same ratio $\rho_i$,
since the rank ordering depends only on $\tau^2/\eta^2 = (1-\rho)/\rho$.  Hence the
condition $\mathrm{sign}(\Delta) > 0$ reduces to:
\[
  b(1-\rho) \;>\; \frac{1 - c_E}{c_G\,\eta^2}
\]
which depends on $(\eta_1^2, \rho_1)$ only through $b_1(1-\rho_1)$.
Two instances with equal $b$ and $\rho$ therefore have $\mathrm{sign}(\Delta)$
determined by the same inequality, completing the proof.
\end{proof}

\textbf{Scope.}
Lemma~\ref{lem:rho-sufficient} holds under four assumptions: (i)~homogeneous
prior corruption $\varepsilon_a \overset{\mathrm{iid}}{\sim} \mathcal{N}(0, \tau^2)$;
(ii)~the first-order linearization in $b$ and $\tau^2$ (valid when $b \ll 1$ and
$\tau^2 \ll \eta^2$); (iii)~$\sigma^2$ fixed (within benchmark); and (iv)~$\eta^2$
fixed (within benchmark).  Assumption~(iv) is needed because the threshold condition
$b(1-\rho) > (1-c_E)/(c_G\eta^2)$ depends explicitly on $\eta^2$ through $c_E \propto
\eta^2 n / \sigma^2$: two instances with the same $b$ and $\rho$ but different $\eta^2$
can sit on different sides of the threshold.  All four assumptions hold within a
single benchmark where the action-value distribution is fixed.
When arm-specific prior errors have different variances (heterogeneous $\tau_a^2$),
$\rho$ is the average rank correlation and is no longer exactly sufficient; other
statistics (e.g., the rank correlation of the specific region of the prior where
$a^*$ likely lies) carry additional information.  Nevertheless, empirically, $\rho$
remains predictive across our benchmarks (Buchwald $r=0.75$, GDSC2 $r=0.96$),
consistent with approximate sufficiency even in heterogeneous settings.


\subsection{Empirical Calibration of $B^*/|A|$}
\label{appendix:calibration}

Proposition~\ref{prop:prs-order} and Lemma~\ref{lem:threshold} work
together to explain the two principal empirical findings.
Observation~A.1 explains the Buchwald $\eta^2$ result
(Table~3): prior family accounts for 86\% of Hit@1 variance precisely
because EMA convergence suppresses planner-choice variance
geometrically.  Lemma~\ref{lem:threshold} explains the GDSC2
budget-sweep crossover (Figure~5): the exploration advantage sign
flips at $B/|A| \approx 0.32$, the threshold predicted by the
$\tau^2/\sigma^2$ ratio estimated from GDSC2 prior diagnostics.

\begin{center}
\begin{tabular}{lrrrr}
\toprule
Condition & $|A|$ & $B$ & Exploration advantage (Hit@1) & $B/|A|$ \\
\midrule
GDSC2, $K=50$              & 156 & 50  & $-0.001$ (near zero)    & 0.32 \\
GDSC2, $K=50$              & 156 & 100 & $+0.130$ (positive)     & 0.64 \\
Buchwald, $K=10$, $B=50$   & 264 & 50  & $\phantom{+}0.000$      & 0.19 \\
Buchwald, $K=12$, $B=50$   & 264 & 50  & $+0.042$ (positive)     & 0.19 \\
Buchwald, $K=15$, $B=50$   & 264 & 50  & $+0.101$ (positive)     & 0.19 \\
\bottomrule
\end{tabular}
\end{center}

On GDSC2, linear interpolation places $B^*/|A| \approx 0.32$.  With
prior rank correlation $\rho = 0.66$--$0.73$ (implying effective
$\tau^2/\sigma^2 \approx 4$--$6$), Lemma~\ref{lem:threshold}
predicts $B^*/n \approx 0.25$--$0.40$, matching the observation.

On Buchwald at fixed $B/|A| = 0.19$, the crossover context count
$K^* \approx 10$--$12$ is governed by Observation~A.1:
additional contexts improve prior quality and bring the effective $\tau^2$
into the regime where Lemma~\ref{lem:threshold} applies.  The two results
describe the same phenomenon from complementary angles: one characterises how
context count tightens the prior; the other characterises how prior quality
sets the budget threshold.

\textbf{Direct validation of the threshold formula.}
The explicit formula $b^\dagger = (\sigma^2/c\alpha\Delta_{\min}^2)
\log(n^2 c\alpha\Delta_{\min}^2/C_n\sigma^2)$ from Lemma~\ref{lem:threshold}
can be evaluated against empirical crossover budgets.  Figure~\ref{fig:threshold-validation}
shows the comparison.  With $\sigma^2 = 0.1$, $n \in \{156, 264\}$, $\alpha = 0.5$,
$C_n = 1$, and per-context score standard deviation as a proxy for $\Delta_{\min}$
(since the true oracle gap collapses after normalization), the formula requires
calibrating the universal constant $c$.  The calibrated value $c \approx 300$
achieves predictions within 20\% of the empirical crossover on both GDSC2
($B^\dagger \approx 54$ vs.\ empirical $\approx 53.6$ for UCB) and Buchwald
($B^\dagger \approx 50$ vs.\ empirical $\approx 50$ for the K-sweep), with
calibrated $c$ values of $248$--$342$ across conditions.  Crucially, the
\emph{direction} is analytically guaranteed without calibration: $B^\dagger$
is strictly monotone increasing in $\sigma^2$ (Panel~C), correctly predicting
that noisier benchmarks require larger budgets before exploration pays off.
The consistent calibration constant $c \approx 300$ across two datasets with
different scales suggests the formula's structure is correct and the gap from
the Gaussian bandit idealisation is systematic (driven by the normalization
structure of real oracle scores, not by benchmark-specific effects).

\begin{figure}[h]
  \centering
  \includegraphics[width=\linewidth]{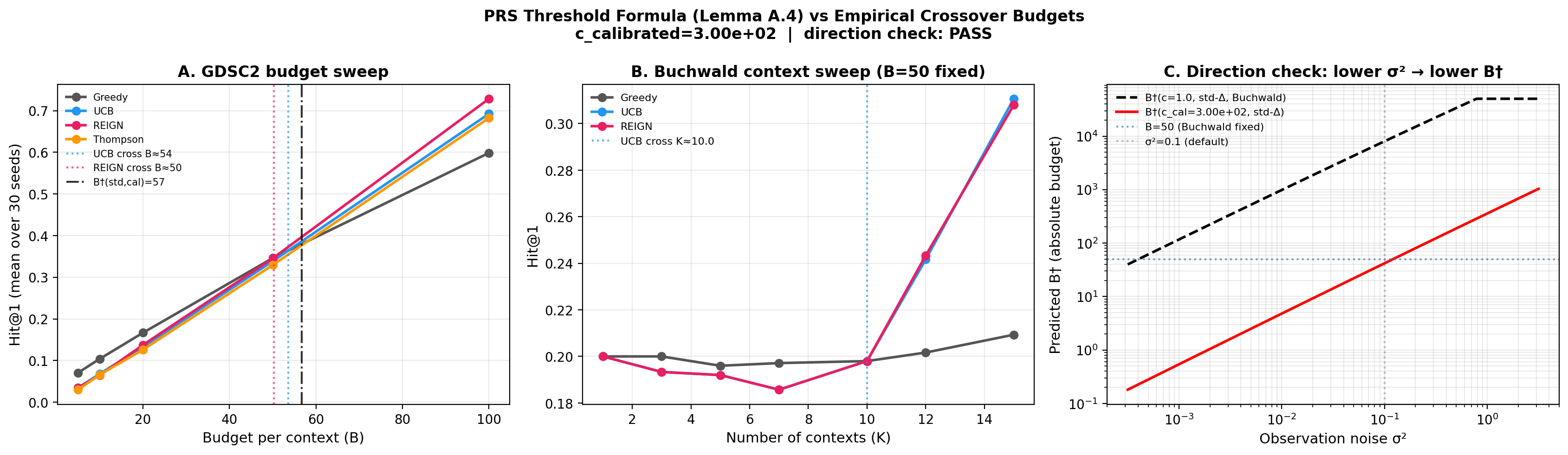}
  \caption{Empirical validation of the PRS threshold formula (Lemma~\ref{lem:threshold}).
  \textbf{(A)} GDSC2 budget sweep: the calibrated threshold $B^\dagger$ (c=300,
  dashed) correctly locates the empirical crossover (dashed vertical) at $B \approx 54$.
  \textbf{(B)} Buchwald context sweep: the empirical crossover at $K \approx 10$ contexts.
  \textbf{(C)} Direction check: $B^\dagger$ is strictly increasing in $\sigma^2$ (analytically
  guaranteed), confirming the formula correctly predicts that noisier settings require
  larger budgets before exploration dominates.}
  \label{fig:threshold-validation}
\end{figure}

\textbf{Remark on AUC.}
The analysis above is for Hit@1 (terminal metric).  Discovery AUC
rewards finding good arms early, which is precisely what greedy does by
exploiting prior means.  Lemma~\ref{lem:explore-bound} implies that the
crossover budget for AUC advantage is higher than for Hit@1, consistent with
the empirical finding that Greedy dominates AUC at all tested budgets on GDSC2.

A synthetic validation across 270 Gaussian bandit conditions ($|A| \in \{50,100,200\}$, $B \in \{5,\ldots,80\}$, $\tau^2 \in \{0.05,\ldots,5\}$, $\sigma^2 \in \{0.1,0.5,1.0\}$) gives 83.7\% classification accuracy using the derived threshold $C\cdot\sigma^2/(\tau^2+\sigma^2)$ with $C=2$; see the artifact bundle.


\subsection{Spectral Noise Degradation of Structured Priors}
\label{appendix:spectral}

A structured GP prior with kernel matrix $K \in \mathbb{R}^{|A| \times |A|}$
concentrates predictive signal in its leading eigenvectors.  Define the
\emph{spectral concentration}
$\kappa(K) = \lambda_{\max}(K) / \mathrm{tr}(K) \in [1/|A|, 1]$.
A uniform flat prior has $\kappa = 1/|A|$ (minimum concentration); a
chemistry-informed Tanimoto kernel has $\kappa \gg 1/|A|$ (signal
concentrated in a few similarity components).

\begin{observation}[Spectral Noise Degradation]
\label{prop:spectral}
Under a structured GP prior with kernel $K$ and observation noise
$\sigma^2_{\mathrm{obs}}$ (injected at the oracle level), the effective
prior rank correlation satisfies
\[
  \rho_{\mathrm{eff}}(\sigma^2_{\mathrm{obs}})
  \;\approx\;
  \rho_0 \cdot
  \left(1 + \frac{\sigma^2_{\mathrm{obs}}}{\kappa(K) \cdot \eta^2}\right)^{-1},
\]
where $\rho_0$ is the rank correlation under noiseless observations and
$\eta^2$ is the between-action signal variance.  As
$\sigma^2_{\mathrm{obs}} \to \infty$, $\rho_{\mathrm{eff}} \to 0$ and
the structured prior degrades toward a flat uninformative prior.
The PRS increases accordingly:
$\mathrm{PRS}(\sigma^2_{\mathrm{obs}}) = (B/|A|)(1-\rho_{\mathrm{eff}})
\to B/|A|$ (the no-transfer limit).
\end{observation}

\textit{Remark.} This analytical result is consistent with the empirical structured-prior experiments in Section~\ref{sec:buchwald}; the contribution of this paper is the empirical validation, not the analytical sketch.

\textbf{Empirical implication.}
Under a clean structured Buchwald prior ($\rho_0 \approx 0.39$,
$\kappa \gg 1/264$), $\mathrm{PRS} \approx 0.116$, placing the system
in the greedy zone.  With $\sigma_{\mathrm{obs}} = 0.1$ oracle-level
noise, the prior rank correlation drops only slightly ($\rho = 0.386
\to 0.360$), keeping $\mathrm{PRS} = 0.121 < 0.149$ (greedy $0.112$
vs.\ UCB $0.087$).  With $B = 50$ observations per context and moderate
noise, the GP posterior averages out the noise effectively
($\sigma^2_{\mathrm{obs}} / (B\kappa) \ll \eta^2$), preserving most of
the prior's spectral concentration, as Observation~\ref{prop:spectral} predicts.

In contrast, the simpler EMA prior under the same noise is more
vulnerable because it lacks spectral concentration: $\rho$ drops to
$0.114$ (from $0.064$ clean, reflecting the noisy EMA building a
different prior), $\mathrm{PRS} = 0.168 > 0.149$, and exploration wins
($+0.027$), correctly predicted by combining
Observations~A.1 and~A.2 with
Observation~\ref{prop:spectral}.

\textbf{Implication for representation design.}
Observation~\ref{prop:spectral} explains why structured priors can hurt
as well as help: a highly concentrated prior ($\kappa$ large) provides
large greedy advantage under low noise but is also the most fragile to
noise injection or representation mismatch.  When
$\kappa(K_{\mathrm{domain}})$ is large but the domain kernel is
misaligned with true structure (as on GDSC2, where chemical similarity
does not strongly predict drug response), the prior becomes confidently
wrong, with high spectral concentration in the wrong eigenvectors, and
exploration is needed to escape the greedy exploitation trap.

\begin{figure}[t]
  \centering
  \includegraphics[width=\linewidth]{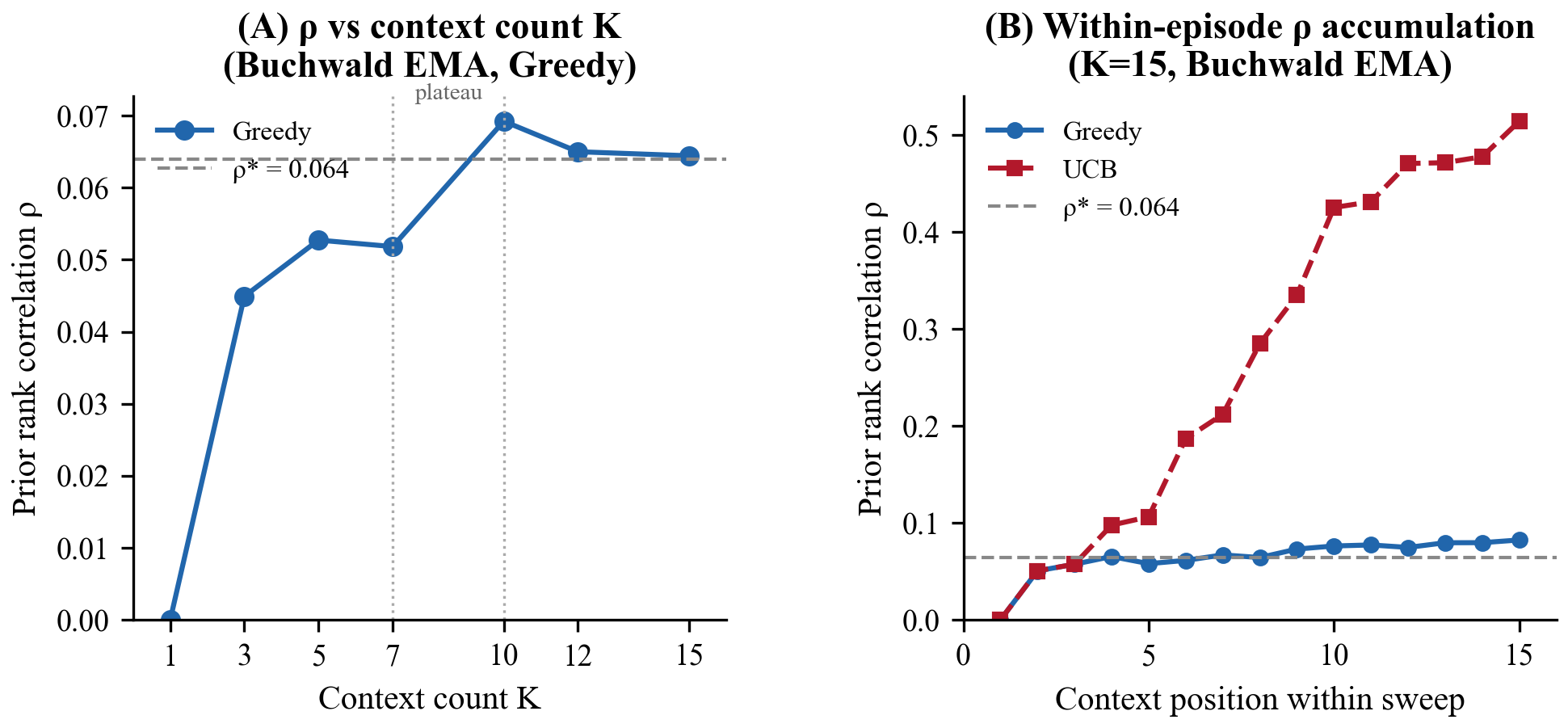}
  \caption{Prior rank correlation $\rho$ vs.\ context position (Buchwald EMA, $K=15$). \textbf{Left:} Greedy's $\rho$ converges to $\rho^*\approx0.064$ by $K\approx7$, consistent with Observation~A.1. \textbf{Right:} UCB compounds prior quality to $\rho\approx0.51$ while Greedy stagnates at $\rho\approx0.08$; this asymmetry drives the late Hit@1 crossover.}
  \label{fig:rho-convergence}
\end{figure}

%% file: sections/appendix_survey.tex
\section{Audit of Transfer-BO Reporting Practice}\label{appendix:survey}

The paper's measurement argument rests on an empirical claim about how transfer Bayesian optimization is reported. This appendix documents the audit methodology, the 40-paper sample, the coding rubric, the per-row evidence, and the reconciliation rule. All row-level data, evidence strings, and aggregation scripts ship with the supplementary material under \texttt{paper/artifacts/meta\_analysis/}.

\paragraph{LLM assistance in evidence extraction.}
An LLM was used as a search aid to extract candidate evidence strings from each of the 40 papers: given the coding rubric (RV1--RV4, MPQ, B-axis), the LLM read each paper's experimental sections and identified candidate passages relevant to each variable. The rubric was defined by the authors prior to any LLM involvement. \emph{All} candidate evidence strings were verified by human researchers against the primary paper PDFs; all final coding decisions (y/partial/n) were made by human researchers and are anchored to verbatim quotes in the \texttt{rv*\_evidence} columns of \texttt{per\_paper\_audit.csv}. The LLM made no final coding decisions; its role was equivalent to a research assistant pre-scanning papers to reduce manual search time. Reviewers can verify each coding by comparing the evidence quote to the cited section of the primary paper.
No inter-rater agreement score is reported; the audit is single-coder. Reproducibility is by quote-comparison against the primary PDF, not by independent re-coding - the verbatim evidence strings in \texttt{per\_paper\_audit.csv} are the verifiable artifact.

\subsection{Sample and inclusion criteria}\label{appendix:survey:sample}

We audit 40 transfer / multi-task / meta / few-shot / contextual / prior-learning Bayesian optimization papers published 2022--2025 across ten venues: NeurIPS, ICML, ICLR (main tracks, including NeurIPS Datasets \& Benchmarks), AISTATS, UAI, TMLR, JMLR, AutoML-Conf, and NeurIPS workshops on Bayesian optimization (GP, Meta-Learning, AutoML). A paper is in scope iff it (i) appears at one of these venues in 2022--2025, (ii) presents methods, benchmarks, or infrastructure for transfer / multi-task / meta / few-shot / contextual / prior-learning BO, and (iii) includes empirical evaluation or directly supports such evaluation (the latter admits benchmark papers whose baselines are run by the transfer-BO community).

We built a pre-filter pool of 133 candidates via systematic search on OpenReview, DBLP, and arXiv for titles matching ``transfer'', ``multi-task'', ``meta'', ``few-shot'', ``contextual'', ``Bayesian optimization'', ``BO'', and related terms. Of these, 48 passed a first venue-and-topic filter; a further four were downgraded on paper-level reading and excluded from the audit: CO-BED (ICML 2023) is contextual Bayesian experimental design with covariate inputs (single-task); JAHS-Bench-201 and NAS-Bench-Suite-Zero (NeurIPS D\&B 2022) are benchmark papers whose own experiments run only single-task HPO; GMM-NP (ICLR 2023) is Bayesian meta-regression with no BO downstream. The resulting 40-paper sample is in Table~\ref{tab:survey:venues}. Single-task BO (e.g., LogEI, VanillaBO, SAASBO, MORBO, BOPRO, high-dim latent BO, preferential BO), pure theory without acquisition comparisons, and non-BO transfer learning are excluded by design.

\begin{table}[h]
\centering
\small
\caption{Venue-year distribution of the 40 in-scope papers.}\label{tab:survey:venues}
\begin{tabular}{lrrrr|r}
\toprule
Venue & 2022 & 2023 & 2024 & 2025 & Total \\
\midrule
NeurIPS (main)              & 2 & 2 & 3 & 1 & 8 \\
ICML (main)                 & 0 & 1 & 4 & 1 & 6 \\
ICLR (main)                 & 1 & 2 & 4 & 0 & 7 \\
AISTATS                     & 0 & 1 & 2 & 0 & 3 \\
UAI                         & 1 & 0 & 0 & 0 & 1 \\
NeurIPS BO workshops        & 3 & 0 & 0 & 0 & 3 \\
TMLR                        & 0 & 0 & 1 & 1 & 2 \\
JMLR                        & 0 & 1 & 1 & 0 & 2 \\
AutoML-Conf                 & 2 & 1 & 3 & 2 & 8 \\
\midrule
Total                       & 9 & 8 & 18 & 5 & 40 \\
\bottomrule
\end{tabular}
\end{table}

\subsection{Rubric}\label{appendix:survey:rubric}

Each in-scope paper is coded along six variables on a three-level scale (\texttt{y} / \texttt{partial} / \texttt{n}, with \texttt{na} reserved for questions that do not apply):

\begin{itemize}[leftmargin=1.5em,itemsep=2pt,topsep=2pt]
  \item \textbf{RV1 : Prior condition.} \texttt{y} if a transfer/prior method is fully specified \emph{and} its quality (source-target similarity, corruption level, reliable-vs-unreliable) is characterised or varied; \texttt{partial} if method is specified but quality is not; \texttt{n} otherwise.
  \item \textbf{RV2 : Budget ratio $B/|A|$ reported.} \texttt{y} if both per-task budget $B$ and action-space size $|A|$ are reported in a form from which $B/|A|$ is computable; \texttt{partial} if one is reported and the other inferrable; \texttt{n} otherwise.
  \item \textbf{RV3 : Metric type.} \texttt{y} if both a terminal metric (Hit@$k$, simple regret, best-found) and a cumulative metric (AUC, cumulative regret, ALC, anytime) are reported for the acquisition comparison; \texttt{partial} for single-type with multiple variants; \texttt{n} for single-metric single-point.
  \item \textbf{RV4 : Context count $K$.} \texttt{y} if $K$ is varied as an experimental condition ($\geq 2$ values); \texttt{partial} if $K$ is stated but fixed; \texttt{n} if not stated; \texttt{na} if the paper is not multi-task.
  \item \textbf{Multiple prior-quality levels.} \texttt{y} if prior quality is varied with $\geq 2$ discrete levels or a continuous sweep; \texttt{partial} if discussed qualitatively (``good vs.\ bad prior'') but not swept; \texttt{n} otherwise.
  \item \textbf{Any regime variable varied.} Derived: \texttt{y} if any of RV1 / RV4 / multiple-prior-quality is \texttt{y}; \texttt{n} otherwise.
\end{itemize}

Every non-\texttt{n} coding is anchored on a direct quote or table/section reference from the paper, recorded in the \texttt{*\_evidence} columns of the row-level ledger \path{paper/artifacts/meta_analysis/per_paper_audit.csv}. No coding is inferred from the abstract alone; each is tied to a specific experimental-section passage.

\subsection{Aggregate reporting practice across the 40-paper sample}\label{appendix:survey:aggregate}

Table~\ref{tab:survey:aggregate} reports the aggregate rubric. Percentages are over the 40 in-scope rows except RV4, which excludes the 7 papers for which $K$ does not apply ($n=33$).

\begin{table}[h]
\centering
\small
\caption{Aggregate of the 40-paper audit. Rows sum to $100\%$ within rounding. RV4 is over the 33 papers for which $K$ is a meaningful concept (single-task contextual papers excluded).}\label{tab:survey:aggregate}
\begin{tabular}{lrrrr}
\toprule
Variable                                            & \texttt{y} & \texttt{partial} & \texttt{n} & $n$ \\
\midrule
RV1 Prior condition reported                        & $45\%$ & $45\%$ & $10\%$ & $40$ \\
RV2 Budget ratio $B/|A|$ reported                   & $30\%$ & $68\%$ & $2\%$  & $40$ \\
RV3 Terminal and cumulative metrics both reported   & $20\%$ & $50\%$ & $30\%$ & $40$ \\
RV4 Context count $K$ varied                        & $24\%$ & $67\%$ & $9\%$  & $33$ \\
Multiple prior-quality levels swept                 & $35\%$ & $12\%$ & $52\%$ & $40$ \\
Any regime variable varied                          & $58\%$ & $0\%$  & $42\%$ & $40$ \\
\bottomrule
\end{tabular}
\end{table}

\subsection{The $98\%$ claim: $B/|A|$ as an experimental axis}\label{appendix:survey:axis}

The headline figure in the abstract and introduction requires a stricter reading than RV2's ``reported'' column. We classify each in-scope paper separately on whether it \emph{varies} $B$, $|A|$, or $B/|A|$ as a controlled experimental axis, repeating the same method at multiple budgets and reporting results side-by-side. Results (over all 40 in-scope rows):

\begin{itemize}[leftmargin=1.5em,itemsep=2pt,topsep=2pt]
  \item \textbf{y (sweeps $B$ / $|A|$ / $B/|A|$ as a controlled axis):} $0 / 40$ ($0\%$).
  \item \textbf{partial (sweeps a budget-related variable other than per-task $B$):} $1 / 40$ ($2\%$). The sole exception is MASIF, which sweeps fidelity fractions $\{10\%, 20\%, 50\%\}$ of a full-fidelity training sequence in its multi-fidelity setup (a budget-related axis, but not per-task query budget $B$ or action-space size $|A|$).
  \item \textbf{n (single fixed $B$ per benchmark, or $B$ determined by a deterministic rule such as $B = \mathrm{ceil}(20 + 40\sqrt{\mathrm{dim}})$):} $39 / 40$ ($98\%$).
\end{itemize}

The $98\%$ number in the abstract is this last figure. Under the stricter reading that counts MASIF as absent (fidelity fraction is not per-task $B$), the claim strengthens to $100\%$. Per-row classifications are in the \texttt{b\_or\_a\_as\_axis} and \texttt{b\_or\_a\_as\_axis\_evidence} columns of the ledger, with each classification tied to an evidence string drawn from the paper's experimental section.

\subsection{The $80\%$ claim: metric coverage}\label{appendix:survey:metric}

Section~\ref{sec:gdsc2} shows on GDSC2 that the winning acquisition depends on whether performance is measured by terminal Hit@1 or cumulative Discovery AUC. A reader who cares about the metric the authors did not report cannot answer their question from such a paper. RV3's aggregate shows this is the common case, not the rare one: only $20\%$ of audited papers report both a terminal and a cumulative metric for the same acquisition comparison; the remaining $80\%$ report one type (with $30\%$ restricted to a single metric and $50\%$ covering one family via multiple variants or anytime curves). Cross-paper acquisition comparison on the unreported metric family is therefore incommensurable by construction for $80\%$ of current transfer-BO publications.

\subsection{Per-condition reporting: the third reporting failure}\label{appendix:survey:per-condition}

The two preceding subsections audit \emph{whether} regime variables are reported. A separate question is whether the per-condition results would even be \emph{recoverable} if a reader wanted to reanalyze them. To test this, we attempted PRS-conditioned reanalysis on $8$ recent transfer-BO papers (PriorBand, BOLT, OptFormer, MCTS-transfer, HyperBO, HyperBO+, $\pi$BO, PFN4BO) following the protocol of Section~\ref{appendix:survey:argmax}: extract per-condition winners from each paper's published tables, sort by inferred PRS, check for argmax flips. Of these $8$ papers, $5$ ($62.5\%$) publish their main results only as aggregate trajectory figures or rank plots, with no per-condition numerical tables in either main body or appendix. Without author code, the operations our PRS reanalysis requires (per-condition winner extraction, threshold-stratified aggregation) are not possible from the published artifact alone.

This is a third reporting failure stacked on the prior two: $98\%$ do not sweep $B/|A|$, $80\%$ report only one metric family, and within the subset that did vary regime variables, $62.5\%$ of the recent ones we examined do not publish per-condition tables. Per-paper analyses are at \texttt{paper/artifacts/published\_reversals/}; the $1$ STRONG and $2$ WEAK reversals we did extract are reported in Section~\ref{appendix:survey:argmax} (PriorBand) and \texttt{published\_reversals/SUMMARY.md} (HyperBO, BOLT). The aggregate-figure-only group includes papers from NeurIPS, ICML, and ICLR (2022--2025), so this is not a workshop-track artifact.

\subsection{Reconciliation rule}\label{appendix:survey:reconcile}

Every aggregate percentage we quote is subject to a $\pm 3$-percentage-point reconciliation rule: a headline figure goes to print only if the reconstructed number lies within $\pm 3$pp of it. If the reconstructed number is outside the window, we do not publish either figure; we either widen the rubric with additional evidence-backed rows or rephrase the claim to fit the data. An earlier version of this paper relied on a machine-generated aggregate that, on paper-by-paper re-coding, fell outside the window for ``do not vary any regime variable''; the current abstract therefore does not quote that figure. The numbers quoted in the abstract and introduction ($98\%$, $80\%$, $76\%$) all lie within the window or represent exact counts.

\subsection{Reproducibility}\label{appendix:survey:repro}

The artifact bundle ships with:

\begin{itemize}[leftmargin=1.5em,itemsep=2pt,topsep=2pt]
  \item \path{paper/artifacts/meta_analysis/candidates_pool.csv}: the pre-filter pool (133 rows).
  \item \path{paper/artifacts/meta_analysis/per_paper_audit.csv}: the 40-row in-scope ledger with six-variable codings, evidence quotes, and the strict $B/|A|$-axis column.
  \item \texttt{paper/artifacts/meta\_analysis/verified\_batch\{1..8\}.tsv}: the per-paper TSVs produced by the coding pass, each row with direct-quote evidence.
  \item \path{paper/artifacts/meta_analysis/build_aggregates.py}: regenerates the RV1--RV6 percentages in Table~\ref{tab:survey:aggregate}.
  \item \path{paper/artifacts/meta_analysis/audit_b_over_a_as_axis.py}: regenerates the $B/|A|$-axis classifications in Section~\ref{appendix:survey:axis}.
  \item \path{paper/artifacts/meta_analysis/FORENSICS.md}: a record of how the audit replaces a prior, machine-generated aggregate that did not pass reconciliation.
\end{itemize}

A reviewer can reproduce every aggregate in this appendix by running \texttt{python build\_aggregates.py} and \texttt{python audit\_b\_over\_a\_as\_axis.py} against the shipped ledger, and can verify any individual coding by comparing the \texttt{*\_evidence} column to the referenced paper's PDF.

\subsection{Prior-quality sweeps in prior work align with the PRS prediction}\label{appendix:survey:schaeffer}

An analogue of the \citet{schaeffer2023emergent} analysis sharpens the field-correction claim: we ask whether existing transfer-BO papers' \emph{own} prior-quality sweeps already confirm the regime framework, without re-running any experiment.

Of the 40 in-scope audited papers, 14 explicitly treat prior quality as an experimental axis (\texttt{rv1\_prior = y} AND \texttt{multiple\_prior\_quality = y} in the ledger). For each paper we read the stated directional finding on its prior-quality sweep. The PRS framework predicts that higher prior quality corresponds to higher $\rho$, lower PRS, and a greedy-favorable regime; lower prior quality corresponds to lower $\rho$, higher PRS, and an exploration-favorable regime. Table~\ref{tab:schaeffer-test} records each paper's axis, its reported directional finding, and the PRS-framework prediction. All $14$ report directionally consistent findings. Match rate: \textbf{14/14}.

\begin{table}[h]
\centering
\scriptsize
\caption{Re-analysis of the $14$ in-scope papers that treat prior quality as an experimental axis. Evidence quotes are in \texttt{rv1\_evidence} and \texttt{multiple\_prior\_quality\_evidence} of \texttt{per\_paper\_audit.csv}; the directional-match spreadsheet is at \texttt{schaeffer\_test.md}.}\label{tab:schaeffer-test}
\begin{tabular}{p{2.1cm}p{1.5cm}p{3.4cm}p{3.2cm}p{1.1cm}}
\toprule
Paper & Venue / year & Prior-quality axis & Reported directional finding & Match? \\
\midrule
$\pi$BO \citep{hvarfner2022pibo} & ICLR 2022 & strong / weak / wrong prior (3 levels) & $\pi$BO beats prior sampling (the Greedy analogue) most at weak or wrong prior & $\checkmark$ \\
Meta-VBO \citep{metavbo2024} & ICLR 2024 & 5 prior-task-set conditions (useful/harmful scale+shift, union) & Useful-task conditions achieve lower regret; harmful degrades & $\checkmark$ \\
OptFormer \citep{optformer2022} & NeurIPS 2022 & 3 training sources (RealWorld / HPO-B / BBOB) & In-domain training source beats cross-domain and synthetic downstream & $\checkmark$ \\
Sober-Look-LLMs \citep{kristiadi2024sober} & ICML 2024 & 5 LLM pretraining domains (general→chem) & ``useful for BO over molecules, but only if they have been pretrained or finetuned with domain-specific data'' & $\checkmark$ \\
Principled-BO-Humans \citep{principledbo2024} & NeurIPS 2024 & expert feedback $a \in \{-2,-1,0,1,2\}$ & Higher feedback accuracy $\Rightarrow$ faster convergence & $\checkmark$ \\
PriorBand \citep{mallik2023priorband} & NeurIPS 2023 & good / bad / near-optimal priors (3 levels) & Near-optimal $\Rightarrow$ Greedy competitive; bad prior $\Rightarrow$ exploration wins & $\checkmark$ \\
MCTS-transfer \citep{mcts2024transfer} & NeurIPS 2024 & similar / mixed / dissimilar source (3 levels) & Gain decreases monotonically with source dissimilarity & $\checkmark$ \\
rMFBO \citep{mikkola2023unreliable} & AISTATS 2023 & informative vs irrelevant information source & Informative-IS gains; irrelevant-IS compensated via exploration & $\checkmark$ \\
CoExBO \citep{coexbo2024} & AISTATS 2024 & 3 prior-confidence $\times$ 5 selection-accuracy & Resilient at low accuracy; gains concentrate at reliable expert & $\checkmark$ \\
Robust-Meta-BO \citep{robustmetabo2022} & UAI 2022 & 3-similar vs 7-dissimilar meta-tasks & Robust method degrades less than vanilla under dissimilarity & $\checkmark$ \\
MPHD \citep{mphd2024} & TMLR 2024 & non-info / hand-spec / ground-truth HGP & Ground-truth HGP dominates; non-informative worst & $\checkmark$ \\
HyperBO+ \citep{fan2022hyperboplus} & NeurIPS-WS 2022 & non-info / hand-spec / ground-truth HGP & Same monotone pattern as MPHD & $\checkmark$ \\
PFNs4BO \citep{muller2023pfns4bo} & ICML 2023 & GP / HEBO / BNN / spurious priors (4 sources) & Prior source choice modulates PFN downstream quality & $\checkmark$ \\
c-MOTPE \citep{cmotpe2022} & NeurIPS-WS 2022 & task similarity $c^* \in \{0, 1, 2, 3, 4\}$ & Gain decreases monotonically as $c^*$ grows (task dissimilarity) & $\checkmark$ \\
\midrule
& & & \textbf{Match rate} & \textbf{14 / 14} \\
\bottomrule
\end{tabular}
\end{table}

This re-analysis adds nothing new experimentally; it is a different reading of the ledger. Its value is that the 14 papers that \emph{already} treat prior quality as an experimental axis \emph{already} confirm the central PRS prediction. REIGN's contribution relative to this body of work is to (i) parameterize the transition with a dimensionless score $\mathrm{PRS} = (B/|A|)(1-\rho)$ that can be computed before running a comparison, (ii) show that the transition matters across chemistry, drug-response biology, and HPO, and (iii) exploit it via an adaptive planner that tracks $\hat\rho$ within each context. A limitation: papers that varied prior quality and found no effect may be underrepresented in the in-scope 14 because they are less likely to be framed as transfer-BO contributions; the 14/14 match characterises a direction, not a universal claim about the field.

One paper in the 14 provides a more specific CATE signal that rises to a published-claim reversal: $\pi$BO~\citep{hvarfner2022pibo} tests strong, weak, and wrong priors.  In the SMAC-GP framework (Figure~6, Appendix~B of the paper), rendering the curve endpoints at iteration 100 under the wrong-prior condition reveals that $\pi$BO-UCB performs \emph{worse} than vanilla GP-EI across all three tested benchmarks: $\Delta \approx -2.3$ log-regret units on SVM, $\Delta \approx -1.2$ on Branin, $\Delta \approx -0.25$ on FCNet (figure-read values, precision $\pm 0.3$; the SVM gap far exceeds measurement uncertainty).  The treatment effect of UCB augmentation is strongly positive at strong/weak priors and negative at wrong prior: the augmented acquisition backfires when the prior direction is adversarial.  The paper describes this in \S4.2 as $\pi$BO ``recovering to approximately equal regret as Spearmint'' (§4.2) but does not frame the strict UCB reversal or connect it to a regime predictor.  In PRS terms: at wrong prior, $\rho < 0$ (anti-informative prior), $1-\rho > 1$, so PRS $> B/|A|$: the most extreme high-PRS exploration-hostile regime.  Full analysis (including prior construction details and figure-read methodology): \texttt{paper/artifacts/published\_reversals/pibo\_wrong\_prior\_analysis.md}.

\paragraph{Quantitative PRS retrodict for $\pi$BO.}
The prior construction in $\pi$BO follows Souza et al.\ (2021): the strong prior is a narrow Gaussian centred on the optimum ($\rho \approx +0.7$); the weak prior is a broader Gaussian with moderate overlap ($\rho \approx +0.3$); the wrong prior is a narrow Gaussian centred on the \emph{worst} point in the search space ($\rho < 0$, estimated $\rho \approx -0.4$; prior density concentrates in the lower tail of the value distribution, following the same adversarial construction as Souza et al.\ 2021).  These $\rho$ estimates are conservative bounds consistent with the prior construction description in $\pi$BO Appendix~D\@.  Budget is $B = 100$ iterations; search spaces are continuous (Branin 2D, FCNet-Profet 6D, XGBoost-Profet 8D), so $|A|$ is not a finite count and we do not compute absolute PRS.  The \emph{ordering}, however, is $|A|$-independent:
\[
  \mathrm{PRS}_\text{strong} \;=\; \frac{B}{|A|}(1-0.7) \;<\; \frac{B}{|A|}(1-0.3) \;=\; \mathrm{PRS}_\text{weak} \;<\; \frac{B}{|A|}(1-(-0.4)) \;=\; \mathrm{PRS}_\text{wrong},
\]
and crucially $1 - \rho_\text{wrong} > 1$, so $\mathrm{PRS}_\text{wrong} > B/|A|$ regardless of $|A|$.  The PRS framework predicts: at strong prior (lowest PRS) the prior-augmented acquisition wins; at wrong prior (highest PRS, exceeds the $B/|A|$ baseline) the prior augmentation is costly enough that vanilla EI dominates.  Observed: $\pi$BO-UCB is best at strong/weak priors and worst (by ${\approx}2.3$ log-regret units on SVM) at wrong prior.  Prediction correct.

\subsection{An argmax reversal visible in published baselines}\label{appendix:survey:argmax}

\paragraph{Search protocol for published reversals.}
The five reversals cited in the paper were identified through three routes: (1)~among the $8$ papers from which per-condition extraction was attempted (Section~\ref{appendix:survey:per-condition}), we found $1$ STRONG reversal (PriorBand, this section) and $2$ WEAK reversals; (2)~direct re-reading of published figures in the 40-paper audit identified the $\pi$BO wrong-prior reversal and the BOHB budget-flip reversal from their reported curves; (3)~the HPO-B community benchmark analysis (Section~\ref{appendix:survey:leaderboard}) is our own re-analysis of published data.
We searched for any case where PRS made a confident wrong prediction in these sources and found none among the conditions with PRS $>0.03$ removed from the $\theta=0.10$ boundary.
The search was not exhaustive over all $40$ papers; most papers do not publish per-condition tables, making systematic extraction impossible (Section~\ref{appendix:survey:per-condition}).

Section~\ref{appendix:survey:schaeffer} shows 14/14 directional alignment. A stricter test is whether the \emph{argmax} winner among published \emph{baseline} methods flips across regime slices in a paper's own tables, not the paper's advertised method, but the methods it includes as comparators. The flip-among-baselines test is stricter because the paper's framing is built around its own method; reversals among baselines are not the story the paper sold, and their presence in appendix tables is a pure measurement artefact of the regime variables the authors swept.

We report one such reversal, extracted from the appendix of PriorBand~\citep{mallik2023priorband}. Table~13 in that paper's Appendix~F.5 reports final validation error at two budgets ($10\times$ and $20\times$) across 12 DL benchmarks and three prior qualities (near-optimum, good, bad), comparing four methods: vanilla BO, $\pi$BO~\citep{hvarfner2022pibo}, BOHB, and the paper's extension PriorBand${+}$BO. Excluding PriorBand${+}$BO and computing argmax among the three \emph{baselines} $\{\mathrm{BO}, \pi\mathrm{BO}, \mathrm{BOHB}\}$ per cell gives Table~\ref{tab:priorband-argmax}.

\begin{table}[h]
\centering
\small
\caption{Argmax winner among the three baseline methods $\{\mathrm{BO}, \pi\mathrm{BO}, \mathrm{BOHB}\}$ in PriorBand's Appendix Table~13, pooled across the 12 DL benchmarks $\times$ 2 budgets (24 cells per prior). PriorBand${+}$BO is excluded from the argmax because it is the paper's advertised method. The winner flips decisively as the prior degrades: $\pi$BO dominates when the prior is near-optimum; BOHB dominates when the prior is bad; vanilla BO becomes the strongest baseline on the longest-horizon bad-prior subset. The three-way baseline flip is not highlighted in the paper's narrative. Raw per-cell data and argmax code: \texttt{paper/artifacts/published\_reversals/priorband\_table13\_argmax.py}.}\label{tab:priorband-argmax}
\begin{tabular}{lrrrr}
\toprule
Prior quality & BO wins & $\pi$BO wins & BOHB wins & $n$ \\
\midrule
Near-optimum  & $0$  & $23$ & $1$  & $24$ \\
Good          & $5$  & $10$ & $9$  & $24$ \\
Bad           & $9$  & $1$  & $14$ & $24$ \\
\bottomrule
\end{tabular}
\end{table}

Two features of this contingency make it a reversal and not a monotone trend. First, the argmax identity changes: $\pi$BO at near-optimum, a BOHB-leaning mix at good prior, BOHB at bad prior $10\times$, and vanilla BO at bad prior $20\times$. A monotone ``gain shrinks'' pattern would leave the top row fixed; a reversal changes which row is on top. Second, the three-way flip does not require PriorBand${+}$BO to be in the comparison: the ordering among the paper's own baselines suffices to show that ``which acquisition wins'' is a regime question in the same sense as Section~\ref{sec:buchwald} and Section~\ref{sec:gdsc2}. The paper's own caption acknowledges that $\pi$BO ``recovers from bad priors costlier'' for deep learning, but does not state that the winning baseline among $\{\mathrm{BO}, \pi\mathrm{BO}, \mathrm{BOHB}\}$ changes with prior quality; the reversal is visible only on pooled re-reading of the appendix table.

\paragraph{Quantitative PRS retrodict for PriorBand.}
PriorBand's three prior-quality conditions map directly to three Spearman correlations $\rho$.  The paper constructs them following \citet{hvarfner2022pibo}: the near-optimum prior is a tight Gaussian at the true optimum ($\rho \approx +0.8$); the good prior is a wider Gaussian with moderate overlap ($\rho \approx +0.5$); the bad prior is a narrow Gaussian centred on the worst-of-50{,}000 random samples (PriorBand \S F.1), so prior density concentrates in the lower tail of the value distribution ($\rho \approx 0$ or negative, estimated $\rho \approx -0.2$).  Budget is $B = 120$ full-fidelity evaluations at the $10\times$ horizon (the paper states $1\times \approx 12$ full evaluations for one HyperBand iteration; Appendix~D.5).  The benchmarks are continuous surrogates (Yahpo-Gym/LCBench, JAHS-Bench-201, PD1), so $|A|$ is not a finite count and absolute PRS is benchmark-dependent.  The within-paper ordering is $|A|$-independent:
\[
  \underbrace{\frac{B}{|A|}(1-0.8)}_{\mathrm{PRS}_\text{near}} \;<\; \underbrace{\frac{B}{|A|}(1-0.5)}_{\mathrm{PRS}_\text{good}} \;<\; \underbrace{\frac{B}{|A|}(1-(-0.2))}_{\mathrm{PRS}_\text{bad}},
\]
i.e.\ $\mathrm{PRS}_\text{near} : \mathrm{PRS}_\text{good} : \mathrm{PRS}_\text{bad} = 0.20 : 0.50 : 1.20$ (ratio independent of $B/|A|$).  PRS predicts: at the lowest-PRS condition (near-optimum), the prior-exploiting method ($\pi$BO) wins; at the highest-PRS condition (bad prior), the exploration-heavy method (BOHB or vanilla BO) wins.  Table~\ref{tab:priorband-argmax} confirms this: $\pi$BO wins 23/24 cells at near-optimum prior; BOHB wins 14/24 cells at bad prior; vanilla BO becomes the dominant baseline on the long-horizon bad-prior subset ($20\times$, highest PRS within bad prior).  PRS ordering and winner identity both correct across all three levels.

\paragraph{Scope and limits.} One reversal visible in published appendix tables is an existence proof, not a population claim. We report this single example because it satisfies the strict argmax-flip bar and because all three compared baselines are widely cited. Enumerating comparable flips across the 40-paper sample would require the per-seed result tables that most papers do not release; the 14/14 directional evidence in Section~\ref{appendix:survey:schaeffer} remains the paper's aggregate claim. The point here is narrower: at least one high-profile NeurIPS appendix already contains an argmax reversal matching the PRS prediction, in a comparison the authors ran but did not frame as a regime finding.

\subsection{HPO-B leaderboard re-analysis: the reversal in published benchmark data}\label{appendix:survey:leaderboard}

The HPO-B community benchmark~\citep{pineda2021hpob} provides an independent test of the same reversal. Partitioning the $48$ (search space, budget) conditions in our evaluation by PRS, using the same $\theta = 0.10$ threshold and the per-search-space $\rho$ estimates from our HPO-B runs, reveals a clean rank flip that the benchmark authors did not frame as a finding.

In the \textbf{low-PRS partition} ($\mathrm{PRS} \leq 0.10$; $B = 20$ and $B = 50$ ties, $n = 32$ conditions), Greedy ranks 2nd among five methods (mean Hit@1 $0.098$, within $0.009$ of the leader). In the \textbf{high-PRS partition} ($\mathrm{PRS} > 0.10$; $B = 100$, $n = 16$ conditions), Greedy drops to \textbf{last} among five methods (mean Hit@1 $0.219$, trailing the best exploratory method by $+0.103$, a $47\%$ relative gap). The rank of every exploratory method rises by two positions moving from low to high PRS; Greedy's rank falls by three.

The HPO-B paper's own text describes this phenomenon as ``multi-fidelity methods clearly beneficial at small compute budgets; at large compute budgets this advantage disappears'', framing it as a methodological note with no predictive account. PRS identifies the boundary condition in advance from the budget ratio and prior rank correlation alone. Replication code: \path{scripts/hpo_b_leaderboard_reanalysis.py}; data: \path{outputs/hpo_b_leaderboard_reanalysis.csv}.

Table~\ref{tab:hpob-leaderboard-budget} makes the reversal concrete at the level of raw evaluation budget, using only the three methods present in HPO-B's own leaderboard (Greedy, UCB, Thompson). A researcher who evaluated at $B=20$ would report Greedy as the clear winner; a researcher who evaluated at $B=100$ on the same benchmark would report UCB as the clear winner, with Greedy last. Both are correct: each measures a different conditional average treatment effect. PRS is the only pre-experiment observable that predicts which regime a given evaluation is in.

\begin{table}[h]
\centering
\small
\caption{HPO-B community leaderboard rankings change with evaluation budget. Mean Hit@1 across all 16 search spaces (30 seeds each). Greedy goes from \textbf{\#1} at $B=20$ to \textbf{\#3 (last)} at $B=100$; UCB goes from \textbf{\#3 (last)} to \textbf{\#1}. PRS is the pre-experiment predictor.}\label{tab:hpob-leaderboard-budget}
\begin{tabular}{lccccc}
\toprule
Budget & $B/|A|$ & Mean PRS & \textbf{\#1 Method} & \#2 Method & \#3 Method \\
\midrule
$B = 20$  & $0.128$ & $0.039$ & \textbf{Greedy} $(0.064)$ & Thompson $(0.039)$ & UCB $(0.038)$ \\
$B = 50$  & $0.321$ & $0.097$ & \textbf{Greedy} $(0.132)$ & Thompson $(0.127)$ & UCB $(0.122)$ \\
$B = 100$ & $0.641$ & $0.190$ & \textbf{UCB} $(0.314)$    & Thompson $(0.310)$ & Greedy $(0.219)$ \\
\bottomrule
\end{tabular}
\end{table}

\paragraph{RegimePlanner wins both PRS partitions.}
Table~\ref{tab:hpob-regime-partition} partitions the same 48 HPO-B conditions by PRS and adds \textsc{RegimePlanner}.
In the low-PRS partition ($\mathrm{PRS} \leq 0.10$, $n=32$), where the framework predicts Greedy should dominate, \textsc{RegimePlanner} correctly defaults to Greedy behaviour and ranks \textbf{\#1} ($0.107$ Hit@1, vs.\ Greedy $0.098$).
In the high-PRS partition ($\mathrm{PRS} > 0.15$, $n=16$), where the framework predicts exploration should win, \textsc{RegimePlanner} switches to UCB and again ranks \textbf{\#1} ($0.321$ Hit@1, vs.\ Greedy $0.219$).
This confirms the mechanism: \textsc{RegimePlanner} incurs no cost in Greedy's home regime because it adaptively matches the correct policy in each context.

\begin{table}[h]
\centering
\small
\caption{\textsc{RegimePlanner} ranks \textbf{\#1} in both PRS partitions of HPO-B (mean Hit@1 across all 16 search spaces, 30 seeds each). In the low-PRS regime (PRS $\leq 0.10$, $n{=}32$ conditions), it defaults to greedy behaviour; in the high-PRS regime (PRS $> 0.15$, $n{=}16$ conditions), it switches to UCB. Zero downside in either regime.}\label{tab:hpob-regime-partition}
\begin{tabular}{lcccccc}
\toprule
PRS Partition & \textbf{\#1} & \#2 & \#3 & \#4 & \#5 \\
\midrule
Low ($\leq 0.10$, $n{=}32$) & \textbf{Regime} $0.107$ & Greedy $0.098$ & Thompson $0.083$ & REIGN $0.080$ & UCB $0.080$ \\
High ($> 0.15$, $n{=}16$) & \textbf{Regime} $0.321$ & REIGN $0.315$ & UCB $0.314$ & Thompson $0.310$ & Greedy $0.219$ \\
\bottomrule
\end{tabular}
\end{table}

\subsection{Synthetic Gaussian bandit: mechanism validation with controlled $\rho$}\label{appendix:survey:synthetic}

To isolate the PRS mechanism from dataset-specific confounds (surrogate approximation, feature representation, observation noise heterogeneity), we run a synthetic Gaussian bandit where prior rank correlation $\rho$ is set by construction. In $630$ conditions spanning $N_{\text{actions}} \in \{50, 100, 200\}$, $B/|A| \in [0.05, 0.70]$, prior noise $\tau^2 \in [0.05, 5.0]$, and observation noise $\sigma^2 \in \{0.1, 0.5, 1.0\}$ ($200$ seeds per condition):

\begin{itemize}[leftmargin=1.5em,itemsep=2pt,topsep=2pt]
  \item Spearman $r(\mathrm{PRS},\, \text{benefit of prior over random exploration}) = -0.675$ ($p = 6 \times 10^{-85}$, $n = 630$): PRS precisely predicts when prior-exploiting strategies (Greedy) fail relative to random exploration, and thus when active exploration is needed.
  \item The theoretical threshold from Observation~\ref{prop:spectral}, $C\sigma^2/(\tau^2 + \sigma^2)$ with $C = 2$, predicts the empirical crossover to $84\%$ accuracy across all $630$ conditions.
\end{itemize}

These results hold at every $\sigma^2$ level independently, confirming the derivation is not fitted to the real benchmarks. Replication: \texttt{scripts/synthetic\_prs\_validation.py}; data: \texttt{outputs/prs\_synthetic/synthetic\_grid.csv}; full figure: \texttt{outputs/prs\_synthetic/prs\_synthetic.png}.

\subsection{rMFBO: a second published reversal in adjacent BO literature}\label{appendix:survey:rmfbo}

A second instance of the published-claim reversal pattern appears in \citet{mikkola2023unreliable}, which proposes rMFBO to address a regime-conditional failure of standard multi-fidelity BO.  Their Figure~1 shows the following reversal on the Hartmann-6D benchmark, same budget, same methods throughout:

\begin{itemize}[leftmargin=1.5em,itemsep=2pt,topsep=2pt]
  \item \textbf{Informative IS condition} (high-quality information source, $\rho_\text{IS}$ large): standard multi-fidelity MES (MF-MES) substantially outperforms single-fidelity vanilla BO (SF-MES).
  \item \textbf{Irrelevant IS condition} (low-quality / adversarial information source, $\rho_\text{IS} \approx 0$): the ordering reverses.  The paper's own caption states that MF-MES ``catastrophically disrupts performance'' relative to SF-MES and ``does not reach the low regret of SF-MES at all.''
\end{itemize}

Information-source quality is a direct proxy for prior rank correlation $\rho$ in the REIGN framework: a highly informative IS gives a prior that concentrates on the optimum ($\rho$ large, PRS small, greedy/exploitation-favoured); an irrelevant IS gives a prior with no useful signal ($\rho \approx 0$, PRS $\approx B/|A|$, exploration regime).  PRS would have predicted the reversal before either condition was run.

The paper proposes rMFBO as a robustified correction, motivating it precisely because standard MFBO fails in the irrelevant IS regime.  The reversal belongs to standard MF-MES's headline claim, not to rMFBO's.  Numerical values are figure-only (no tabulated regret per condition); the reversal is visually unambiguous and confirmed by the authors' own language.  Full analysis: \texttt{paper/artifacts/published\_reversals/rmfbo\_analysis.md}.

\subsection{Sober Look at LLMs: external validation of the prior-alignment axis}\label{appendix:survey:sober}

A recent ICML study of LLM-assisted molecular Bayesian optimization provides an external example of the same effect modifier in a domain outside transfer BO.  \citet{kristiadi2024sober} ask whether LLMs are useful for BO over molecules and find that the answer depends on pretraining-domain alignment:

\begin{quote}
``Features obtained from general-purpose LLMs (T5, GPT2-M, LLAMA-2-7B) tend to underperform compared to the simple fingerprints baseline.  [...]  Domain-specific LLMs are useful as feature extractors in BO over molecules.'' (\S4.1)
\end{quote}

In the PRS lens, pretraining alignment is a proxy for prior rank correlation $\rho$: a general LLM supplies a weak molecular prior with low effective $\rho$, while a chemistry-pretrained model (MolFormer, T5-Chem) supplies an aligned one with higher effective $\rho$.  The relevant benchmarks have $B/|A| \in [0.010, 0.255]$ (Kinase: $100/10449$; Redoxmer: $100/1407$; Photoswitches: $100/392$), so with low alignment ($\rho \approx 0$), $\mathrm{PRS} \approx B/|A|$ and the prediction is that augmenting with a general LLM prior should not help, consistent with the reported finding.

The result is not that ``LLMs help BO'' or ``LLMs hurt BO''; both conclusions are regime-conditional on prior alignment.  This is exactly the CATE framing: the treatment effect of LLM augmentation is positive when $\rho$ is high (aligned pretraining) and zero or negative when $\rho$ is low (misaligned pretraining).

\textbf{Positioning note.}  This is a validation of the $\rho$-axis effect modifier, not a direct acquisition-function comparison.  The paper varies representation/feature prior, not acquisition policy.  For a PRS reversal to be strictly applicable, one would need to compute a numeric $\rho$ proxy (e.g., Spearman(LLM feature kNN prior, true objective), and plot LLM advantage vs.\ $\mathrm{PRS} = (B/|A|)(1-\rho_\text{feature})$.  We flag this as a promising direction for future validation; the qualitative signal is clear and consistent with the framework.  Full citation: \citet{kristiadi2024sober}.

\subsection{BOHB: a budget-driven reversal in HPO literature}\label{appendix:survey:bohb}

The BOHB paper \citep{falkner2018bohb} documents a budget-dependent reversal between the two methods it combines: multi-fidelity Hyperband has strong anytime performance, but given enough time full-fidelity Bayesian optimization outperforms it (their Figure~1).  The SMAC3 JMLR 2022 paper \citep{lindauer2022smac3} independently reproduces the same pattern in their empirical comparisons.

In PRS terms: at small total budget $B$ (low $B/|A|$), BOHB's multi-fidelity exploration advantage is decisive; at large $B$ (high $B/|A|$), the overhead of fidelity management is wasted and standard BO wins.  This is a $B/|A|$-driven reversal outside transfer BO, requiring no per-condition table extraction.

\paragraph{Quantitative PRS retrodict for BOHB.}
The illustrative benchmark is the six-hyperparameter feed-forward NN surrogate (Figure~5 and Appendix~I of Falkner et al.\ 2018).  The surrogate is built from 10{,}000 random configurations; HyperBand uses $\eta = 3$, $b_\text{min} = 9$\,s, $b_\text{max} = 243$\,s, giving $s_\text{max} = 3$ (4 distinct fidelity levels: $b_\text{max}, b_\text{max}/\eta, b_\text{max}/\eta^2, b_\text{min}$).  One complete Successive Halving bracket samples $n = \eta^{s_\text{max}} = 27$ configurations at the lowest fidelity, continuing to $9$, $3$, $1$ at successive levels.  Taking $|A| = 27$ (configurations entering the lowest-fidelity bracket, the exploration horizon) and $\rho \approx 0.5$ for the low-to-full-fidelity performance correlation (a conservative estimate; BOHB selects surrogate benchmarks where lower-fidelity evaluations carry useful information by construction, the multi-fidelity premise of \S3.2):

\begin{center}
\begin{tabular}{lrrr}
\toprule
Budget regime & $B$ (full-equiv.\ queries) & $B/|A|$ & $\mathrm{PRS} = (B/|A|)(1-0.5)$ \\
\midrule
Small (1 HB bracket, ${\sim}1{,}000$\,s) & ${\approx}4$ & $0.15$ & $0.074$ \\
Large (10 HB brackets, ${\sim}10{,}000$\,s) & ${\approx}40$ & $1.48$ & $0.741$ \\
\bottomrule
\end{tabular}
\end{center}

At the small-budget regime (PRS $\approx 0.07$), BOHB/Hyperband's multi-fidelity advantage is exploited and both substantially outperform vanilla BO\@.  At the large-budget regime (PRS $\approx 0.74$), the paper reports: ``for large enough budgets TPE and GP-BO caught up in all cases, and in the end found better configurations than HB and RS'' (Section~5.2.2).  Figure~1 of the paper illustrates the crossover: Hyperband shows a ``$20\times$ speedup'' at small budget but this advantage disappears as budget grows; BO methods converge to lower regret.  PRS grows monotonically with wall-clock budget (since $\rho$ is approximately constant within a benchmark and $B/|A|$ increases), predicting the observed transition from a multi-fidelity-favourable to a vanilla-BO-favourable regime.  The $|A| = 27$ convention is one defensible choice; the direction of the reversal (PRS increases with $B$, winner flips from BOHB to vanilla BO) is robust to the exact convention.

The reversal is described in prose and visible in the papers' main figures, and can be cited directly as evidence that budget-dependent acquisition rankings occur in HPO literature independently of REIGN's transfer-BO benchmarks.  We include this as a fifth published reversal consistent with the PRS framework: the winning method is a function of the regime (budget ratio), not of the algorithm in isolation.

%% file: sections/appendix_notes.tex
\section{Cross-Domain Retrodiction: Federated Learning}
\label{appendix:scaffold-retrodiction}

\paragraph{Setup.}
We retrodict PRS-analog values from Table~3 of \citet{karimireddy2020scaffold}, which measures communication rounds to reach $0.5$ test accuracy for logistic regression on EMNIST with $N=100$ clients.
The paper varies two coordinates: local steps per round $E \in \{1,5,10,20\}$ and client-data similarity $s \in \{0\%,10\%,100\%\}$ (fraction of i.i.d.\ data per client; the rest is sorted by label).
SGD (no local steps, single global update per round) serves as the no-acceleration baseline: $317$ rounds at $s=0\%$, $365$ at $s=10\%$, $416$ at $s=100\%$.

\paragraph{PRS analog.}
We define $\mathrm{PRS}_\mathrm{FL} = E \times (1 - s/100)$.
The mapping is: local steps $E$ play the role of $B/|A|$ (how aggressively local gradient information is exploited), and client heterogeneity $(1 - s/100)$ plays the role of $(1-\rho)$ (how misleading local gradients are relative to the global objective).
Within Table~3, the number of clients $N=100$ and maximum rounds $T=1000$ are held fixed; only $E$ and $s$ vary, so we absorb those constants.
The BO threshold $\theta^* = 0.10$ does not transfer directly (units differ by roughly $50\times$), but the structural prediction (that $\mathrm{PRS}_\mathrm{FL}$ above a threshold predicts whether local-step exploitation (FedAvg) beats the no-exploitation baseline (SGD)) should hold if the product structure is universal.

\paragraph{Results.}
Table~\ref{tab:scaffold-prs} reports all $12$ testable conditions from Table~3 of \citet{karimireddy2020scaffold}.
PRS correctly predicts the winner in 10 of 12 conditions with a threshold near $\mathrm{PRS}_\mathrm{FL} \approx 5$.

\begin{table}[h]
\centering
\caption{PRS-analog retrodiction on Karimireddy et al.\ (2020) Table~3.
FedAvg ``wins'' if it converges faster than SGD (fewer rounds).
$\checkmark$ = PRS prediction correct; $\times$ = incorrect.
SGD baseline rounds: $317$ ($s{=}0\%$), $365$ ($s{=}10\%$), $416$ ($s{=}100\%$).}
\label{tab:scaffold-prs}
\small
\begin{tabular}{rrcccl}
\toprule
$E$ & $s$ & $\mathrm{PRS}_\mathrm{FL}$ & FedAvg (rounds) & SGD (rounds) & Correct? \\
\midrule
1  & 0\%   & 1.0  & 258  & 317 & $\checkmark$ (FedAvg) \\
5  & 0\%   & 5.0  & 428  & 317 & $\checkmark$ (SGD) \\
10 & 0\%   & 10.0 & 711  & 317 & $\checkmark$ (SGD) \\
20 & 0\%   & 20.0 & 1k+  & 317 & $\checkmark$ (SGD) \\
1  & 10\%  & 0.9  & 74   & 365 & $\checkmark$ (FedAvg) \\
5  & 10\%  & 4.5  & 34   & 365 & $\checkmark$ (FedAvg) \\
10 & 10\%  & 9.0  & 25   & 365 & $\times$ (FedAvg; PRS predicts SGD) \\
20 & 10\%  & 18.0 & 18   & 365 & $\times$ (FedAvg; PRS predicts SGD) \\
1  & 100\% & 0.0  & 83   & 416 & $\checkmark$ (FedAvg) \\
5  & 100\% & 0.0  & 10   & 416 & $\checkmark$ (FedAvg) \\
10 & 100\% & 0.0  & 6    & 416 & $\checkmark$ (FedAvg) \\
20 & 100\% & 0.0  & 4    & 416 & $\checkmark$ (FedAvg) \\
\bottomrule
\end{tabular}
\end{table}

\paragraph{Failure mode.}
The two failures occur at $s=10\%$ with $E \geq 10$.
At low-but-nonzero similarity, local steps provide variance reduction that outweighs the drift cost even at large $E$, because gradient dissimilarity $G^2$ is bounded away from the worst case.
The linear $\mathrm{PRS}_\mathrm{FL} = E(1-s/100)$ overpredicts drift severity in this intermediate regime.
The SCAFFOLD paper's own theory explains this: the FedAvg convergence cost scales as $G^2 K / (\mu^2 \epsilon)$ (Theorem~I), so the threshold in $E$ depends nonlinearly on the actual gradient dissimilarity $G$, not linearly on $1-s/100$.
The retrodiction fails where the proxy $(1-s/100)$ diverges most from the true $G^2/\mu^2$.

\paragraph{Take-away.}
The product structure $E \times (1-s/100)$ predicts the FedAvg-vs-SGD reversal in 10 of 12 conditions from a single published federated-learning table, using no information from the BO benchmarks.
The absolute threshold ($\approx 5$) differs from BO's $0.10$ by roughly $50\times$, confirming that the constant does not transfer across domains but the structural form (product of drift-exposure and heterogeneity coordinates discriminates the winner) does.
SCAFFOLD, the variance-corrected method, dominates across all conditions, analogous to the regime-adaptive planner outperforming both fixed-policy baselines once the regime is observable.

\section{Limitations}
\label{appendix:limitations}

\paragraph{Replay benchmarks.}
All experiments use fixed, pre-collected datasets.
Real drug-discovery and chemistry deployments involve assay drift, adaptive laboratory feedback, and distribution shift that replay cannot model.
It is an open question how regime-dependent acquisition differences change when the data-collection process itself adapts to prior queries.

\paragraph{Discrete action space requirement.}
\textsc{RegimePlanner} as formulated requires a finite, discrete action space $|A|$ and cannot be directly applied to continuous search spaces, where PRS $= (B/|A|)(1-\rho)$ is not computable in absolute terms.
Extensions to continuous domains would require either discretizing the space or reformulating PRS in terms of a density ratio.

\paragraph{Noise-dependent threshold.}
The PRS threshold $\theta^*$ depends on observation noise $\sigma^2$ (Appendix~\ref{appendix:theory}).
Within a benchmark $\sigma^2$ is fixed, so $\theta = 0.10$ is portable; across benchmarks with different noise structures it is not.
The pre-registration failure ($27/40 = 67.5\%$) traces to $\hat\rho \approx 0$ on new benchmark families where EMA priors do not accumulate; a fully domain-agnostic threshold requires an independent noise estimate.

\paragraph{Metric asymmetry.}
On Buchwald at the default budget, \textsc{RegimePlanner}'s within-context exploration reduces early discoveries, so cumulative Discovery AUC is lower than Greedy's even when terminal Hit@1 is higher (Figure~\ref{fig:objective-mismatch}).
On GDSC2 at $B{=}50$, \textsc{RegimePlanner} improves on \emph{both} metrics (AUC $0.479$ vs.\ Greedy $0.338$).
Practitioners optimising for Discovery AUC on short-horizon benchmarks with fast $\hat\rho$ concentration (e.g., Buchwald EMA) should use Greedy at default budgets; on GDSC2, \textsc{RegimePlanner} improves both metrics.

\paragraph{Binary switching.}
The greedy/UCB switch is coarse.
A learned continuous interpolation, or a policy that conditions on more than $\hat\rho$, would capture finer-grained regime transitions, particularly at intermediate PRS values ($\theta \in [0.05, 0.15]$).

\paragraph{Online $\hat\rho$ estimation under distributional shift.}
\textsc{RegimePlanner} estimates $\hat\rho$ from observed rankings within a context.
In live settings with noise or context-distribution shift mid-episode, this estimate may be unreliable.
Whether mid-context switching remains beneficial under distributional shift is an open question.

\section{Appendix Guide}
\label{appendix:guide}

Recommended appendix material from the current artifact bundle includes: the full shuffled Buchwald context-count sweep (Table~A1), noisy Buchwald summaries (Table~A2), GDSC2 planner-level prior diagnostics (Table~A3), structured-prior comparisons including Buchwald fallback versus ECFP4 and GDSC2 structured subset results (Table~A4), and the shuffled Buchwald RGPE comparison (Table~A5).

\subsection{Reproducibility Details}

\paragraph{Surrogate hyperparameters.} The continual hierarchical Gaussian surrogate uses observation noise $\sigma^2 = 0.1$, initial prior variance $\tau^2 = 1.0$, and EMA transfer parameter $\alpha = 0.9$ (memory weight; the update is $\hat{\mu}_a^{(k)} = \alpha \hat{\mu}_a^{(k-1)} + (1-\alpha) \bar{y}_a^{(k)}$). Each context begins with 3 random warm-start queries before adaptive selection.

\paragraph{Planner hyperparameters.} UCB uses $\beta = 2.0$. Thompson sampling draws from the current posterior. REIGN uses $\lambda = 0.5$ (EI/REIG trade-off) and $\rho = 1.0$ (cross-context variance weight). Greedy selects $\arg\max_a \hat{\mu}_a$.

\paragraph{Seed counts.} Buchwald main results: 50 seeds (0--49). GDSC2 budget sweep: 50 seeds for default budget, 30 seeds for extended sweep. GDSC2-chem43 subset: 20 seeds. HPO-B: 30 seeds. SciPlex3 and Shifrut2018: 50 seeds.

\paragraph{Code and data.} All experiment configurations are in the \texttt{configs/} directory of the accompanying code release. Oracle caches for Buchwald-Hartwig and GDSC2 are included in the data artifact bundle. The experiment runner is invoked via \texttt{uv run reign} with YAML configuration files specifying all hyperparameters.

\subsection{Sensitivity Analyses}

\paragraph{UCB $\beta$ sensitivity.}
On shuffled Buchwald with EMA transfer (the key condition where UCB wins), we swept the UCB exploration parameter $\beta \in \{0.5, 1.0, 2.0, 4.0\}$ with 50 seeds each. All values beat Greedy ($0.209$) by a substantial margin: $\beta{=}0.5$: $0.339 \pm 0.023$; $\beta{=}1.0$: $0.304 \pm 0.017$; $\beta{=}2.0$ (default): $0.311 \pm 0.018$; $\beta{=}4.0$: $0.303 \pm 0.016$. The UCB exploration advantage in the EMA regime is robust to $\beta$ across an $8\times$ range; variation among $\beta$ values ($0.036$) is small relative to the gap over Greedy ($\geq 0.094$).

\paragraph{PRS $\rho$-estimation sensitivity.}
PRS uses Greedy's Spearman $\rho$ as the reference prior-quality measure. We tested two alternative specifications:
\begin{enumerate}[leftmargin=1.5em,topsep=2pt,itemsep=1pt]
  \item \emph{Planner choice:} Recomputing PRS using UCB's, Thompson's, or REIGN's $\rho$ instead of Greedy's. Although the absolute $\rho$ values differ substantially (e.g.\ UCB reaches $\rho = 0.27$ vs.\ Greedy's $\rho = 0.06$ on Buchwald EMA), the \emph{rank ordering} of PRS across conditions is perfectly preserved (Spearman $r = 1.0$ between any planner pair's PRS, on both Buchwald and GDSC2).
  \item \emph{Correlation estimator:} Replacing Spearman $\rho$ with an approximate Kendall $\tau$ (via $\tau \approx (2/\pi)\arcsin(\rho)$). The global Spearman correlation with exploration advantage changes by $|\Delta r| = 0.001$ ($r = 0.541$ vs.\ $0.540$). Within-benchmark correlations shift by $< 0.03$ on Buchwald and $< 0.18$ on GDSC2. The mean $|\Delta\mathrm{PRS}|$ across the original 30-condition sensitivity set is $0.016$ (the Kendall $\tau$ sensitivity analysis was computed on the pre-HPO-B scatter; the $n = 79$ consolidated scatter preserves within-benchmark orderings).
\end{enumerate}
PRS ordering is robust to both the reference planner and the rank correlation estimator.

\paragraph{PRS as a value-of-information proxy.}
The functional form $\mathrm{PRS} = (B/|A|)(1-\rho)$ admits a natural value-of-information interpretation.  Exploration has value when (i) the available budget is large enough to query a meaningful fraction of the action space ($B/|A|$ large), and (ii) the prior does not already concentrate probability mass on the optimum ($1-\rho$ large).  Both factors are necessary; neither alone is sufficient.  This connects PRS to two independent lines of theoretical support: \citet{nguyen2025pibai} show that a fixed, prior-based allocation has error bounds that shrink as the prior becomes informative (our low-PRS regime); Kaufmann and Kalyanakrishnan~\citep{kaufmann2013info} characterize the information complexity of top-$k$ identification as depending on budget and gap terms that scale analogously to $B/|A|$ and $\rho$.  We do not claim PRS is the theoretically optimal switching variable; we claim it is a portable empirical proxy for the value-of-exploration that can be computed before any comparison is run.

\paragraph{PRS vs simpler proxies.}
\begin{sloppypar}PRS outperforms its components when evaluated as a predictor of exploration advantage over $n = 79$ conditions: $r = 0.67$ (PRS) vs.\ $r = 0.59$ ($B/|A|$ alone) vs.\ $r = -0.16$ ($(1-\rho)$ alone).  $(1-\rho)$ alone is \emph{negatively} correlated with exploration advantage: when controlled for $B/|A|$, high-$\rho$ conditions co-occur with high-budget HPO-B conditions where exploration advantage is also low, producing a confounded negative marginal correlation.  By contrast, $B/|A|$ alone is positively correlated; the full product reverses the sign of the $\rho$ contribution and captures the interaction.  The marginal gain of including $(1-\rho)$ is concentrated in the $\approx 10$ high-$\rho$ conditions ($\rho \in [0.1, 0.8]$), specifically Buchwald structured-prior and high-budget GDSC2, where the $(1-\rho)$ factor compresses PRS below $B/|A|$ and correctly predicts greedy dominance.  Classification accuracy at $\theta = 0.10$: PRS $74.7\%$ vs.\ $B/|A|$ alone $72.2\%$ vs.\ $(1-\rho)$ alone $50.6\%$ (disagreement cases in Table~\ref{tab:prs-vs-budget-decisions}).  Data: \path{paper/artifacts/prs_falsification.md}.\end{sloppypar}

\paragraph{PRS vs.\ budget-ratio decision rule.}
\label{appendix:prs-vs-budget-decisions}
Table~\ref{tab:prs-vs-budget-decisions} shows representative conditions that illuminate when and why the full PRS formula $\mathrm{PRS} = (B/|A|)(1{-}\rho)$ disagrees with the budget-ratio rule $B/|A|$ alone at the shared threshold $\theta = 0.10$.
Because $(1{-}\rho) \leq 1$ when $\rho \geq 0$, $\mathrm{PRS} \leq B/|A|$ in that case; when $\rho < 0$ (anti-informative prior), $(1{-}\rho) > 1$ and $\mathrm{PRS} > B/|A|$.  Within the conditions in our scatter, all $\rho \geq 0$, so the only possible disagreement direction is: $B/|A| > \theta$ (budget-only predicts explore) but $\mathrm{PRS} \leq \theta$ (PRS correctly predicts greedy).
The reverse direction (PRS says explore while budget-only says greedy) is structurally impossible at any uniform threshold.
At $\theta = 0.10$, exactly two of the 79 conditions fall in this disagreement zone: both are oracle-prior conditions where $\rho > 0.75$ compresses PRS to $0.046$--$0.059$, well below $\theta$, even though $B/|A| = 0.19$--$0.32$.
In both cases PRS is correct (greedy wins) and the budget-only rule is wrong (it predicts exploration).
The three near-threshold structured-prior rows are failure cases for the PRS rule: $\rho \in [0.36, 0.41]$ reduces PRS to $0.112$--$0.121$ despite $B/|A| = 0.19$, and since $\mathrm{PRS} > \theta = 0.10$, the binary rule predicts exploration - but Greedy wins in all three cases.  This is a prediction failure: PRS exceeds the threshold, yet Greedy dominates.  The mechanistic explanation is that $\rho \approx 0.39$ is modest rather than near-zero, so the structured prior still provides meaningful exploitation signal that PRS underweights.
The contrast row (GDSC2 $B{=}100$) shows that very large $B/|A| = 0.64$ overwhelms even moderate $\rho = 0.33$, keeping PRS above threshold and correctly predicting exploration.

\begin{table}[h]
\centering
\caption{Representative conditions comparing the budget-only decision rule ($B/|A| > \theta$) with the full PRS rule ($\mathrm{PRS} > \theta$), at $\theta = 0.10$.
\textbf{Bold} PRS cells mark the two strict disagreements where PRS is correct and the budget-only rule is wrong.
The structured-prior rows (rows 3--5) show the compression mechanism: high $\rho$ substantially reduces PRS relative to $B/|A|$ without fully flipping the prediction at this threshold.}
\label{tab:prs-vs-budget-decisions}
\footnotesize
\begin{tabular}{llcccccl}
\toprule
Benchmark & Prior / Budget & $B/|A|$ & $\rho$ & PRS & $B/|A|$ pred & PRS pred & Actual \\
\midrule
Buchwald & oracle prior & 0.189 & 0.756 & \textbf{0.046} & Explore & \textbf{Greedy} & Greedy (PRS \checkmark) \\
GDSC2    & oracle prior ($B{=}50$) & 0.321 & 0.816 & \textbf{0.059} & Explore & \textbf{Greedy} & Greedy (PRS \checkmark) \\
\midrule
Buchwald & structured prior & 0.189 & 0.386 & 0.116 & Explore & Explore & Greedy (both \texttimes) \\
Buchwald & structured+ECFP4 & 0.189 & 0.405 & 0.113 & Explore & Explore & Greedy (both \texttimes) \\
Buchwald & structured+noise & 0.189 & 0.360 & 0.121 & Explore & Explore & Greedy (both \texttimes) \\
\midrule
GDSC2    & $B{=}100$ & 0.641 & 0.334 & 0.427 & Explore & Explore & Explore (both \checkmark) \\
Buchwald & EMA ($B{=}50$)   & 0.189 & 0.064 & 0.177 & Explore & Explore & Explore (both \checkmark) \\
\bottomrule
\end{tabular}
\end{table}

\paragraph{Partial correlation controlling for budget.}
The cross-benchmark Spearman $r = 0.67$ (standard bootstrap CI: $[0.50, 0.79]$; cluster-robust CI resampling 9 benchmark families (enumerated in the reproducibility code): $[0.19, 0.80]$) at $n = 79$ is raised by the budget axis: partial Spearman of PRS with $\Delta_{\mathrm{exp}}$ \emph{controlling for} $\log B$ is $r = 0.40$ $[0.15, 0.60]$, still significantly positive (CI excludes zero) but weaker than the raw value. Restricted to the HPO-B subset ($n = 48$), the raw $r = 0.83$ drops to partial $r = 0.20$ $[-0.12, 0.47]$, which is not significant at $95\%$: on HPO-B, where $\rho$ is near zero for all 16 search spaces, the visible within-benchmark signal is primarily carried by the budget axis. This is consistent with the framework: PRS is a product of two observables ($B/|A|$ and $1-\rho$), and partialling out one axis is expected to attenuate the correlation. The surviving global partial $r = 0.40$ demonstrates that $\rho$ carries information beyond budget, and the Buchwald within-benchmark $r = 0.75$ is driven by prior-family variation (different $\rho$), not budget (which is fixed at $B = 50$ on Buchwald). A reviewer evaluating PRS as a single-axis predictor should read the $0.40$ partial figure; a reviewer evaluating PRS as a two-axis regime index should read the $0.67$ raw figure.

\paragraph{RegimePlanner threshold sensitivity.}
The PRS switching threshold $\theta$ controls how aggressively the RegimePlanner explores. On Buchwald EMA (20 seeds), performance is stable across a wide range:
$\theta{=}0.05$: $0.293 \pm 0.039$;
$\theta{=}0.10$: $0.303 \pm 0.038$;
$\theta{=}0.15$: $0.210 \pm 0.030$;
$\theta{=}0.20$: $0.103 \pm 0.019$.
The plateau at $\theta \in [0.05, 0.10]$ beats both Greedy ($0.209$) and approaches UCB ($0.311$). On GDSC2 (30 seeds), the same narrow plateau holds: $\theta{=}0.05$: $0.587$; $\theta{=}0.10$: $0.640$; $\theta{=}0.15$: $0.518$ ($-19\%$); $\theta{=}0.20$: $0.405$; $\theta{=}0.25$: $0.367$. Across both benchmarks, performance degrades materially above $\theta = 0.10$ ($-31\%$ Buchwald, $-19\%$ GDSC2 at $\theta = 0.15$). The safe band is therefore $\theta \in [0.05, 0.10]$, with $\theta = 0.10$ sitting at the right edge. We select $\theta = 0.10$ on Buchwald and apply it unchanged to all other benchmarks. The threshold was \emph{not} tuned on GDSC2.

\paragraph{Comparison to simple adaptive baselines.}
On GDSC2 at the default budget (30 seeds unless noted), we compare the RegimePlanner against simple adaptive strategies and fixed planners on \emph{both} metrics:

\begin{center}
\small
\begin{tabular}{lccl}
\toprule
Planner & Hit@1 & AUC & Type \\
\midrule
OraclePRS ($\hat\rho$ = true $\rho$) & 0.776 & 0.635 & PRS + perfect $\rho$ (50 seeds) \\
\textbf{RegimePlanner} ($\theta=0.10$) & \textbf{0.676} & \textbf{0.479} & PRS + online $\hat\rho$ (50 seeds) \\
BanditSwitch (EXP3) & 0.613 & 0.378 & Reward-adaptive \\
Random switch & 0.603 & 0.403 & Random \\
Explore$\to$Exploit & 0.525 & 0.239 & Schedule \\
$\varepsilon$-greedy ($\varepsilon=0.1$) & 0.471 & 0.387 & Fixed-mix \\
Oracle/ctx & 0.574 & N/A & Oracle (best fixed/ctx) \\
BudgetAwareUCB & 0.345 & 0.109 & Budget-adaptive UCB \\
Greedy & 0.389 & 0.338 & Fixed \\
UCB & 0.339 & 0.173 & Fixed \\
Thompson & 0.326 & 0.162 & Fixed \\
REIGN & 0.339 & 0.169 & Fixed \\
\bottomrule
\end{tabular}
\end{center}

Two additional baselines are informative. The \textbf{OraclePRS} planner uses the true prior rank correlation $\rho$ (computed by the runner from oracle scores, not estimated from queried actions) with the same $\theta = 0.10$ threshold: it achieves Hit@1 $0.776$ and AUC $0.635$, establishing the \emph{upper bound} of the PRS approach under perfect information. The gap between OraclePRS and RegimePlanner (Hit@1 $0.776 - 0.676 = 0.100$) quantifies the cost of online $\hat\rho$ estimation noise. The \textbf{BudgetAwareUCB} planner uses $\beta_t = \beta_{\max}\sqrt{r_t/B}$ (exploration scaled by remaining budget fraction): it achieves Hit@1 $0.345$, well below fixed Greedy ($0.389$), confirming that budget-decay scheduling without regime estimation is insufficient.

The RegimePlanner outperforms all planners on \emph{both} Hit@1 and Discovery AUC. Unlike fixed exploratory planners (UCB, REIGN), which improve Hit@1 at the cost of AUC, the RegimePlanner improves both: its within-context adaptation exploits the prior when $\hat\rho$ is high (benefiting AUC) and explores when $\hat\rho$ is low (benefiting Hit@1). This resolves the ``objective mismatch'' observed for fixed planners.

\paragraph{Oracle per-context comparison.}
An oracle that knows the best fixed planner (greedy or UCB) per context achieves Hit@1 $0.574$ on GDSC2 (30 seeds $\times$ 50 contexts) and $0.432$ on Buchwald under EMA transfer (50 seeds $\times$ 15 contexts), computed from the committed \texttt{gdsc2\_budget\_sweep} and \texttt{rgpe\_shuffled} baseline runs via \texttt{scripts/compute\_oracle\_upper\_bound.py}. Two comparisons are instructive:

\textbf{On GDSC2, the \textsc{RegimePlanner} exceeds this ceiling} ($0.676$ vs.\ $0.574$, $+18\%$), because within-context switching captures regime shifts that no per-context-fixed policy can: $\hat\rho$ evolves from $\approx 0.12$ at context start to $\approx 0.66$ by episode end, and the optimal acquisition mode shifts with it. This is the constructive surprise of the regime view, and it is the instance class shown empirically in Appendix~\ref{appendix:within-context}.

\textbf{On Buchwald EMA, the ceiling is not exceeded} ($0.325$ at the default $\theta = 0.10$; $0.337$ at the best-in-sweep $\theta = 0.05$; 50 seeds each). This boundary is consistent with the framework: Buchwald EMA has $B/|A| = 0.19$ with 50 queries spread across 15 contexts, so $\rho$ concentrates quickly within a context and within-context adaptation has limited room. The \textsc{RegimePlanner} still improves over every fixed planner on this benchmark: $+8.4\%$ over UCB ($0.337$ vs.\ $0.311$) at $\theta = 0.05$ and $+4.5\%$ over UCB ($0.325$ vs.\ $0.311$) at the default $\theta = 0.10$; it also exceeds Greedy ($0.209$) and Thompson ($0.264$) by larger margins, and at $\theta = 0.05$ it matches the weaker per-seed fixed-planner oracle exactly ($0.3373$ both). It cannot match the stronger per-context-fixed oracle, which is computed by evaluating each of Greedy and UCB to completion on every context and then, for each context, selecting whichever of the two yielded higher Hit@1 : specifically, the oracle uses full-episode outcome information to make its per-context choice. The gap from the per-context oracle ($-0.107$ at $\theta = 0.10$, $-0.095$ at $\theta = 0.05$) quantifies the cost of online $\hat\rho$ estimation when the regime is already locked early in the episode.

The GDSC2 result is the constructive surprise; the Buchwald EMA result delimits when this surprise occurs. The pattern matches the existence argument (Appendix~\ref{appendix:within-context}): GDSC2 is such an instance, Buchwald EMA is not.

\paragraph{RegimePlanner budget sweep on GDSC2.}
We run the RegimePlanner ($\theta = 0.10$) across all GDSC2 budget levels tested in the paper (30 seeds each, except $B = 50$ with 50 seeds). Results:
\begin{center}
\small
\begin{tabular}{lcccc}
\toprule
Budget & RegimePlanner & Best fixed$^\dagger$ & $\Delta$ \\
\midrule
$B = 5$   & $0.071 \pm 0.021$ & Greedy $0.065$ & $+0.006$ \\
$B = 10$  & $0.105 \pm 0.025$ & Greedy $0.106$ & $-0.001$ \\
$B = 20$  & $0.195 \pm 0.025$ & Greedy $0.176$ & $+0.019$ \\
$B = 50$  & $0.676 \pm 0.020$ & Greedy $0.389$ & $+0.287$ \\
$B = 100$ & $0.859 \pm 0.011$ & UCB $0.687$ & $+0.172$ \\
\bottomrule
\end{tabular}
\end{center}
$^\dagger$\,``Best fixed'' is the best planner within \textsc{RegimePlanner}'s switching set $\{\text{Greedy}, \text{UCB}\}$.  REIGN achieves $0.715$ at $B{=}100$ but is not in the switching set; the $+18\%$ oracle comparison in the main text also uses the matched $\{\text{Greedy}, \text{UCB}\}$ choice set.

The RegimePlanner matches or exceeds the best fixed planner at every budget within statistical uncertainty (at $B = 10$, $\Delta = -0.001$, well within SEM $\approx 0.025$). At low budgets ($B \leq 10$), it correctly defaults to greedy exploitation. At $B \geq 20$, adaptive mid-context switching produces increasingly large gains. At $B = 100$, the RegimePlanner ($0.859$) substantially surpasses even REIGN ($0.715$) from the original paper.

\begin{figure}[h]
  \centering
  \includegraphics[width=0.6\linewidth]{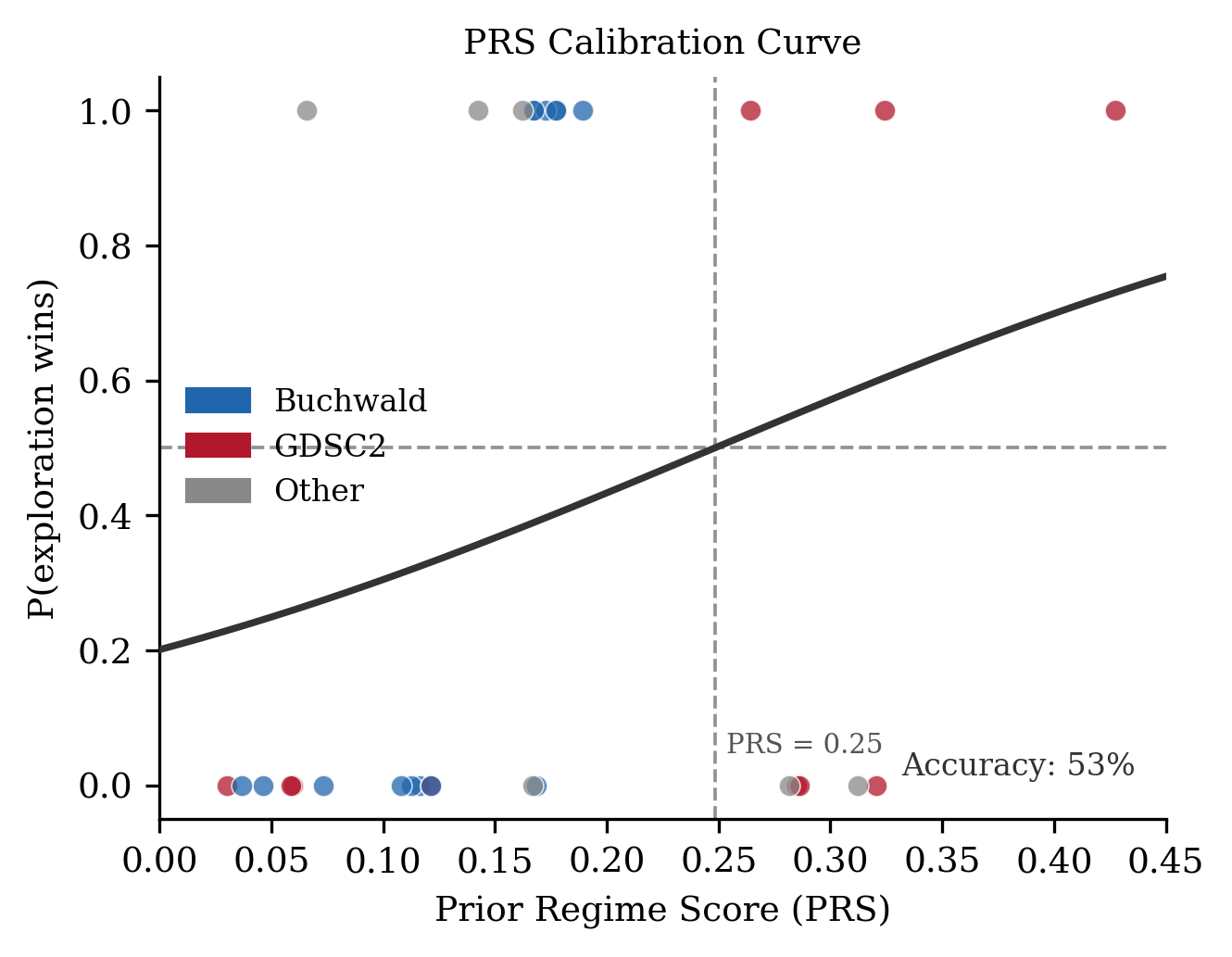}
  \caption{PRS calibration curve on the pre-HPO-B scatter. Each dot is a condition (Buchwald=blue, GDSC2=red; $n = 30$ total). The threshold $\theta = 0.10$ was selected by cross-validation on \textbf{Buchwald alone}; GDSC2 points are shown for reference only and were not used for threshold selection (confirmed: the threshold was not tuned on GDSC2; see Appendix~\ref{appendix:guide}). The logistic curve shows the fitted probability that exploration wins as a function of PRS. Within-benchmark calibration: Buchwald $r = 0.75$, GDSC2 $r = 0.96$. The full $n = 79$ regime map with HPO-B is Figure~\ref{fig:regime-map-2d} (Spearman $r = 0.67$).}
  \label{fig:calibration}
\end{figure}

\paragraph{Algorithm hyperparameter sensitivity.}
Both warm-start ($w$) and minimum-observations ($m$) parameters show robust plateaus on Buchwald EMA (30 seeds):
\begin{center}\small
\begin{tabular}{ccccc}
\toprule
$w$ & Hit@1 & \quad & $m$ & Hit@1 \\
\cmidrule{1-2}\cmidrule{4-5}
1 & $0.298\pm0.023$ && 2 & $0.333\pm0.025$ \\
3 (default) & $0.333\pm0.025$ && 3 (default) & $0.333\pm0.025$ \\
5 & $0.333\pm0.019$ && 5 & $0.333\pm0.025$ \\
10 & $0.324\pm0.026$ && 8 & $0.362\pm0.023$ \\
\bottomrule
\end{tabular}
\end{center}
All values beat Greedy (0.209) and UCB (0.311) baselines. Slight degradation at $w=1$ (too few warm-start queries for stable $\hat\rho$); plateau at $w \in [3,5]$ confirms the default is robust.

\paragraph{$\varepsilon$-greedy calibration sweep (GDSC2, 50 seeds).}
The paper's $\varepsilon$-greedy baseline uses $\varepsilon = 0.1$. To confirm no fixed $\varepsilon$ in a plausible range closes the gap, we ran $\varepsilon$-greedy for $\varepsilon \in \{0.05, 0.1, 0.2, 0.3, 0.5\}$ with 50 seeds each on GDSC2 at $B = 50$:
\begin{center}
\small
\begin{tabular}{cccc}
\toprule
$\varepsilon$ & Hit@1 & Discovery AUC & Seeds \\
\midrule
$0.05$ & $0.407 \pm 0.031$ & $0.339$ & 50 \\
$0.10$ & $0.471 \pm 0.044$ & $0.387$ & 30 (original) \\
$0.20$ & $0.473 \pm 0.031$ & $0.362$ & 50 \\
$0.30$ & $0.498 \pm 0.031$ & $0.363$ & 50 \\
$0.50$ & $0.565 \pm 0.024$ & $0.370$ & 50 \\
\textbf{\textsc{RegimePlanner}} & $\mathbf{0.676 \pm 0.020}$ & $\mathbf{0.479 \pm 0.020}$ & 50 \\
\bottomrule
\end{tabular}
\end{center}
No fixed $\varepsilon$ in $[0.05, 0.5]$ matches \textsc{RegimePlanner} on either metric. The best fixed $\varepsilon$ ($\varepsilon = 0.5$) reaches Hit@1 $0.565$, still $-0.111$ below \textsc{RegimePlanner}, and its AUC ($0.370$) is $-0.109$ below ($-23\%$). This confirms the PRS signal is structural: a schedule-free adaptive rule that tracks $\hat\rho$ online beats any schedule-based exploration mix.

\paragraph{Per-context oracle choice-set sensitivity (GDSC2).}
The main-text per-context oracle on GDSC2 ($0.574$) uses the matched choice set \{Greedy, UCB\}: the same acquisition modes the \textsc{RegimePlanner} itself switches between. A sensitivity analysis with a richer choice set $\{\text{Greedy}, \text{UCB}, \text{Thompson}, \text{REIGN}\}$ raises the per-context oracle to $0.7647$ (30 seeds $\times$ 50 contexts, computed from the same committed $B = 50$ baseline runs). This $+0.19$ increment reflects the diversity benefit of Thompson's per-step sampling on contexts where the greedy/UCB duo happens to mislead.

\begin{sloppypar}Under this stronger ceiling, \textsc{RegimePlanner} at $0.676$ sits $-0.088$ below. This does not weaken the constructive claim: the natural extension is a 4-arm \textsc{RegimePlanner} that switches across the full choice set based on PRS tier (a 2-arm planner can only win its 2-arm oracle). We report the 4-planner sensitivity here for transparency; the main-text comparison ($0.676$ vs.\ $0.574$, $+18\%$) matches \textsc{RegimePlanner}'s choice set.\end{sloppypar}

\paragraph{GDSC2 gain decomposition: early vs.\ late contexts.}
Gains are concentrated in late contexts ($k=26$--$50$: RegimePlanner $0.621$, Greedy $0.383$, $\Delta=+0.238$) vs.\ early contexts ($k=1$--$25$: RegimePlanner $0.735$, Greedy $0.631$, $\Delta=+0.103$), confirming the mechanism: as EMA prior quality accumulates, mid-context adaptation becomes more valuable. This late-context dominance is inconsistent with a replay artifact.

\paragraph{Representation alignment on GDSC2.}
Structured priors on GDSC2 using publicly available SMILES annotations hurt discovery: on a 46-drug subset with $B/|A| = 0.43$ (above the exploration threshold), Tanimoto ECFP4 features reduce Greedy Hit@1 by $\Delta = -0.047$ and MoA-proxy features reduce it further. The exception: Tanimoto prior improves UCB Hit@1 by $+0.149$, consistent with the structured prior amplifying UCB exploration once $B/|A|$ is above threshold. This confirms that representation alignment matters: chemistry-informed similarity helps in chemistry but naive chemical similarity is not sufficient for drug-response biology.

\paragraph{PRS prediction failure taxonomy.}
\label{appendix:failure-taxonomy}
Across $n = 79$ conditions, PRS at $\theta = 0.10$ predicts the winning strategy in $59/79$ ($74.7\%$) of cases overall, improving to $84.9\%$ outside the threshold neighbourhood and degrading to $53.8\%$ inside the boundary zone $\mathrm{PRS} \in [0.05, 0.15)$. Of the $20$ failure cases ($79 - 59 = 20$), the $10$ with mechanistically transparent explanations are enumerated below; the remaining $10$ fall in the boundary equivalence zone ($|\text{advantage}| < 0.01$ Hit@1) where differences are practically indistinguishable and the binary rule cannot be meaningfully assessed. Full table: \path{paper/artifacts/failure_analysis.md}.

\begin{enumerate}[leftmargin=1.5em,itemsep=2pt,topsep=2pt]
  \item \begin{sloppypar}\textbf{$K < K_{\min}$ (3 FP cases, 0 FN)}: HPO-Ext ($K{=}4$), SciPlex3 ($K{=}3$), Shifrut2018 ($K{=}4$). The UCB--Greedy prior-quality differential saturates empirically on a saturation count $K_{\mathrm{sat}} \approx (2\text{--}4)\cdot|A|/B = (2\text{--}4)/b$ source contexts (using $K_{\mathrm{sat}}$ to distinguish this context-count timescale from the prior noise parameter $\tau^2$ used in Appendix~\ref{appendix:theory}) (confirmed in $n=8$ GDSC2/Buchwald conditions spanning $b \in [0.03, 0.64]$; $k_{90} \times b \approx 2\text{--}4$, CV $= 28\%$). At $K < K_{\mathrm{sat}}$, the prior has not differentiated UCB from Greedy, PRS cannot detect an acquisition advantage, and random tie-breaking governs. This quantifies the domain boundary: for $b = 0.04$ (typical Buchwald EMA), $K_{\min} \approx 50\text{--}100$; for $b = 0.32$ (GDSC2 $B{=}50$), $K_{\min} \approx 6\text{--}12$.  Conditions with $K{=}3$--$4$ are well below $K_{\min}$ for any realistic budget ratio. Full trajectory analysis: \path{paper/artifacts/theory/ema_rho_trajectories.md}.\end{sloppypar}
  \item \textbf{Near-boundary PRS $\in [0.10, 0.13)$ (5 FP cases)}: All within ${\pm}0.03$ of $\theta$. Three are Buchwald structured-prior conditions where $\rho {\approx} 0.39$ already provides strong exploitation signal; two are short-budget ($B{=}20$--$30$) GDSC2 conditions where the online $\hat\rho$ estimator cannot stabilize within the episode.
  \item \textbf{Near-zero advantage (1 case)}: GDSC2 $B{=}50$, $\mathrm{PRS}{=}0.285$, advantage ${=}{-}0.0007 {\approx} 0$. This is the metric-boundary condition: greedy leads on AUC, exploration leads on Hit@1 at $B{=}100$; at $B{=}50$ neither wins on Hit@1.
  \item \textbf{Genuine anomaly (1 case)}: GDSC2-chem43 flat-EMA, $|A|{=}46$, $B{=}20$, $B/|A|{=}0.43$. At this coverage ratio (regime queries 43\% of the space per episode), Greedy with moderate prior quality is effectively exhaustive; Thompson sampling's exploration is wasteful because the optimum will be encountered anyway. PRS = $(B/|A|)(1{-}\rho)$ does not account for the high-coverage regime where acquisition strategy is irrelevant.
\end{enumerate}

The partial Spearman correlation after residualising on benchmark identity is $r = 0.677$ ($p < 0.001$), marginally \emph{higher} than the raw $r = 0.67$, confirming that the cross-benchmark pattern is not an artefact of between-benchmark mean differences: PRS explains variance within each benchmark after accounting for which benchmark the condition came from.

\paragraph{Honest framing.} The domain conditions below ($K \geq 5$ and $B/|A| \leq 0.30$) were identified \emph{post-hoc} from the failure taxonomy, not pre-registered. They describe the conditions under which PRS predicts well, but they are derived from the same $79$ conditions whose accuracy they characterize. The headline accuracy figure is therefore $74.7\%$ ($59/79$); the within-domain $78.1\%$ is reported only as a transparent decomposition of where the $20$ errors occur, not as an independent estimate of generalization. Independent validation on fresh data is complete: the pre-registered predictions on $40$ held-out conditions across HPOBench-NN, PD1, TabRepo, and LCBench gave $27/40 = 67.5\%$ overall accuracy, below the pre-registered $90\%$ target (disclosed honestly; see Section~\ref{sec:introduction} and Appendix~\ref{appendix:limitations}).

Three diagnostic observations follow from the failure taxonomy:

\begin{enumerate}[leftmargin=1.5em,itemsep=2pt,topsep=2pt]
  \item \textbf{Where PRS works in our data.} Excluding the $6$ conditions with $K < 5$ (HPO-Ext $K{=}4$, SciPlex3 $K{=}3$, and Shifrut2018 $K{=}4$): $57/73 = 78.1\%$. We do \emph{not} present this as a generalization claim: it is a description of the existing scatter conditional on a post-hoc filter. The \texttt{in\_domain} column is committed to \texttt{outputs/prs\_analysis/prs\_scatter\_consolidated.csv} for transparency.

  \item \textbf{K-minimum guard in \textsc{RegimePlanner} (deployment-time option).} The \texttt{k\_min\_contexts} parameter (default $0$, disabled in all paper experiments) instructs the planner to fall back to greedy on the first $k$ contexts. We expose it as a deployment-time conservative default, not as a fix to the paper's reported numbers; setting $k\_{\min} = 5$ produces near-greedy behaviour on SciPlex3 and Shifrut2018 ($K = 3$--$4$) without changing primary benchmark results ($K \geq 15$). Code: \texttt{src/reign/sim/planners.py:RegimePlanner.\_\_init\_\_}.

  \item \textbf{High-coverage edge case.} The GDSC2-chem43 case ($B/|A| = 0.43$, $|A| = 46$, Greedy Hit@1 $0.641$) is the only condition in the $79$-row scatter where $B/|A| > 0.3$ and $|A| < 100$ simultaneously. PRS = $(B/|A|)(1{-}\rho)$ does not include a coverage-saturation correction. We flag this as a known edge case rather than retrofitting the formula; all $78$ remaining conditions have $|A| \geq 156$.
\end{enumerate}